\documentclass[12pt]{article}
\usepackage{amssymb}
\usepackage{amsmath}
\usepackage{xcolor} 
\usepackage{amsthm}
\usepackage{amsfonts}
\usepackage{graphicx}
\usepackage{enumerate}
\usepackage{bbm} 
\usepackage{dsfont}
\usepackage[authoryear,longnamesfirst]{natbib}
\usepackage{xr}
\usepackage{todonotes}
\usepackage{placeins}
\usepackage{float}
\usepackage{hyperref}
\usepackage{multirow}
\usepackage{multicol}
\usepackage{tikz}
\numberwithin{equation}{section}
\newtheorem{theorem}{Theorem}

\newtheorem{assumption}{Assumption}

\newtheorem{proposition}{Proposition}
\newtheorem{lemma}{Lemma}
\theoremstyle{definition}
\theoremstyle{definition}
\newtheorem{remark}{Remark}

\usepackage{comment}

\renewcommand{\hat}{\widehat}
\renewcommand{\tilde}{\widetilde}
\newenvironment{parts}
  {\begin{list}{}{\setlength{\leftmargin}{2.2em}\setlength{\labelwidth}{1.8em}%
   \setlength{\labelsep}{0.4em}\setlength{\itemsep}{0.6em}\setlength{\topsep}{0.4em}}}
  {\end{list}}

\allowdisplaybreaks

\begin{document}

\title{\Large High-Dimensional Tail Index Regression\thanks{\footnotesize\setlength{\baselineskip}{4.4mm}First arXiv date: Mar 02 2024. We thank Yi He, Yuan Liao, Denis Chetverikov, and participants at the ESIF Economics and AI+ML Meeting for their very helpful comments and suggestions.}}
\author{ Yuya Sasaki\thanks{\footnotesize\setlength{\baselineskip}{4.4mm} Yuya Sasaki: {yuya.sasaki@vanderbilt.edu}.  Brian and Charlotte Grove Chair and Professor of Economics, Department of Economics, Vanderbilt University, VU Station B \#351819, 2301 Vanderbilt Place, Nashville, TN 37235-1819.\smallskip} 
\and Jing Tao\thanks{\footnotesize\setlength{\baselineskip}{4.4mm} Jing Tao: {jingtao@uw.edu}. Associate Professor of Economics, Department of Economics, University of Washington, Box 353330, Savery 305, University of Washington, Seattle, WA 98195-3330.\smallskip} 
\and Yulong Wang\thanks{\footnotesize\setlength{\baselineskip}{4.4mm} Yulong Wang: {yuw925@lehigh.edu}. Associate Professor of Economics, Department of Economics, Lehigh University, Bethlehem, PA 18015.} }

\date{}
\maketitle
\begin{abstract}
Motivated by the empirical observation of power-law distributions in the credits (e.g., ``likes'') of viral posts on social media, we introduce a high-dimensional tail index regression model and propose methods for estimation and inference of its parameters.
First, we propose a regularized estimator, establish its consistency, and derive its convergence rate. 
Second, we construct a debiased version of the regularized estimator to facilitate inference and prove its asymptotic normality. Monte Carlo simulation studies corroborate our theoretical findings. We apply these methods (1) to the text analysis of viral posts on X (formerly Twitter) and (2) to the analysis of abnormally high health expenditures.

{\small {\ \ \ \newline
} \newline }

\end{abstract}

\newpage

\section{Introduction}

A large body of literature is dedicated to tail features of distributions \citep[e.g.,][]{deHaan07,Resnick07}. 
A common assumption is that the underlying distribution $F$ can be well approximated by a Pareto distribution with shape parameter $\alpha$ in its tail.
This condition is formally defined as \textit{regular variation} and implies that common tail features of interest, such as tail probabilities, extreme quantiles, and tail conditional expectations, can be expressed in terms of $\alpha$.  
Estimates of these features are then obtained by plugging in estimated values of $\alpha$. 
The literature contains numerous suggestions along these lines, some of which are reviewed in the surveys cited below.

\begin{figure}[b]
\centering
\includegraphics[width=9.5cm]{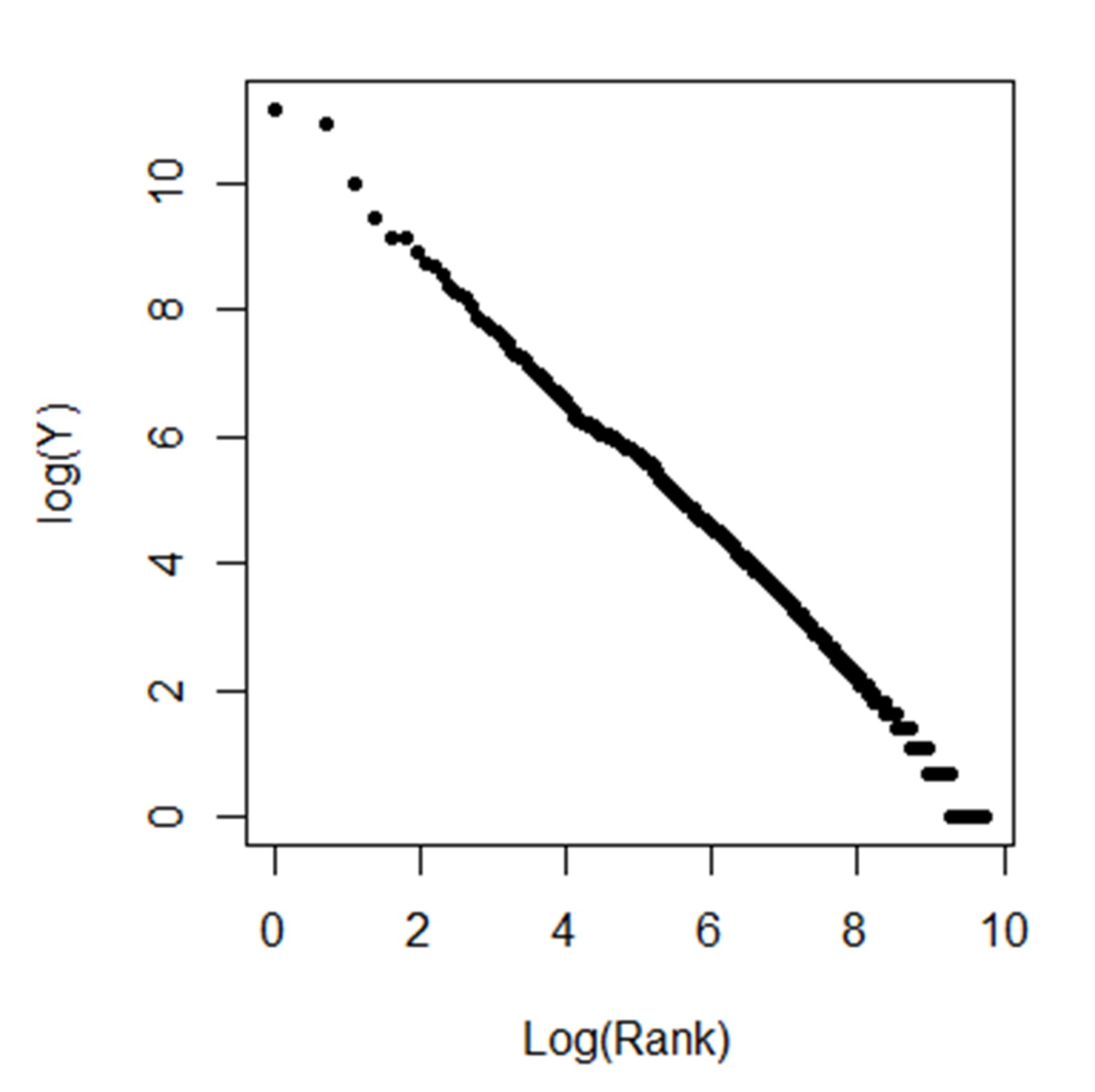}
\caption{\small The log-log plot for the distribution of ``likes'' in a set of posts about LGBTQ$+$ in X. The horizontal axis plots the logarithm of the descending rank of $Y$ while the vertical axis plots $\log(Y)$.}
\label{fig:intro:loglog}
\end{figure} 

As a motivating example, consider the distribution of credits in social media to motivate this framework in contemporary applications.
Figure \ref{fig:intro:loglog} displays the so-called \textit{log-log plot} for the distribution of the number $Y$ of ``likes'' in LGBTQ$+$ posts in X (formerly Twitter). 
If the distribution of $Y$ is Pareto with exponent $\alpha$, the log-log plot would appear linear, as in this figure, with its slope indicating $-1/\alpha$.
This observation motivates us to exploit the Pareto tail for analyzing viral posts on social media.
 
We emphasize that the Pareto tail is not unique to our data. 
It has also been documented and explained for numerous economic, finance, and insurance datasets, including city sizes, firm sizes, stock returns, and natural disasters.  
See \citet{gabaix2009,gabaix2016JEP} for examples.
In text analysis and linguistics, Zipf's law states that, given a large sample of words, the frequency of any word is approximately inversely proportional to its rank in the frequency table. 
This empirical finding can also be characterized by the Pareto distribution with $\alpha \approx 1$ \citep[e.g.,][p.139]{fagan2010introduction}.
 
Suppose that the conditional distribution of $Y$ given $X$ has an approximately Pareto tail with shape parameter $\alpha(X)$ depending on covariates $X$. 
We are interested in the relationship between $X$ and the tail features of $Y$ through $\alpha(X)$. 
One family of existing methods imposes a parametric structure on $\alpha(X)$. 
\citet{wang2009tail} propose a tail index regression (TIR) method by modeling  $\alpha(X) = \exp(X^{\intercal}\theta_0)$ and estimating the pseudo-parameter $\theta_0$. 
\citet{nicolau2023tail} extend the TIR method to accommodate weakly dependent data.
\citet{LiLengYou2020} consider the semiparametric setup $\alpha(X)=\alpha(X_1,X_2) = \exp(X_1^{\intercal}\theta_0 + \eta(X_2))$ for some smooth function $\eta$. 
By combining $\alpha(X) = \exp(X^{\intercal}\theta_0)$ with a power transformation of $Y$, \citet{WangLi2013} study the estimation of conditional extreme quantiles of $Y$ given $X$. 
Another family of existing methods considers fully nonparametric models and local smoothing \citep[e.g.,][]{GardesGirard2010, GardesArmelleAntoine2012, DaouiaGardesGirardLekina2010, DaouiaGardesGirard2013}.

A common feature of these existing approaches is that $X$ is assumed to be of a fixed and low dimension.  
In the current paper, we relax this restriction by allowing the dimension of $X$ to increase with the sample size and even to exceed the sample size. 
Our empirical question related to Figure \ref{fig:intro:loglog} motivates this high-dimensional model. 
Specifically, let $Y_i$ denote the number of ``likes'' of the $i$-th post, and let $X_i$ denote a long vector of binary indicators of whether this post contains a list of keywords. 
Smaller values of $\alpha(X)$ imply that using the words indicated by such $X$ is associated with more extreme numbers of ``likes.'' 
A high-dimensional setup is crucial since the number of keywords is large. 
    
To address this question, we develop a novel high-dimensional tail index regression (HDTIR) method. 
By modifying the TIR method \citep{wang2009tail}, we propose an $\ell_1$-regularized maximum likelihood estimator (MLE).
For statistical inference, we further propose debiasing the regularized estimator and establishing its asymptotic normality.

In addition to the X data, we further demonstrate the policy relevance of the proposed framework with another application in health economics. 
In this application, we analyze the determinants of abnormally high health expenditures using recent survey data. We find that difficulty seeing and cancer are significantly associated with higher tail risks, indicating a higher likelihood of extreme medical spending. This finding highlights the importance of modeling heterogeneous tail behavior and demonstrates that the proposed method uncovers risk patterns that are not captured by conventional moment- or quantile-based approaches.

To the best of our knowledge, the current paper is the first study of estimation and inference theory for the high-dimensional tail index regression model, which constitutes a key contribution. 
Our estimation and inference problems are related to the extensive literature on high-dimensional generalized linear models \citep[e.g.,][among others]{van2008high,  HuangZhang2012, negahban2012unified, van2014asymptotically, zhang2014confidence, belloni2018uniformly, chernozhukov2018double, cai2023statistical}. 
None of the aforementioned papers focuses on tail index regression. 

We emphasize that, while our estimation and inference procedures build on the standard approach for $\ell_1$-penalized $M$-estimation and debiased (de-sparsified) inference, the new components of our analysis stem from the tail-index likelihood and its exact conditional Pareto structure. First, on the tail region the (rescaled) score $\exp(X_i^{\intercal}\theta_0)\log(Y_i/w_n)-1$ has conditional mean zero and conditional variance exactly one, and the Hessian satisfies
\[
\mathbb{E}\left[\exp(X_i^{\intercal}\theta_0)\log(Y_i/w_n)\,X_iX_i^{\intercal}\mid Y_i>w_n\right]=\Sigma_{w_n}.
\]
These identities are specific to the conditional exponential distribution of $\log(Y_i/w_n)$ and are not available for generic generalized linear models. Second, they allow us to build the debiasing (projection) direction from the simple second-moment matrix $\mathbb{E}[X_iX_i^{\intercal}\mid Y_i>w_n]$ rather than from a response-dependent weighted Hessian, and they yield the ``no-sandwich'' conditional variance formula \eqref{eq:var} for the leading score term of the $\sqrt{n_0}$-scaled debiased estimator. Third, our rates are governed by the \emph{tail} sample size $n_0$ rather than the full sample size $n$. The remaining steps of the proofs adapt existing high-dimensional $M$-estimation arguments \citep[e.g.,][]{HuangZhang2012, negahban2012unified, van2014asymptotically, zhang2014confidence, cai2023statistical} to this tail likelihood.

Finally, we show in Appendix \ref{app:2sls} that the least-squares formulation of the model extends to settings with endogenous covariates. The log-log transformation renders the tail index regression exactly linear with an additive error, so that a two-stage least-squares version of the estimator is available. The distinctive difficulty is that the estimation sample is selected on the outcome, so that an instrument that is valid in the population need not remain valid in the tail, and we give a primitive sufficient condition under which it does. We also show that cross-fitting allows the projection direction to be constructed from a quadratic objective \citep{javanmard2014confidence} rather than from the $\ell_1$ objective used in the existing high-dimensional two-stage least-squares literature \citep{gold2020inference}. With this construction, the proof does not require a convergence rate for the estimated projection direction.

Our work also contributes to the vast literature of text analysis. $\ell_1$-regularized estimation has been applied to high-dimensional text regressions \citep[e.g.,][]{taddy2013multinomial} in economics. Nonetheless, there is no method tailored to analyzing tail features relevant to distributions of credits, such as the number of ``likes'' for viral posts in social media. Our proposal addresses this gap in the literature as well.
 
The current paper is also related to the recent literature on shrinkage methods with heavy-tailed data \citep[e.g.,][]{wong2020lasso, fan2021shrinkage, zhu2021taming, babii2023machine}. 
These methods focus on modeling the conditional mean $\mathbb{E}[Y|X]$ using all $n$ observations. 
The heavy-tail feature typically leads to a slower convergence rate, i.e., from $\log{p}$ to a polynomial of $p$, where $p$ denotes the dimension of $X$. 
The asymptotic distributions become more complicated, as does the subsequent statistical inference. 
By contrast, our method relies on the regular variation assumption and focuses on the conditional tail index of $Y$. 
This tail feature requires using only the tail $n_0 < n$ observations, but restores the conventional $\log{p}$ rate. 
Note that our $p$ may increase with $n_0$, and we leave its relation with $n$ unspecified. 

Finally, our paper is also closely related to the literature on extremal quantile regressions \citep{chernozhukov2017extremal} and extremal treatment effects \citep{d2018extremal, zhang2018extremal, deuber2024estimation}. 
The tail index plays a crucial role in the estimation and inference of these parameters.

\medskip\noindent
{\bf Notation.}
Throughout the paper, we use the following notation. For a $p$-dimensional vector $X=(X_1,\dots,X_p)^{\intercal} \in \mathbb{R}^p$, we use $\|X\|_q=(\sum_{i=1}^{p}|X_i|^q)^{1/q}$ to denote the vector $\ell_q$ norm for $1\leq q<\infty$, and $\|X\|_{\infty}=\max_{1\leq i\leq p}|X_i|$ to denote the vector maximum norm.  For a set $S\subseteq\{1,\dots,p\}$,
let ${X}_{S}=\{X_{j}:j\in S\}$ and $S^{c}$ be the complement of $S$.
For a $p\times q$ matrix $A=(a_{i_1i_2})\in\mathbb{R}^{p\times q}$, we use $\|A\|_{1}=\sum_{i_1=1}^{p}\sum_{i_2=1}^{q}|a_{i_1,i_2}|$, $\|A\|_{2} = \|A\|_F = \{\sum_{i_1=1}^{p}\sum_{i_2=1}^{q}(a_{i_1,i_2})^2\}^{1/2}$ and $\|A\|_{\infty}=\max_{1\leq i_1\leq p,1\leq i_2\leq q}|a_{i_1i_2}|$ to denote the element-wise $\ell_1$, $\ell_2$ and $\ell_{\infty}$ norms, respectively. Let $\|A\|_{\ell_d}=\sup_{X\in\mathbb{R}^{q},|X|_d\leq 1}|AX|_d$ denote the matrix operator norm for $1\leq d\leq \infty$. More specifically, the operator $\ell_1$, $\ell_2$ and $\ell_{\infty}$ norms are denoted by
$\|A\|_{\ell_1}=\max_{1\leq i_2\leq q}\sum_{i_1=1}^{p}|a_{i_1i_2}|$, $\|A\|_{\ell_2}=\sqrt{\lambda_{\max}(A^{\intercal}A)}$ and $\|A\|_{\ell_\infty}=\max_{1\leq i_1\leq p}\sum_{i_2=1}^{q}|a_{i_1i_2}|$, respectively.
Let $\lambda_{\min}(A)$ and $\lambda_{\max}(A)$ denote the minimum and maximum eigenvalues of the matrix $A$, respectively.
Let $I_p$ be the $p\times p$ identity matrix.
For two positive sequences $\{a_n\}$ and $\{b_n\}$, $a_n\gtrsim b_n$ means that $a_n>c b_n$ for all $n$ large enough and some constant $c$, $a_n\lesssim b_n$ if $b_n \gtrsim a_n$ holds, and $a_n \asymp b_n$ means that $a_n \lesssim b_n$ and $b_n \lesssim a_n$.
Moreover, $a_n \ll b_n$ means $a_n/b_n \rightarrow 0$.

\section{The High-Dimensional Tail Index Regression}\label{sec:main}
\subsection{Regularized Estimation}
 
Let $\{(X_{i},Y_{i})\}_{i=1}^n$ be $n$ i.i.d.\ copies of $(X,Y)$, where $Y$ is a real-valued response of interest and $X$ is a $p$-dimensional random vector of explanatory factors with $p=p_n\rightarrow\infty$ as the sample size $n$ diverges to $\infty$ and effective tail sample size $n_0\to\infty$ in probability. Throughout, rate conditions involving a random tail sample size are understood in probability. We are interested in modeling the relationship between $X$ and the tail feature of the distribution of $Y$.
Without loss of generality, we focus on the right tail and collect observations such that $Y_{i}>w_n$ for some tail threshold $w_n$.
Let $n_0:=\sum_{i=1}^{n}1 [Y_{i}> w_n]$ denote the effective sample size, and rearrange the indices such that $1 [Y_i> w_n]=1$ for all $i \in \{1,\cdots,n_0\}$.
The following assumptions describe our model.
 
\begin{assumption}[Conditional Pareto Tail]\label{assum:shape}
(i) $\{Y_i,X_i\}$ is i.i.d. with its distribution satisfying
\begin{equation}\label{eq:Pareto}
\mathbb{P}\left( \left. Y>y\right\vert Y>w_n, X=x\right) = \left(\frac{y}{w_n}\right)^{-\alpha \left( x\right) }
\end{equation}
for all $y> w_n$ and a deterministic, sufficiently large threshold $w_n$,
where $\alpha \left( x\right):=\exp \left( x^{\intercal }\theta_{0}\right)$.
(ii) The parameter $\theta_0$ lies in the interior of a convex parameter set $\Theta\subseteq\mathbb{R}^p$ on which the tail index is bounded away from zero and infinity. That is, there exist constants $0<\underline{\alpha}\leq\overline{\alpha}<\infty$, independent of $n$, $p$, and $w_n$, such that $\underline{\alpha}\leq\exp(x^{\intercal}\theta)\leq\overline{\alpha}$ for all $x$ in the support of $X$ and all $\theta\in\Theta$.
\end{assumption}
 
\begin{assumption}[Compact Support]\label{assum:support}
For each $j=1,\dots,p$, $X_{i,j}$ has a compact support $\mathcal{X}_{j}\subset\mathbb{R}$. In particular, $\sup_{x\in\mathcal{X}_{j}}|x|\leq M<\infty$ for a constant $M$ that does not depend on $j$, $p$, or $n$. 
\end{assumption}
 
Assumption \ref{assum:shape} imposes the restriction that $Y$ conditional on $X$ has a Pareto distribution with exponent $\alpha(x)$ in the tail region $\{y\in\mathbb{R}:y>w_n\}$.
Such a Pareto tail has been documented in numerous empirical datasets, as emphasized in the introductory section.
See \citet{gabaix2009,gabaix2016JEP} for comprehensive reviews.
Relative to the low-dimensional TIR model of \citet{wang2009tail}, we additionally require $\alpha(x)$ to be bounded above by $\overline{\alpha}<\infty$ as well as below by $\underline{\alpha}>0$. Both bounds are used in the theory below. The lower bound controls the curvature of the tail log-likelihood, while the upper bound is used, for example, to lower bound $\mathbb{E}[\log(Y_i/w_n)\mid X_i,Y_i>w_n]=1/\alpha(X_i)\geq 1/\overline{\alpha}$ (see the proof of Proposition \ref{prop:RSC} in the supplement). In the low-dimensional setting the upper bound is automatic because $x^{\intercal}\theta_0$ is bounded on the support of $x$. In the high-dimensional setting, it becomes a genuine and standard restriction and is therefore stated explicitly.
 
This condition could be relaxed to
\[
\mathbb{P}\left( \left. Y>y\right\vert Y>w_n, X=x\right)
= \left(\frac{y}{w_n}\right)^{-\alpha(x)}\frac{\mathcal{L}(y,x)}{\mathcal{L}(w_n,x)},
\]
where $\mathcal{L}(\cdot,x)$ is slowly varying at infinity \citep[e.g.,][]{wang2009tail}. 
Slow variation gives $\mathcal{L}(tw_n,x)/\mathcal{L}(w_n,x)\rightarrow 1$ for each fixed $t>1$ and $x$ as $w_n\rightarrow \infty$.
With that said, there are two benefits of imposing \eqref{eq:Pareto}.
First, assuming the exact conditional Pareto tail substantially simplifies the theory, especially given that we focus on high-dimensional $X_{i}$.
Second, the same empirical strategy can be used under an approximate Pareto model, but establishing asymptotically negligible approximation bias would
require additional second-order conditions and restrictions on the threshold sequence, \citep[e.g.,][]{drees1998estimate, drees1998smooth}.
More discussions about the effect of $w_n$ can be found in \citet[][Chapter 3]{deHaan07}, among many others.
We could also relax the i.i.d. condition at the cost of more sophisticated theory, but we focus on this sampling assumption to explicate our main contribution concerning the high-dimensional setup.
 
Assumption \ref{assum:support} requires each coordinate of $X_i$ to have a compact support.
This is coherent with our empirical application, in which $X_i$ is a vector of binary indicators of keywords.
The proofs use this assumption only through the coordinate bound $|X_{i,j}|\le M$. The bound makes the score coordinates $\{\exp(X_i^{\intercal}\theta_0)\log(Y_i/w_n)-1\}X_{i,j}$ sub-exponential (Lemma \ref{lem:score_tail}). It bounds each linear form $b^{\intercal}X_i$ by $M\|b\|_1$, which on the restricted cone is of order $\sqrt{s_0}\,\|b\|_2$. It also bounds the entries of $X_iX_i^{\intercal}$ by $M^{2}$, which gives entrywise concentration of the empirical second-moment matrices. See Remark \ref{rem:RE} for how the restricted strong convexity of the tail log-likelihood is obtained.
 
\begin{remark}\label{rem:RE}
The standard analysis of $\ell_1$-penalized estimators invokes a restricted eigenvalue or compatibility condition on the empirical Gram matrix. Such a condition is usually justified by dimension-free sub-Gaussian control of the linear forms $b^{\intercal}X_i$, which bounded coordinates alone do not provide. For Theorem \ref{Thm:consistency}, we do not impose such a condition and do not use such control. The exact Pareto tail gives $\mathbb{E}[\log(Y_i/w_n)\mid X_i,Y_i>w_n]=1/\alpha(X_i)$, so the conditional expectations of the weights $\log(Y_i/w_n)$ in the quadratic form used to bound the curvature lie between $1/\overline{\alpha}$ and $1/\underline{\alpha}$.
The lower bound on the restricted cone then follows from the population lower eigenvalue condition and concentration of the weighted and unweighted empirical second-moment matrices, using Assumption \ref{assum:support} and the exponential tails of $\alpha(X_i)\log(Y_i/w_n)$. The argument uses the sparsity condition $s_0^{2}\log p=o(n_0)$ in Theorem \ref{Thm:consistency}. The formal statement and proof are given in Proposition \ref{prop:RSC} of the supplementary appendix.
\end{remark}
 
We now introduce our high-dimensional tail index regression (HDTIR) estimator.
Define the negative log-likelihood function $\ell_{n_0}$ of $Y$ conditional on $Y> w_n$ by
\begin{equation}\label{def:likelihood}
\ell_{n_0}(\theta)=\frac{1}{n_0}\sum_{i=1}^{n_0}\left\{ \exp\left(X_{i}^{\intercal}\theta \right)\log\left(Y_{i}/w_n\right)-X_{i}^{\intercal}\theta\right\}.
\end{equation}
Our regularized HDTIR estimator is given by
\begin{equation}\label{eq:lasso}
\hat{\theta}=\arg\min_{\theta}\left\{ \ell_{n_0}(\theta)+\lambda_{n_{0}}\|\theta\|_{1}\right\}.
\end{equation}
We denote the sparsity level of the parameter by $s_0$, i.e., $\|\theta_{0}\|_{0}\leq s_{0}$.
Let $\Sigma_{w_n}=\mathbb{E}\left[X_{i}X_{i}^{\intercal}|Y_{i}>w_n\right]$.
The following theorem states the convergence rate of this regularized estimator.
 
\begin{theorem}\label{Thm:consistency}
Suppose that Assumptions \ref{assum:shape} and \ref{assum:support} hold and that $\lambda_{\min}(\Sigma_{w_n})\geq C_{2}^{-1}$ for a constant $C_{2}>1$ that does not depend on $n$, $p$ or $w_n$. Let $\lambda_{n_{0}}=c\sqrt{(\log p)/n_{0}}$ for a sufficiently large constant $c>0$. If $s_{0}^{2}\log p=o(n_{0})$, then with probability approaching one,
\begin{equation*}
\|\hat{\theta}-\theta_{0}\|_{1}\lesssim\sqrt{\frac{s_{0}^{2}\log p}{n_{0}}}, \quad \|\hat{\theta}-\theta_{0}\|_{2}\lesssim\sqrt{\frac{s_{0}\log p}{n_{0}}}, \quad \mbox{and}\quad \frac{1}{n_{0}}\sum_{i=1}^{n_{0}}\left[X_{i}^{\intercal}\left(\hat{\theta}-\theta_{0}\right)\right]^{2}\lesssim\frac{s_{0}\log p}{n_{0}}.
\end{equation*}
\end{theorem}

Theorem \ref{Thm:consistency} establishes the convergence rates of the regularized estimator. The lower eigenvalue condition on $\Sigma_{w_n}$ is the only design condition beyond Assumption \ref{assum:support} required for these rates. No dimension-free upper bound on $\lambda_{\max}(\Sigma_{w_n})$ is imposed. This matters for uncentered binary covariates. With $\mu_{w_n}:=\mathbb{E}[X_i\mid Y_i>w_n]$,
\[
\Sigma_{w_n}
=\operatorname{Var}(X_i\mid Y_i>w_n)
+\mu_{w_n}\mu_{w_n}^{\intercal},
\]
so $\lambda_{\max}(\Sigma_{w_n})\ge\|\mu_{w_n}\|_2^{2}$, which can increase with $p$ when the coordinate means are nonzero. An upper eigenvalue bound would give the prediction rate from the $\ell_2$ rate. We instead obtain the prediction bound directly from the first-order optimality condition and entrywise concentration of the weighted and unweighted second-moment matrices. Thus uncentered binary and other discrete covariates are admissible subject to the remaining assumptions.

The sparsity condition $s_0^{2}\log p=o(n_0)$ enters in two places. First, the population eigenvalue condition is transferred to the empirical Gram matrix on the restricted cone by entrywise concentration (Remark \ref{rem:RE} and Proposition \ref{prop:RSC} in the supplement). This transfer requires $s_0\sqrt{(\log p)/n_0}\to0$. Second, the same condition ensures that $M\|\hat{\theta}-\theta_0\|_1\le1$ with probability approaching one, which the oracle inequality for the nonlinear likelihood requires. The $\ell_1$ and $\ell_2$ rates coincide with those for the Lasso in the sparse linear model \citep[Corollary 2]{negahban2012unified} and for generalized linear models with canonical links \citep{van2008high, HuangZhang2012, cai2023statistical}. Our contribution is not the rates but their derivation for the tail index likelihood from the $n_0$ tail observations alone.

\begin{remark}[On the deterministic threshold]\label{rem:threshold} 
Theorems \ref{Thm:consistency} and \ref{Thm:normality2} treat $w_n$ as deterministic. 
Replacing $w_n$ with an empirical quantile $\hat{w}_n$ changes both the log excesses $\log(Y_i/w_n)$ and membership of the tail sample. These effects may cancel at first order, but exploiting such cancellation would require stronger assumptions than those imposed here. Specifically, suppose that the exact conditional Pareto specification holds with the same parameter $\theta_0$ above a deterministic baseline $0<v_n<w_n$, with $\mathbb{P}(\hat{w}_n\geq v_n)\to1$. Then,
$\Psi_n(t):=
\mathbb{E}\!\left[
1[Y_i>t]X_i
\left\{\alpha(X_i)\log(Y_i/t)-1\right\}
\right]=0$,
$t\geq v_n$,
so the two threshold effects cancel in the population score. Extending this observation to the debiased estimator would additionally require quantile regularity and suitable uniform control of the empirical scores, initial and projection estimates, variance normalization, and remainder terms over the relevant threshold neighborhood. Assumption \ref{assum:shape} imposes the Pareto specification only above $w_n$, and hence does not cover a downward fluctuation of
$\hat{w}_n$ below $w_n$, nor do our assumptions impose the required uniform conditions. We therefore retain the deterministic-threshold formulation and leave inference with an estimated threshold under such stronger assumptions for future work.
\end{remark}

As extensively discussed in the literature, $\hat{\theta}$ cannot be directly used to construct a confidence interval for $\theta_0$. In the next subsection, we introduce a debiased estimator to address this issue and facilitate statistical inference.

\subsection{Debiased Estimation and Inference}\label{sec:debias}
The preliminary regularized estimator exhibits a bias that is non-negligible relative to its stochastic variation.
Hence, its asymptotic distribution cannot be used directly for inference. In high-dimensional settings, it is standard to construct a debiased (de-sparsified) estimator to remove this bias and restore valid asymptotic inference. Motivated by this, the present section develops a debiased estimation and inference procedure. We adopt a cross-fitting scheme in which each bias-correction step uses a complementary subsample of that used to obtain the initial estimate.
 
We first present the debiasing step.
Let $K$ be a fixed integer greater than 1.
Take a $K$-fold random partition $(\mathcal{D}_{k})^{K}_{k=1}$ of the indices $[n_0]=\{1,\dots,n_0\}$ so that the size of each fold $\mathcal{D}_{k}$ is $n_{k}=n_0/K$ for simplicity.
For each $k=1,\dots,K$, define the set $\mathcal{D}^{c}_{k}=\{1,\dots,n_0\}\backslash \mathcal{D}_{k}$ of indices in the complement of the fold.
 
For each $k \in \{1,...,K\}$, we estimate $\hat{\theta}_{k}$ via \eqref{eq:lasso} by using the subsample $\mathcal{D}^{c}_{k}$, and we construct the projection direction $\hat{u}_{j,k}$ via \eqref{alg:cross_fitting}--\eqref{constraint2:cross} using \emph{only the covariates} $\{X_i:i\in\mathcal{D}_{k}\}$ of the fold $\mathcal{D}_{k}$ (the responses $\{Y_i:i\in\mathcal{D}_{k}\}$ are not used in constructing $\hat{u}_{j,k}$).
Specifically,
\begin{align}\label{alg:cross_fitting}
\hat{u}_{j,k}  =&\arg\min_{u\in\mathbb{R}^p} u^{\intercal}\left(\frac{1}{n_k}\sum_{i\in \mathcal{D}_{k}}X_{i}X_{i}^{\intercal}\right)u\\
s.t. & \ \left\Vert \left(\frac{1}{n_k}\sum_{i\in \mathcal{D}_{k}}X_{i}X_{i}^{\intercal}\right)u-e_{j}\right\Vert _{\infty}\leq\gamma_{1n_{0}}\quad \mbox{and}\label{constraint1:cross}\\
     & \ \max_{i\in\mathcal{D}_{k}}|X_{i}^{\intercal}u|\leq\gamma_{2n_{0}}.\label{constraint2:cross}
\end{align}
Here $e_j$ denotes the $j$-th canonical basis vector of $\mathbb{R}^{p}$, and $\gamma_{1n_0}$ and $\gamma_{2n_0}$ are tuning parameters specified in Assumption \ref{assum:rate_tuning}.
Because $\hat{u}_{j,k}$ depends on the fold $\mathcal{D}_{k}$ only through the covariates, the debiasing score evaluated at $\theta_0$ below retains conditional mean zero. Specifically, conditional on $\{X_i:i\in\mathcal{D}_{k}\}$ and on the observations in $\mathcal{D}_{k}^{c}$, the only remaining randomness in the score is that of $Y_i$ given $X_i$, and $\mathbb{E}[\{\exp(X_i^{\intercal}\theta_0)\log(Y_i/w_n)-1\}X_i\mid X_i,Y_i>w_n]=0$ by Lemma \ref{lem:pareto_moments}. This is why $\hat{u}_{j,k}$ may be computed on the same fold on which the score is evaluated. The initial estimator $\hat{\theta}_k$ is computed on $\mathcal{D}_k^{c}$, so that it is independent of the observations in $\mathcal{D}_k$.
Then, for each $j \in \{1,\dots,p\}$, let $\hat{\theta}_{j,k}$ denote the $j$-th coordinate of $\hat{\theta}_k$ and
\begin{align}\label{eq:theta_tilde_MLE}
    \tilde{\theta}_{j,k}:=\hat{\theta}_{j,k}-\frac{\hat{u}_{j,k}^{\intercal}}{n_{k}}\sum_{i\in \mathcal{D}_{k}}\left\{ \exp\left(X_{i}^{\intercal}\hat{\theta}_{k}\right)\log\left(Y_{i}/w_n\right)-1\right\} X_{i},
\end{align}
and define the debiased estimator by taking the average across the $K$ folds:
\begin{equation}
\tilde{\theta}_{j}:=\frac{1}{K}\sum^{K}_{k=1} \tilde{\theta}_{j,k}. \label{eq:dtheta_cross}
\end{equation}
 
We emphasize that the construction of $\hat{u}_{j}$ takes advantage of our Pareto tail assumption. Specifically, the existing literature usually constructs $\hat{u}_{j}$ using the Hessian, which does not involve $Y_i$. In our case, the score and Hessian of $\ell_{n_{0}}$ evaluated at $\hat\theta$ take the forms of
\begin{align*}
\dot{\ell}_{n_{0}}(\hat{\theta})=&\frac{1}{n_{0}}\sum_{i=1}^{n_{0}}\left\{ \exp\left(X_{i}^{\intercal}\hat{\theta}\right)\log\left(Y_{i}/w_n\right)-1\right\} X_{i}
\qquad\text{ and }\\
\ddot{\ell}_{n_{0}}(\hat{\theta})=&\frac{1}{n_{0}}\sum_{i=1}^{n_{0}}\log\left(Y_{i}/w_n\right)\exp\left(X_{i}^{\intercal}\hat{\theta}\right)X_{i}X_{i}^{\intercal},
\end{align*}
respectively. 
The leading score evaluated at $\theta_0$ has conditional mean zero given the retained covariates. Sample splitting does not generally make the score evaluated at $\hat{\theta}$ conditionally centered. Instead, independence of the initial estimator and the evaluation observations is used to control the remainder terms in the debiasing expansion.
However, our Hessian involves $\exp(X_{i}^{\intercal}\hat{\theta}) \log (Y_{i}/w_n )X_{i} X_{i}^{\intercal}$, which contains $Y_i$.
To address this issue, we note that the conditional distribution of $\log (Y_{i}/w_n)$ given $X_i$ and $Y_i>w_n$ is exponential with rate $\exp (X_i^{\intercal}\theta_0)$ under Assumption \ref{assum:shape}.
It follows that
\[
\mathbb{E}\left[ \exp (X_{i}^{\intercal}\theta_0) \log (Y_{i}/w_n) X_{i}X_{i}^{\intercal} \mid Y_i>w_n \right] = \mathbb{E}\left[ X_{i}X_{i}^{\intercal} \mid Y_i>w_n \right],
\]
which motivates our construction of $\hat{u}_{j}$.
 
For the debiased estimator \eqref{eq:dtheta_cross}, we define the asymptotic variance estimator by
\begin{equation}\label{eq:var}
    \hat{V}_{1j} :=\frac{1}{K^{2}}\sum_{k=1}^{K}\frac{n_0}{n_k}\left[\hat{u}_{j,k}^{\intercal}\left(\frac{1}{n_{k}}\sum_{i\in \mathcal{D}_{k}}X_{i}X_{i}^{\intercal}\right)\hat{u}_{j,k}\right].
\end{equation}
Using the conditional Pareto tail again, the rescaled score $\exp(X_i^{\intercal}\theta_0)\log(Y_i/w_n)-1$ has conditional variance exactly one (Lemma \ref{lem:pareto_moments}(iii)). Consequently, conditional on the retained covariates and
the fold assignment, the leading score term in each fold, scaled by $\sqrt{n_k}$, has the ``no-sandwich" variance
$\hat{u}_{j,k}^{\intercal}\hat{\Sigma}_k\hat{u}_{j,k}$,
where
$\hat{\Sigma}_k=n_k^{-1}
\sum_{i\in\mathcal{D}_k}X_iX_i^{\intercal}$.
Thus, \eqref{eq:var} is the exact conditional variance of the aggregated leading score term in
$\sqrt{n_0}(\tilde{\theta}_j-\theta_{0j})$,
and no separate estimation of a middle matrix of a
sandwich form is required.
The normalization does not require replacing
$\hat{\Sigma}_k$ by $\Sigma_{w_n}$ or establishing convergence
of $\hat{V}_{1j}$ to a non-random limit.
Two additional assumptions are needed to derive the limiting
distribution of the debiased estimator.

\begin{assumption}\label{assum:rate_tuning} For some fixed constants
$c,c',c''>0$, independent of $n$, $p$, and $w_n$, with $c$ and $c'$ sufficiently large, (i) $\lambda_{n_0}=c\sqrt{(\log p)/n_0}=o(1)$, (ii) $\gamma_{1n_0}=c'\sqrt{(\log p)/n_0}=o(1)$,
and (iii) $\gamma_{2n_0}=c''\sqrt{\log n_0}$.
\end{assumption}

\begin{assumption}\label{assum:proj}
Let $u_j^{*}:=\Sigma_{w_n}^{-1}e_j$ for $j=1,\dots,p$. There is a constant $K_0<\infty$, independent of $n$, $p$ and $w_n$, such that for every $j$ and every $t>0$,
\[
\mathbb{P}\left(\left\vert X_i^{\intercal}u_j^{*}\right\vert>t\;\middle|\;Y_i>w_n\right)\ \le\ 2\exp\!\left(-t^{2}/K_0^{2}\right).
\]
Moreover, $c''\ge 2K_0$ in Assumption \ref{assum:rate_tuning}(iii).
\end{assumption}

Assumption \ref{assum:rate_tuning} sets the tuning rates. The penalty $\lambda_{n_0}$ and the tolerance $\gamma_{1n_0}$ are of the order of the maximal score, $\sqrt{(\log p)/n_0}$. The level $\gamma_{2n_0}\asymp\sqrt{\log n_0}$ bounds the linear form $X_i^{\intercal}\hat{u}_{j,k}$ over the fold. This bound controls the third-moment term in the Berry--Esseen bound for the leading term and the second-order remainder of the nonlinear likelihood.

Assumption \ref{assum:proj} concerns the linear form $X_i^{\intercal}u_j^{*}$ with $u_j^{*}=\Sigma_{w_n}^{-1}e_j$. This linear form equals the residual from the population linear projection of $X_{i,j}$ on the other covariates, divided by the conditional second moment of that residual in the tail subpopulation. Assumption \ref{assum:proj} requires it to be sub-Gaussian with a constant free of $p$. Assumption \ref{assum:support} does not imply this, since boundedness gives only $|X_i^{\intercal}u_j^{*}|\le M\|u_j^{*}\|_1$ and $\|u_j^{*}\|_1$ may grow with $p$. The maximum of $n_0$ such sub-Gaussian variables is of order $\sqrt{\log n_0}$, which explains the level of $\gamma_{2n_0}$ and the requirement $c''\ge2K_0$. Assumption \ref{assum:proj} is used in Lemma \ref{lem:projection} in two ways. It shows that $u_j^{*}$ is feasible for \eqref{alg:cross_fitting}--\eqref{constraint2:cross}, and it bounds the quadratic objective at $u_j^{*}$. Two sufficient conditions are available. Assumption \ref{assum:proj} holds with $K_0=MC_3/\sqrt{\log 2}$ if $\max_j\|\Sigma_{w_n}^{-1}e_j\|_1\le C_3$, and it holds if $X_i$ given $Y_i>w_n$ is a sub-Gaussian vector with a constant free of $p$.
 
We now state the asymptotic distribution for the debiased estimator $\tilde{\theta}_{j}$ defined in \eqref{eq:dtheta_cross}.
 
\begin{theorem}\label{Thm:normality2}
Suppose that Assumptions \ref{assum:shape}--\ref{assum:proj} hold. In addition, suppose the conditions of Theorem \ref{Thm:consistency}, the $K$-fold partition is drawn independently of the data, and
\begin{equation}\label{eq:sparsity2}
s_0\,\big(\log(p\vee n_0)\big)^{3/2}\ \ll\ \sqrt{n_0}.
\end{equation}
Then, for each $j=1,\dots,p$,
\[
\sqrt{n_0}\,\hat{V}_{1j}^{-1/2}\left(\tilde{\theta}_{j}-\theta_{0j}\right)\overset{d}{\rightarrow}\mathcal{N}(0,1)\qquad\text{as }n_0\rightarrow\infty .
\]
\end{theorem}
 
\begin{remark}[On the sparsity condition]\label{rem:sparsity}
The sparsity requirement $s_0\ll\sqrt{n_0}/\{\log(p\vee n_0)\}^{3/2}$ is stronger than the condition $s_0\ll\sqrt{n_0}/\log p$ familiar from the debiased-Lasso literature for linear models. The extra factor arises from the second-order remainder of the nonlinear likelihood. The projection direction satisfies the constraint $\max_i|X_i^{\intercal}\hat{u}_{j,k}|\leq\gamma_{2n_0}\asymp\sqrt{\log n_0}$ in \eqref{constraint2:cross}, and this $\sqrt{\log n_0}$ factor enters the bound on the remainder term $I_{3,k}$ in the proof of Theorem \ref{Thm:normality2} (see the supplement). When $\log n_0\asymp\log p$, the sufficient condition simplifies to $s_0(\log p)^{3/2}\ll\sqrt{n_0}$. If $\max_i|X_i^{\intercal}\hat{u}_{j,k}|$ can instead be bounded by a constant, the same proof yields the usual condition $s_0\ll\sqrt{n_0}/\log p$.
\end{remark}
 
One could obtain a similar asymptotic normality result without cross-fitting (or sample splitting). However, following the insight from \cite{cai2023statistical}, with the cross-fitting method we propose, the debiased estimator achieves asymptotic normality without requiring the inverse matrix $\Sigma_{w_n}^{-1}$
to be weakly sparse, which relaxes a standard assumption in the literature  \citep[see, e.g.,][for linear regression models]{van2014asymptotically, zhang2014confidence}.

As an alternative to cross-fitting, we can also consider a simpler sample-splitting implementation in which the initial regularized estimation and the subsequent bias-correction step are conducted on two independent subsamples. More specifically, the effective sample of size $n_0$ is randomly divided into two disjoint subsets $\mathcal{D}_{1}=\{(X_{i},Y_{i})\}_{i=1}^{n_{0}/2}$ and $\mathcal{D}_{2}=\{(X_{i},Y_{i})\}_{i=n_0/2+1}^{n_{0}}$.
We use the subsample $\mathcal{D}_{2}$ to obtain $\hat{\theta}$ via \eqref{eq:lasso} and use the subsample $\mathcal{D}_{1}$ for the debiasing step. 
This independence facilitates control of the terms involving $\hat{\theta}-\theta_0$ in the debiasing expansion. The leading score evaluated at $\theta_0$ has conditional mean zero given the initial estimator and the covariates of the evaluation sample $\mathcal{D}_1$. The score evaluated at $\hat{\theta}$ need not have this property.
While sample splitting can be less statistically efficient than cross-fitting due to using only a fraction of the data in each step, it provides a transparent baseline construction and yields the same type of asymptotic normality result under comparable high-dimensional sparsity and tuning conditions.

\subsection{Extensions}\label{sec:extensions}

We extend our analysis in four directions.
First, the baseline method is based on maximum likelihood. Appendix~\ref{app:ols} develops an alternative based on least squares.

Second, the least-squares formulation allows for endogenous covariates. The log-log transformation \eqref{eq:z} in Appendix~\ref{app:ols} turns the tail index regression into an exact linear model, in which an unobserved determinant of tail heaviness enters the error additively. In the likelihood score, by contrast, it enters multiplicatively. Appendix~\ref{app:2sls} uses this linear form to construct a two-stage least-squares version of the HDTIR with debiased inference. The main difficulty is that the estimation sample is selected on the outcome, so an instrument that is valid in the population need not be valid in the tail. We give a sufficient condition under which it remains valid.

Third, the coefficient $\theta_j$ may not be directly interpretable in applications. It can be transformed into a more interpretable measure, such as the marginal effect on the conditional long-rise\footnote{The long rise is the right-tail counterpart of the short fall.} probability. Appendix~\ref{sec:transformation} shows how to conduct estimation and inference for this transformed measure.

Fourth, the method can be extended to large-scale online data, such as social media data that arrive continuously. Appendix~\ref{app:online} presents this extension, which uses a variant of stochastic gradient descent to process the tail observations one at a time.
 
\section{Simulation Studies}\label{sec:simulation}

In this section, we use simulated data to numerically evaluate the performance of our proposed method of estimation and inference.
Two designs for the $p$-dimensional parameter vector $\theta_0$ are employed:
\begin{align*}
&\text{1. Sparse Design:} && \theta_0 = (1.0,0.9,0.8,\ldots,0.2,0.1,0.0,0.0,0.0,\ldots)^\intercal, \qquad\text{and}
\\
&\text{2. Exponential Design:} && \theta_0 = (1.0,0.5,0.5^2,0.5^3,\cdots)^\intercal.
\end{align*}

A random sample of $(Y_i,X_i^\intercal)^\intercal$ is generated as follows.
Three designs for the $p$-dimensional covariate vector $X_i$ are used:
\begin{align*}
&\text{1. Gaussian Design:} && X_i \sim \mathcal{N}(0, 0.1^2 \cdot I_p),
\\
&\text{2. Uniform Design:} && X_i \sim \text{Uniform}(-0.1,0.1) \qquad\text{and}
\\
&\text{3. Bernoulli Design:} && X_i \sim 0.1 \cdot \text{Bernoulli}(0.1).
\end{align*}
where $I_p$ denotes the $p \times p$ identity matrix.
In turn, generate the exponents by
$$
\alpha_i = \exp(X_i^\intercal \theta_0)
$$
and then generate $Y_i$ by
$$
Y_i = \Lambda^{-1}(U_i; \alpha_i), \qquad U_i \sim \text{Uniform}(0,1),
$$
where $\Lambda(\ \cdot\ ; \alpha)$ denotes the CDF of the Pareto distribution with the unit scale and exponent $\alpha$.

In each iteration, we draw a random sample $(Y_i,X_i^\intercal)^\intercal$ of size $n=$ 10,000.
Setting the cutoff $w_n$ to the 95th empirical percentile of $\{Y_i\}_{i=1}^n$, we have the effective sample size of $n_0=$ 500 from five percent of $n$.
We vary the dimension $p \in \{250,500,1000\}$ of the parameter vector $\theta_0$ across sets of simulations.
While there are $p$ coordinates in $\theta_0$, we focus on the first coordinate $\theta_{01} = 1.0$ for evaluating our method of estimation and inference.
Throughout, we use $K = 5$ for the number of subsamples in cross-fitting folds.
The other tuning parameters are set according to Assumption \ref{assum:rate_tuning} where $c=1$, $c'=1$ and $c''=100$.
We run 10,000 Monte Carlo iterations for each design.

Table \ref{tab:simulation_results} summarizes the simulation results.
\begin{table}[t]
\centering\renewcommand{\arraystretch}{0.8}
\begin{tabular}{cccccccccc}
\\\\\hline\hline
$n_0$ & $p$ & $\theta$ & $X$ & $\Lambda$ && Bias & SD & RMSE & 95\% \\
\hline
500 & 250 & Sparse      & Gaussian & Pareto && 0.023 & 0.545 & 0.545 & 0.918\\
500 & 250 & Exponential & Gaussian & Pareto && 0.012 & 0.494 & 0.494 & 0.936\\
500 & 500 & Sparse      & Gaussian & Pareto && 0.022 & 0.542 & 0.542 & 0.919\\
500 & 500 & Exponential & Gaussian & Pareto && 0.007 & 0.487 & 0.487 & 0.942\\
500 &1000 & Sparse      & Gaussian & Pareto &&0.033 & 0.520 & 0.521 & 0.927\\
500 &1000 & Exponential & Gaussian & Pareto &&0.000 & 0.483 & 0.483 & 0.938\\
500 & 250 & Sparse      & Uniform  & Pareto && 0.000 & 0.824 & 0.824 & 0.941\\
500 & 250 & Exponential & Uniform  & Pareto &&-0.018 & 0.799 & 0.799 & 0.949\\
500 & 500 & Sparse      & Uniform  & Pareto &&-0.004 & 0.825 & 0.825 & 0.941\\
500 & 500 & Exponential & Uniform  & Pareto &&-0.015 & 0.791 & 0.792 & 0.948\\
500 &1000 & Sparse      & Uniform  & Pareto &&-0.008 & 0.820 & 0.820 & 0.941\\
500 &1000 & Exponential & Uniform  & Pareto &&-0.020 & 0.800 & 0.800 & 0.944\\
500 & 250 & Sparse      & Bernoulli& Pareto &&-0.233 & 0.716 & 0.753 & 0.961\\
500 & 250 & Exponential & Bernoulli& Pareto &&-0.101 & 0.831 & 0.837 & 0.955\\
500 & 500 & Sparse      & Bernoulli& Pareto &&-0.252 & 0.710 & 0.753 & 0.964\\
500 & 500 & Exponential & Bernoulli& Pareto &&-0.112 & 0.823 & 0.830 & 0.958\\
500 &1000 & Sparse      & Bernoulli& Pareto &&-0.244 & 0.713 & 0.753 & 0.963\\
500 &1000 & Exponential & Bernoulli& Pareto &&-0.101 & 0.826 & 0.832 & 0.955\\
\hline\hline
\end{tabular}
\caption{\small Simulation results. The sets of results vary with the dimension $p$ of the parameter vector $\theta_0$, the design for the parameter vector $\theta_0$, and the design for the covariate vector $X$. For each row, displayed are the bias (Bias), standard deviations (SD), root mean square errors (RMSE), and the coverage frequencies by the 95\% confidence interval (95\%).}${}$
\label{tab:simulation_results}
\end{table}
The sets of results vary with the dimension $p$ of the parameter vector $\theta$, the design for the parameter vector $\theta$, and the design for the covariate vector $X$. 
Each row displays the bias (Bias) of the debiased estimator $\widetilde\theta_1$, standard deviations (SD), root mean square errors (RMSE), and the coverage frequencies by the 95\% confidence interval (95\%).

For each set, the bias is much smaller than the standard deviations and hence the 95\% confidence interval delivers accurate coverage frequencies.
We ran many other sets of simulations with different values of $n_0$ and $p$ as well as parameter designs and data-generating designs, and confirmed that the simulation results turned out to be similar in qualitative patterns to those presented here.
The additional results are omitted from the paper to avoid repetitive exposition.

\section{Empirical Applications}\label{sec:application}

This section presents two empirical applications.
The first application implements text analysis to study viral posts on social media.
The second application examines the relation between physical conditions and abnormally high health expenditures.

\subsection{Application 1: Text Analysis of Viral Posts on LGBTQ+}\label{sec:application_text}

In this section, we apply our proposed method to analyze LGBTQ$+$-related posts on X. Our goal is to infer the statistical relation between specific words and the number of ``likes'' given to these posts. The dataset comprises tweets containing the keyword ``LGBT,'' scraped from Twitter between August 21 and August 26, 2022.\footnote{The data set is publicly available at \texttt{https://www.kaggle.com/datasets/vencerlanz09/lgbt-tweets}.}

Each observation in our study represents a single post. Our sample includes a total of $n = 32,456$ posts. The data records the number of likes, $Y_i$, that the $i$-th post has received. As we will demonstrate below, $Y_i$ follows a heavy-tailed distribution. Most posts attract a small number of likes, while a few viral posts receive a large number of likes.
We construct a word bank consisting of 936,556 unique words used across the $n = 32,456$ posts in our sample. The $j$-th coordinate, $X_{ij}$, of the covariate vector $X_i$ takes a value of 1 if the $j$-th word in the word bank is used in the $i$-th post and 0 otherwise.
From these 936,556 unique words, we only include the 500 most frequently used words to create the binary indicators in $X_i$. Therefore, the dimension $p$ of $X_i$ is 500. This list explicitly excludes articles, auxiliary verbs, and prepositions.

Figure \ref{fig:intro:loglog} in the introduction presents the log-log plot of the empirical distribution $\{Y_i\}_{i=1}^n$ for posts with a positive number of likes. We focus on posts with positive likes because the logarithm of zero is undefined. The horizontal axis represents the rank of $Y$, while the vertical axis represents $\log(Y)$. The approximate linearity of this log-log plot suggests that the distribution of $Y$ follows a power law, indicating that $Y$ is characterized by a Pareto distribution.

\begin{figure}[t]
\centering
\includegraphics[width=12cm]{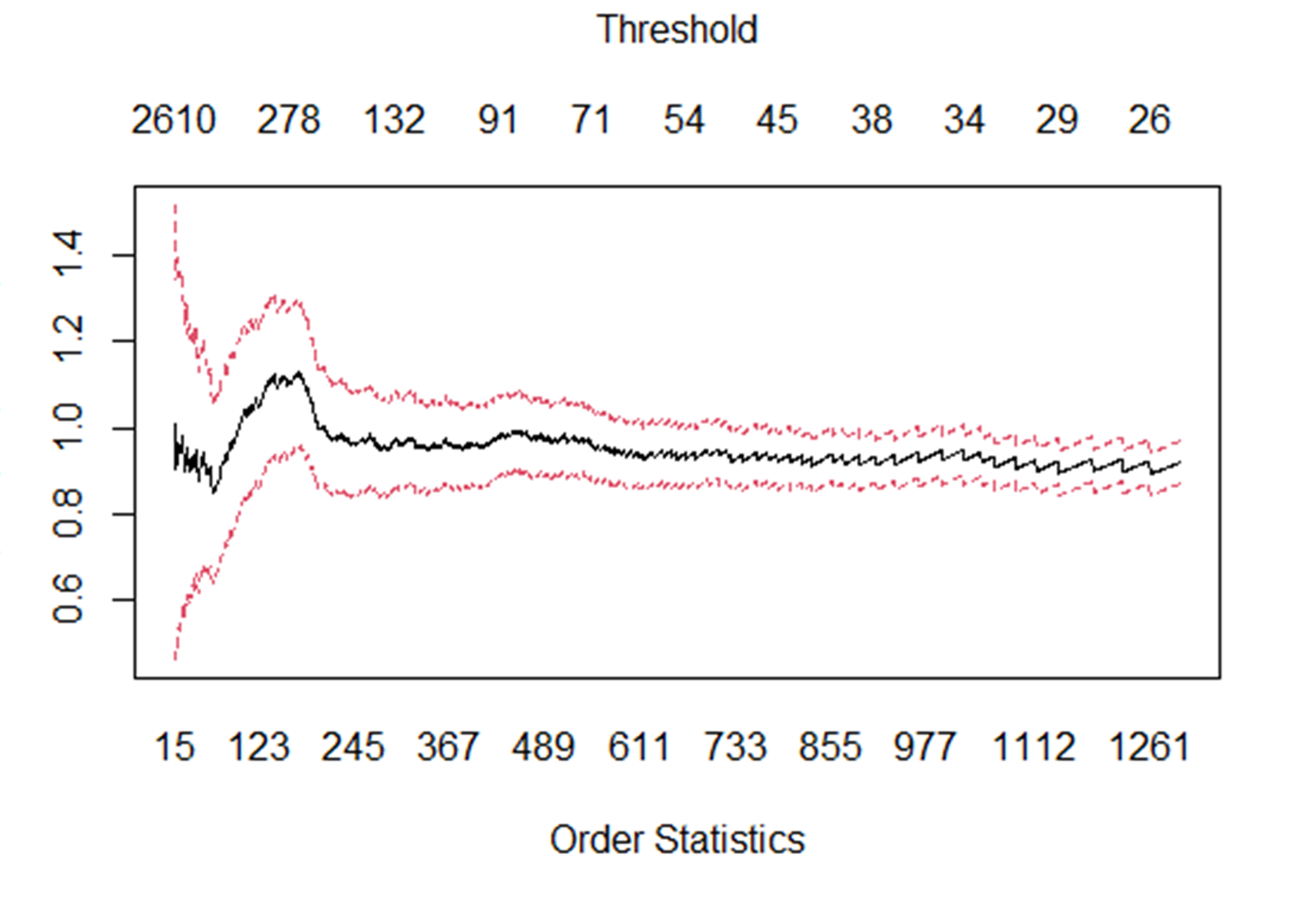}
\caption{\small The Hill plot for the distribution of ``likes'' among LGBTQ$+$-related posts on X. The horizontal axis represents the order statistics (equivalently, the threshold), while the vertical axis reports the estimated Pareto tail index. The plot is constructed using exceedances from the 96th percentile up to the maximum observation.}${}$
\label{fig:lgbt_hill}
\end{figure} 


We use the Hill plot both to assess the Pareto-tail approximation and to examine the sensitivity of the estimated marginal tail index to the choice of threshold. Each point in the plot corresponds to a candidate threshold and an estimate computed from observations exceeding that threshold. Thus, a relatively flat region indicates that the estimated marginal tail index changes little over the corresponding threshold choices. Figure~\ref{fig:lgbt_hill} exhibits a reasonably flat region among thresholds above the 96th percentile, while including lower cutoffs near the 95th percentile reveals greater variation. We use this stability pattern to guide the threshold choice and set $w_n$ to the 96th percentile.

Table \ref{tab:top30} displays the 30 most frequently used words in LGBTQ$+$ posts. For each word, the total number of times it appeared (Count) and the total number of posts in which it appeared (Tweets) are shown. The last value corresponds to $\sum_{i=1}^n X_{ij}$ for $j=1,\dots,30$. All characters have been converted to lowercase to ensure the counting is not case-sensitive.
Notice that the most frequent word, ``lgbt,'' is distinct from the eleventh most frequent word, ``\#lgbt.'' The former is a plain word, while the latter functions as a hashtag, serving to link posts with others containing the same hashtag.

\begin{table}[tb]\renewcommand{\arraystretch}{0.825}
\scalebox{0.92}{
\begin{tabular}{|rlrrc|rlrrc|rlrr|}
\hline
& \small Word & \small Count & \small Tweets && & \small Word & \small Count & \small Tweets && & \small Word & \small Count & \small Tweets\\ 
\hline
1 & lgbt 	& 23734 & 8133 && 	11 & \#lgbt 		& 4989 & 1280 && 21 & 	it's	 & 3194 & 	1842\\
2 & and 	& 18706 & 12357 && 	12 & this 		& 4956 & 3532 && 22 & 	lgbt+	 & 3165 & 	578\\
3 & i 		& 10162 & 1559 && 	13 & community 	& 4189 & 3640 && 23 & 	their	 & 3127 & 	2641\\
4 & that 	&9022 & 6931 && 		14 & have 		& 4186 & 3605 && 24 & 	my	 & 2747 & 	2003\\
5 & you 	& 8965 & 5374 && 		15 & just 		& 3563 & 2900 && 25 & 	don't	 & 2690 & 	2195\\
6 & people 	& 7117 & 5495 && 		16 & or 		& 3555 & 2903 && 26 & 	what	 & 2662 & 	1840\\
7 & it 		& 7015 & 5025 && 		17 & so 		& 3517 & 2534 && 27 & 	he	 & 2631 & 	1548\\
8 & not 	& 5987 & 4650 && 		18 & if 		& 3346 & 2207 && 28 & 	your	 & 2608 & 	2068\\
9 & they 	& 5695 & 3705 && 		19 & all 		& 3231 & 2624 && 29 & 	gay	 & 2559 & 	1833\\
10 &  but 	& 5012 & 3770 && 		20 & who 		& 3227 & 2723 && 30 & 	do	 & 2559 & 	2072\\
\hline 
\end{tabular}
}
\caption{\small The top 30 most frequent words used in LGBTQ$+$ posts. Displayed next to each word are the total number of times it appeared (Count) and the total number of posts in which it appeared (Tweets). All the characters are unified to lower-case letters so that the counting is not case-sensitive.}${}$
\label{tab:top30}
\end{table}

We apply our proposed method of estimation and inference to analyze the association between the use of these and other words and the tail shape of the distribution of the number of likes. As discussed above, we set $w_n$ to the 96th percentile of the empirical distribution of $\{Y_i\}_{i=1}^n$, which results in an effective sample size of $n_0 = 1,300$. The rules for selecting the tuning parameters remain the same as those used in our simulation studies.

Table \ref{tab:frequent30_results} presents the estimates, standard errors, 95\% confidence intervals, and t-statistics for $\theta_j$ for the 30 most frequently used words, listed in the same order as in Table \ref{tab:top30}. Notably, the most frequent word, ``lgbt,'' has a significantly negative coefficient, while the eleventh most frequent word, ``\#lgbt,'' has an insignificant coefficient.
Recall that smaller values of the Pareto exponent correspond to heavier tails and more extreme values of $Y_i$. Therefore, this finding suggests that using the plain word ``lgbt'' is associated with a substantially larger number of likes, whereas using the hashtag ``\#lgbt'' may not be.
Most of the other words in Table \ref{tab:frequent30_results} are statistically insignificant, with the exceptions of the generic words, such as ``it,'' ``have,'' and ``all,'' which have positive coefficients.

\begin{table}
\centering
\scalebox{0.89}{\renewcommand{\arraystretch}{0.825}
\begin{tabular}{|rlccrrcc|rlccrrc|}
\hline
$j$ & Word & $\widetilde\theta_j$ & SE & \multicolumn{2}{c}{95\% CI} & t && $j$ & Word & $\widetilde\theta_j$ & SE & \multicolumn{2}{c}{95\% CI} & t\\
\hline
1	& lgbt		& -0.13 & 0.06 & [-0.26 & 0.00] & -1.99	&& 16	 & or	 	& -0.05 & 0.10 & [-0.25 & 0.15] & -0.48\\
2	& and		& -0.01 & 0.06 & [-0.13 & 0.10] & -0.25	&& 17	 & so	 	&  0.16 & 0.10 & [-0.05 & 0.36] &  1.49\\
3	& i		    &  0.09 & 0.13 & [-0.17 & 0.36] &  0.70	&& 18	 & if	 	& -0.08 & 0.12 & [-0.32 & 0.16] & -0.67\\
4	& that		&  0.04 & 0.07 & [-0.10 & 0.17] &  0.53	&& 19	 & all	 	&  0.26 & 0.09 & [ 0.07 & 0.44] &  2.76\\
5	& you		& -0.01 & 0.08 & [-0.17 & 0.15] & -0.11	&& 20	 & who	    & -0.03 & 0.09 & [-0.21 & 0.16] & -0.28\\
6	& people	&  0.06 & 0.07 & [-0.09 & 0.20] &  0.75	&& 21	 & it's	 	& -0.11 & 0.14 & [-0.39 & 0.16] & -0.82\\
7	& it		&  0.16 & 0.08 & [ 0.01 & 0.31] &  2.08	&& 22	 & lgbt+	&  0.06 & 0.17 & [-0.28 & 0.40] &  0.36\\
8	& not		&  0.08 & 0.08 & [-0.07 & 0.24] &  1.04	&& 23	 & their	& -0.07 & 0.11 & [-0.29 & 0.14] & -0.66\\
9	& they	    &  0.07 & 0.09 & [-0.10 & 0.24] &  0.79	&& 24	 & my	 	&  0.08 & 0.12 & [-0.16 & 0.31] &  0.64\\
10	& but		&  0.06 & 0.09 & [-0.11 & 0.23] &  0.71	&& 25	 & don't	&  0.02 & 0.12 & [-0.22 & 0.25] &  0.15\\
11	& \#lgbt	&  0.18 & 0.18 & [-0.17 & 0.53] &  1.00	&& 26	 & what	    &  0.02 & 0.13 & [-0.23 & 0.27] &  0.15\\ 
12	& this		&  0.14 & 0.08 & [-0.02 & 0.30] &  1.67	&& 27	 & he	 	& -0.15 & 0.12 & [-0.39 & 0.09] & -1.23\\
13	& community	&  0.01 & 0.09 & [-0.16 & 0.18] &  0.12	&& 28	 & your	    & -0.18 & 0.13 & [-0.44 & 0.08] & -1.35\\
14	& have	    &  0.17 & 0.08 & [ 0.01 & 0.34] &  2.06	&& 29	 & gay	 	&  0.07 & 0.11 & [-0.15 & 0.28] &  0.62\\
15	& just		& -0.06 & 0.11 & [-0.26 & 0.15] & -0.54	&& 30	 & do	 	&  0.01 & 0.12 & [-0.22 & 0.25] &  0.10\\
\hline
\end{tabular}
}
\caption{\small Estimates, standard errors, 95\% confidence intervals, and the t statistics for $\theta_j$ for the top 30 most frequent words. These words are listed in the same order as in Table \ref{tab:top30}.}${}$
\label{tab:frequent30_results}
\end{table}

While the reported estimates of $\theta_j$ are useful for identifying statistically significant words, the magnitude of $\theta_j$ is not immediately interpretable, since it enters the Pareto exponent model $\alpha(x)=\exp(x^\intercal\theta)$.
A more interpretable quantity is the marginal effect of $x_j$ on the conditional tail probability
$
\mathbb{P}\!\left(Y>y \mid Y>w_n, X=x\right),
$
i.e., the conditional long-rise probability -- see Appendix~\ref{sec:transformation}. Using the transformation formulas developed in that appendix, we find that the significant word ``lgbt'' increases the conditional long-rise probability of exceeding the 99th percentile for the benchmark subpopulation by 4.3 percentage points.

\subsection{Application 2: Abnormal Health Expenditures}\label{sec:application_health}
In this section, we apply our proposed method to examine the statistical relationships between physical/chronic conditions and abnormally high health expenditures. It is well documented that roughly 50\% of the total population health care spending is determined by the 5\% of the sickest individuals in society \citep[e.g.,][]{French2016,Karlsson2016}. 
A large body of the health economics literature has investigated the determinants of high health expenditures \citep[e.g.,][]{manning1987monte,finkelstein2006multiple,obermeyer2019dissecting,karlsson2024getting}. However, existing studies primarily focus on moment-based measures and do not explore how these factors influence the tails of the expenditure distribution, namely, extremely abnormal spending. We demonstrate that our method is well suited to addressing this important issue, especially given that there are a large number of patient characteristics.

We use data from the 2023 wave of the Medical Expenditure Panel Survey (MEPS), made publicly available by the Agency for Healthcare Research and Quality (AHRQ).
We focus on a set of physical and chronic conditions as the primary factors of interest. These include physical limitations, work limitations, difficulty seeing, difficulty hearing, difficulty walking, cognitive difficulties, diabetes, high blood pressure, coronary heart disease, asthma, cancer, arthritis, and high cholesterol.
In addition, we include a rich set of socio-demographic variables, along with their interaction terms. For continuous variables, we incorporate polynomial expansions up to degree 10. In total, the model contains $p=823$ explanatory variables.
The dataset contains 17,578 observations.
While the MEPS data set provides sample weights, we do not use them in our analysis.

\begin{figure}[t]
\centering
\includegraphics[width=12cm]{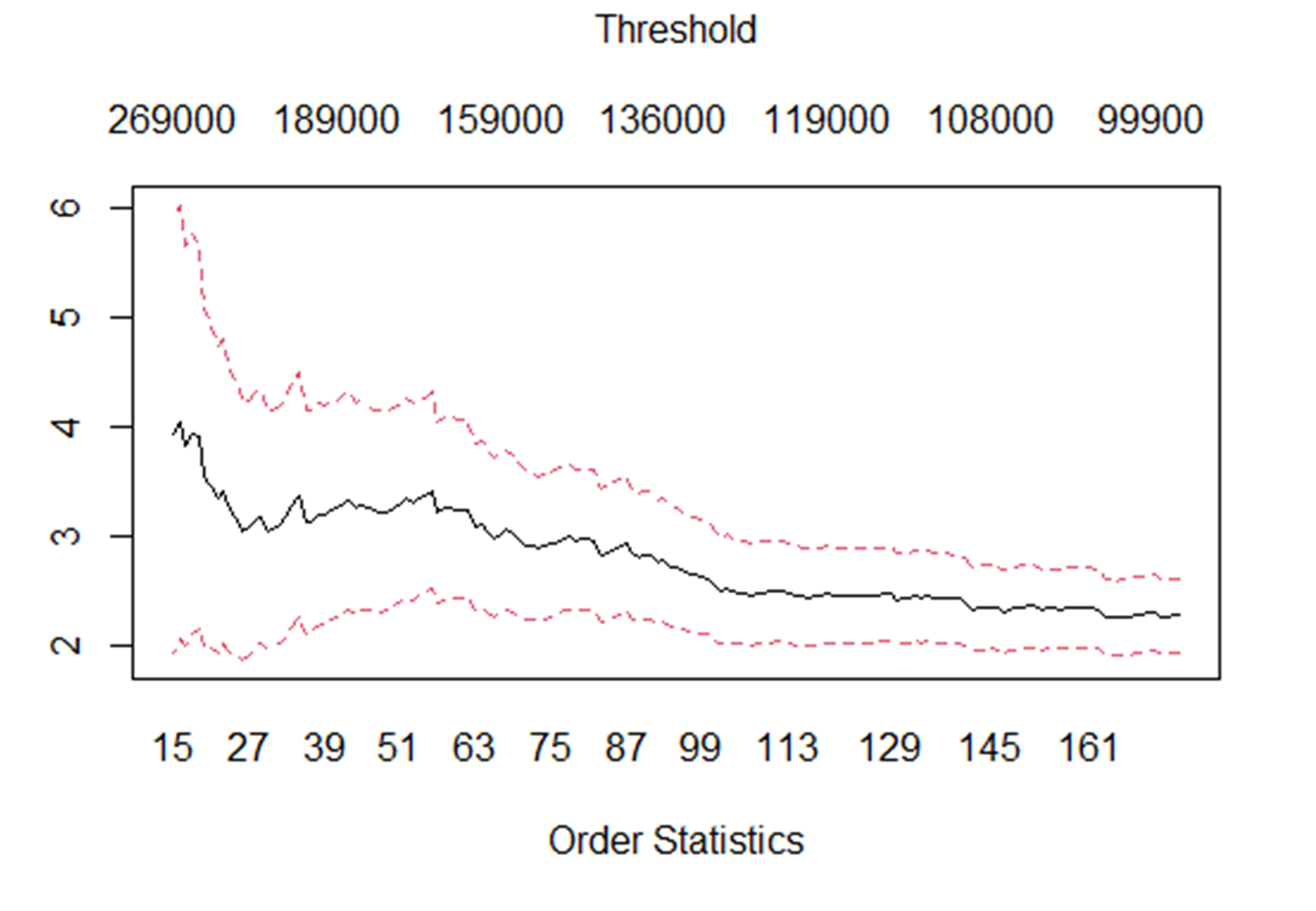}
\caption{\small The Hill plot for the distribution of health expenditures. The horizontal axis represents the order statistics (equivalently, the threshold), while the vertical axis reports the estimated Pareto tail index. The plot is constructed using exceedances from the 99th percentile up to the maximum observation.}${}$
\label{fig:health_hill}
\end{figure} 


As in the text application, the Hill plot provides both a diagnostic for the Pareto-tail approximation and an assessment of threshold sensitivity of the estimated marginal tail index.
Figure~\ref{fig:health_hill} exhibits a reasonably flat region among thresholds above the 99th percentile, whereas lower cutoffs near the 98th percentile produce greater variation.
We therefore select the 99th percentile as $w_n$, using the observed stability of the marginal tail-index estimates across candidate cutoffs to guide the threshold choice.

\begin{table}[tb]
\centering\renewcommand{\arraystretch}{0.85}
\begin{tabular}{|lcclrc|}
\hline
Condition & $\widetilde\theta_j$ & SE & \multicolumn{2}{c}{95\% CI} & t\\
\hline
Physical Limitations&   -0.070 & 0.074 & [-0.214 & 0.075] & -0.941\\
Work Limitations&       -0.117 & 0.074 & [-0.262 & 0.028] & -1.587\\
Difficulty Seeing&      -0.225 & 0.114 & [-0.449 & 0.000] & -1.963\\
Difficulty Hearing&     -0.119 & 0.136 & [-0.386 & 0.148] & -0.873\\
Difficulty Walking&     -0.078 & 0.082 & [-0.240 & 0.083] & -0.950\\
Cognitive Difficulties& -0.068 & 0.082 & [-0.229 & 0.093] & -0.832\\
Diabetes&                0.006 & 0.079 & [-0.147 & 0.160] &  0.082\\
High Blood Pressure&    -0.033 & 0.080 & [-0.189 & 0.123] & -0.419\\
Coronary Heart Disease& -0.084 & 0.077 & [-0.236 & 0.067] & -1.093\\
Asthma&                 -0.165 & 0.099 & [-0.359 & 0.030] & -1.662\\
Cancer&                 -0.189 & 0.089 & [-0.362 &-0.015] & -2.130\\
Arthritis&               0.008 & 0.078 & [-0.145 & 0.161] &  0.098\\
High Cholesterol&       -0.058 & 0.076 & [-0.207 & 0.092] & -0.758\\
\hline
\end{tabular}
\caption{\small Estimates, standard errors, 95\% confidence intervals, and the t statistics for $\theta_j$ for the physical and chronic conditions.}${}$
\label{tab:health_expenditures_results}
\end{table}

Table~\ref{tab:health_expenditures_results} summarizes the results. 
Observe that difficulty seeing and cancer yield statistically significant negative coefficients, implying a smaller tail index and hence a heavier tail, which corresponds to a higher likelihood of abnormally large health expenditures.
None of the other conditions are statistically significant at the 5\% level. 
Asthma is marginally significant at the 10\% level, with the same negative sign, suggesting that it may also be associated with unusually high health expenditures.

The reported estimates of $\theta_j$ are not immediately interpretable. As a more interpretable quantity, we next consider the marginal effect of $x_j$ on the conditional long-rise probability
$
\mathbb{P}\!\left(Y>y \mid Y>w_n, X=x\right).
$
Using the transformation formulas developed in Appendix~\ref{sec:transformation}, we find that cancer increases the conditional long-rise probability of exceeding the 99.5th percentile for the benchmark subpopulation by 7.0 percentage points.

\section{Summary}

This paper introduces a novel high-dimensional tail index regression (HDTIR) model inspired by observing power-law distributions in social media posts, particularly in the distribution of ``likes'' on viral content. We study estimation and inference for the parameters of the tail index model when the dimension of the explanatory variables increases and may exceed the sample size.

We begin by developing a regularized estimation method for the HDTIR model, demonstrating its consistency and establishing its convergence rate. To facilitate inference, we introduce a debiasing technique that corrects the bias introduced by regularization. This allows us to derive the asymptotic normality of the debiased estimator, providing a robust framework for statistical inference in high-dimensional settings.

Extensive simulation studies validate the theoretical properties of our model, showing strong performance even in finite samples. In addition, we apply the HDTIR method to two empirical datasets.
The first application is on a dataset of viral posts on X related to LGBTQ$+$ topics. This empirical analysis shows how specific words are associated with the likelihood of a post going viral, with terms like `lgbt' playing a significant role while hashtags like `\#lgbtq' do not. The results demonstrate the practical utility of the HDTIR model in understanding and predicting the factors that drive the popularity of online content.


The second application considers the relationship between physical/chronic conditions and abnormally high health expenditures. We find that difficulty seeing and cancer are significantly associated with higher likelihoods of abnormally high health expenditures.

The appendix contains four extensions.
Appendix~\ref{app:ols} develops a least-squares alternative to the maximum likelihood estimator. It is less efficient under the Pareto tail but is simple to compute.
Appendix~\ref{app:2sls} extends the least-squares estimator to endogenous covariates. It proposes a two-stage least-squares version of the HDTIR with debiased inference and gives a sufficient condition under which an instrument remains valid in the tail sample.
Appendix~\ref{sec:transformation} transforms the coefficients into marginal effects on conditional tail probabilities, which are easier to interpret in applications.
Appendix~\ref{app:online} extends the method to large-scale online data, such as social media data that arrive continuously. It uses a variant of stochastic gradient descent to process the tail observations one at a time.

\bibliography{bib}

\section*{Data Availability Statement}
The data used in this article are publicly available.

\newpage
\appendix
\section*{Appendix}
The appendix collects extensions to the baseline method.
Appendix \ref{app:ols} presents the least-square-type estimator.
Appendix \ref{app:2sls} extends the least-squares estimator to endogenous covariates. 
Appendix \ref{sec:transformation} transforms the coefficients into tail probabilities.
Appendix \ref{app:online} considers the extension to online data. 

\section{The OLS Approach to the HDTIR}\label{app:ols}
Our HDTIR estimator \eqref{eq:lasso} builds on the maximum likelihood estimation. 
Recently, \citet{nicolau2024simple} develop a simple tail index estimator based on linear regression. 
Inspired by this method, the current appendix section presents a least-squares approach to the HDTIR.

Assumption \ref{assum:shape} implies that the conditional distribution of $Y_i$ given $\{Y_i>w_n,X_i=x\}$ is the standard Pareto distribution with scale $w_n$ and exponent $\alpha(x)$, and hence $\log(Y_i/w_n)|Y_i>w_n,X_i=x$ is an exponential random variable with rate $\alpha(x)$. 
Define
\begin{equation}\label{eq:z}
    Z_i = - \log\left(\log \left( \frac{Y_i}{w_n} \right)\right) - \eta
\end{equation}
for $Y_i>w_n$, where $\eta \approx 0.5772$ denotes the Euler's constant. 
Note that the conditional mean of $Z_i$ given $X_i=x$ and $Y_i>w_n$ is simply $\log(\alpha(x))=x^\intercal\theta_0$ under the specification that $\alpha(x)=\exp(x^\intercal\theta_0)$. 

Therefore, $Z_i$ follows a shifted Type-I extreme value distribution conditional on $Y_i>w_n,X_i=x$ and satisfies
\begin{equation*}
    \mathbb{E}[Z_i - X_i'\theta_0 |X_i, Y_i>w_n] = 0 \text{\, and\, } Var[Z_i|X_i,Y_i>w_n] = \frac{\pi^2}{6}
\end{equation*}
\citep[cf.][Section 2.1.1]{nicolau2024simple}.
Accordingly, our regularized HDTIR-LS estimator can be given by
\begin{equation}\label{eq:ols}
\hat{\theta}_{LS}
=
\arg\min_{\theta\in\mathbb{R}^{p}}
\left\{
\frac{1}{2n_0}
\sum_{i=1}^{n_0}
(Z_i-X_i^{\intercal}\theta)^2
+
\lambda_{2,n_0}\|\theta\|_1
\right\},
\end{equation}
where $\lambda_{2,n_0}=c_2\sqrt{(\log p)/n_0}$
for a sufficiently large constant $c_2>0$.
When fitting on a subsample, replace $n_0$
in the objective and penalty by the subsample size.
The Hessian of the unpenalized normalized least-squares loss is $\hat{\Sigma}=n_0^{-1}\sum_{i=1}^{n_0}X_iX_i^{\intercal}$.
Under the exact conditional Pareto model, this equals the conditional expectation of the likelihood Hessian at $\theta_0$, given the retained covariates.
Its population counterpart is $\Sigma_{w_n}$.

The debiasing approach based on the least squares estimator $\hat{\theta}_{LS}$ is simpler. 
Following the same notation as in Section \ref{sec:debias}, for each $k\in \{1,\dots,K\}$, we estimate $\widehat{\theta}_{k}$ via \eqref{eq:ols} by using the subsample of $\mathcal{D}^{c}_{k}$, and estimate $\hat{u}_{j,k}$ via \eqref{alg:cross_fitting}--\eqref{constraint2:cross} by using the subsample $\mathcal{D}_{k}$. 
Then, for each $j\in \{1,\dots,p\}$, let 
\begin{align}\label{eq:theta_tilde_LS}
    \tilde{\theta}_{j,k}:=\hat{\theta}_{j,k} + \frac{\hat{u}_{j,k}^{\intercal}}{n_{k}}\sum_{i\in \mathcal{D}_{k}}\left\{Z_i-X_i^{\intercal}\hat{\theta}_{k} \right\} X_{i}.
\end{align}

Suppose that the conditions of Theorem \ref{Thm:normality2} hold, with $s_0\ge1$ and $p\to\infty$, and use the least-squares penalty specified above.
Under the exact conditional Pareto model, the errors are centered Gumbel variables with variance $\pi^2/6$ and sub-exponential tails.
The regularized least-squares estimator satisfies
\[
\|\hat{\theta}_{LS}-\theta_0\|_1
=
O_p\!\left(
s_0\sqrt{\frac{\log p}{n_0}}
\right),
\qquad
\|\hat{\theta}_{LS}-\theta_0\|_2
=
O_p\!\left(
\sqrt{\frac{s_0\log p}{n_0}}
\right).
\]
The same rates hold for the initial estimators $\hat{\theta}_k$ fitted on the complementary folds.
Let $\tilde{\theta}_{j,LS}:=K^{-1}\sum_{k=1}^{K}\tilde{\theta}_{j,k}$,
where $\tilde{\theta}_{j,k}$ is defined in \eqref{eq:theta_tilde_LS}, and let $\hat V_{1j,LS}=\hat V_{1j}\cdot\pi^2/6$.
Then, for each fixed coordinate $j$,
\[
\sqrt{n_0}\,
\hat{V}_{1j,LS}^{-1/2}
\left(
\tilde{\theta}_{j,LS}-\theta_{0j}
\right)
\overset{d}{\longrightarrow}\mathcal{N}(0,1).
\]
The estimation rates follow from the Lasso oracle
inequality and concentration of the empirical Gram
matrix. The asymptotic normality follows from the
linear debiasing decomposition and the argument used in the proof of Theorem \ref{Thm:normality2}.

The key difference lies in efficiency. Because the MLE exploits the full likelihood implied by the conditional Pareto tail, it achieves a smaller asymptotic variance, whereas the OLS estimator is generally less efficient but remains consistent and asymptotically normal. Consequently, the OLS approach provides a robust and computationally convenient alternative whose large-sample behavior mirrors that of the MLE, up to an efficiency loss.

\paragraph{Simulations.}
We conduct Monte Carlo simulations to compare the OLS estimator \eqref{eq:ols} with the MLE \eqref{eq:lasso}. 
The data-generating process is the same as in Section \ref{sec:simulation}. 
Table \ref{tab:simulation_results_least_squareds} presents the results. 

We have two key findings.
First, the HDTIR-OLS estimator performs generally well across all specifications. 
It has a small bias and correct coverage. 
Second, compared with the MLE estimator \eqref{eq:lasso}, OLS has a larger standard error and RMSE, as expected. 
This is because the MLE is the most efficient estimator under the Pareto tail assumption, and hence achieves smaller asymptotic variance. 
This result is coherent with the simulations in the low-dimensional case \citep{nicolau2024simple}.

\begin{table}[tb]
\centering\renewcommand{\arraystretch}{0.82}
\begin{tabular}{cccccccccc}
\\\\\hline\hline
$n_0$ & $p$ & $\theta$ & $X$ & $\Lambda$ && Bias & SD & RMSE & 95\% \\
\hline
500 & 250 & Sparse      & Gaussian & Pareto &&-0.102 & 0.592 & 0.601 & 0.944\\
500 & 250 & Exponential & Gaussian & Pareto &&-0.030 & 0.595 & 0.596 & 0.947\\
500 & 500 & Sparse      & Gaussian & Pareto &&-0.103 & 0.592 & 0.601 & 0.942\\
500 & 500 & Exponential & Gaussian & Pareto &&-0.038 & 0.593 & 0.594 & 0.943\\
500 &1000 & Sparse      & Gaussian & Pareto &&-0.098 & 0.590 & 0.598 & 0.943\\
500 &1000 & Exponential & Gaussian & Pareto &&-0.029 & 0.594 & 0.595 & 0.946\\
500 & 250 & Sparse      & Uniform  & Pareto &&-0.030 & 1.006 & 1.007 & 0.951\\
500 & 250 & Exponential & Uniform  & Pareto &&-0.027 & 1.017 & 1.017 & 0.946\\
500 & 500 & Sparse      & Uniform  & Pareto &&-0.034 & 1.031 & 1.031 & 0.943\\
500 & 500 & Exponential & Uniform  & Pareto &&-0.026 & 1.014 & 1.014 & 0.947\\
500 &1000 & Sparse      & Uniform  & Pareto &&-0.048 & 1.031 & 1.032 & 0.946\\
500 &1000 & Exponential & Uniform  & Pareto &&-0.023 & 1.017 & 1.017 & 0.946\\
500 & 250 & Sparse      & Bernoulli& Pareto &&-0.006 & 1.167 & 1.166 & 0.948\\
500 & 250 & Exponential & Bernoulli& Pareto &&-0.011 & 1.188 & 1.188 & 0.942\\
500 & 500 & Sparse      & Bernoulli& Pareto &&-0.039 & 1.175 & 1.176 & 0.948\\
500 & 500 & Exponential & Bernoulli& Pareto &&-0.017 & 1.174 & 1.175 & 0.948\\
500 &1000 & Sparse      & Bernoulli& Pareto &&-0.022 & 1.180 & 1.180 & 0.947\\
500 &1000 & Exponential & Bernoulli& Pareto &&-0.006 & 1.180 & 1.180 & 0.946\\
\hline\hline
\end{tabular}
\caption{\small Simulation results for the least squares method. The sets of results vary with the dimension $p$ of the parameter vector $\theta_0$, the design for the parameter vector $\theta_0$, and the design for the covariate vector $X$. For each row, displayed are the bias (Bias), standard deviations (SD), root mean square errors (RMSE), and the coverage frequencies by the 95\% confidence interval (95\%).}${}$
\label{tab:simulation_results_least_squareds}
\end{table}

\section{Extension to Endogenous Covariates}\label{app:2sls}
Throughout the paper we have maintained that the covariates are exogenous in the sense that the tail index is $\alpha(X_i)=\exp(X_i^{\intercal}\theta_0)$. In many applications of interest, one may be concerned that an unobserved determinant of tail heaviness is correlated with a covariate of interest, or that a covariate is measured with error. This appendix section develops an instrumental-variable extension of the HDTIR. The extension builds on the least-squares formulation of Appendix \ref{app:ols}. 

Let $C_i$ denote an unobserved scalar that shifts the tail index, and replace Assumption \ref{assum:shape} by its conditional analog
\begin{equation}\label{eq:frailty}
\mathbb{P}\left(Y_i>y \mid Y_i>w_n, X_i=x, C_i=c\right)=\left(\frac{y}{w_n}\right)^{-\alpha(x,c)},
\qquad
\alpha(x,c)=\exp\left(x^{\intercal}\theta_0+c\right),
\end{equation}
for $y>w_n$. We allow $\mathbb{E}[C_i\mid X_i]\neq 0$, so that $X_i$ is endogenous, and we assume the availability of a vector $Q_i\in\mathbb{R}^{p_q}$ of instruments, where $p_q$ may itself diverge with $n_0$ and may exceed $n_0$.

Write $\log(Y_i/w_n)=E_i/\alpha(X_i,C_i)$ with $E_i\mid (X_i,C_i,Y_i>w_n)\sim\text{Exp}(1)$.
Substituting into the transformation \eqref{eq:z} gives the linear representation
\begin{equation}\label{eq:z_frailty}
Z_i=X_i^{\intercal}\theta_0+\underbrace{C_i+\varepsilon_i}_{=:U_i},
\qquad
\varepsilon_i=-\left(\log E_i+\eta\right),
\end{equation}
where $\varepsilon_i$ is a centered standard Gumbel variate with variance $\pi^2/6$ that is independent of $(X_i,C_i,Q_i)$ given $Y_i>w_n$, under the exclusion restriction imposed below. Thus, the log-log transformation converts the tail index regression into a linear model in which the confounder enters additively. This makes the extension below possible.

Before turning to estimation, we address a difficulty that is specific to tail index regression. The estimation sample is not a random sample from the population. It consists of the observations selected by the event $\{Y_i>w_n\}$, that is, selected on the outcome. Because the selection probability may depend on $C_i$, conditioning on selection may tilt the distribution of the confounder given the instrument. Population instrument validity therefore need not survive selection.

Write $S_i=1[Y_i>w_n]$ and $\pi_n(x,c)=\mathbb{P}(Y_i>w_n\mid X_i=x,C_i=c)$, and impose the exclusion restriction $Q_i\perp Y_i\mid (X_i,C_i)$. A direct calculation, recorded as Proposition \ref{prop:selection} below, gives
\begin{equation}\label{eq:tilt}
\mathbb{E}\left[U_i\mid Q_i,S_i=1\right]
=\frac{\mathbb{E}\left[C_i\,\pi_n(X_i,C_i)\mid Q_i\right]}{\mathbb{E}\left[\pi_n(X_i,C_i)\mid Q_i\right]} .
\end{equation}
Equation \eqref{eq:frailty} alone does not determine $\pi_n(x,c)$. If, in addition, \eqref{eq:frailty} holds
with a fixed threshold $w_0>0$ in place of $w_n$, then, for $w_n>w_0$,
\[
\pi_n(x,c)=\pi_0(x,c)
\left(\frac{w_n}{w_0}\right)^{-\alpha(x,c)}
=\pi_0(x,c)
\exp\left\{-\lambda_n e^{x^{\intercal}\theta_0+c}\right\},
\]
where $\pi_0(x,c):=\mathbb{P}(Y_i>w_0\mid X_i=x,C_i=c)$ and $\lambda_n:=\log(w_n/w_0)$.
Relative to the subpopulation $\{Y_i>w_0\}$, further selection into $\{Y_i>w_n\}$ therefore exponentially
tilts the joint distribution of $(X_i,C_i)$ toward smaller values of $x^{\intercal}\theta_0+c$.
Since the instrument shifts $X_i$, it may also shift the selection probability, and the right-hand side of \eqref{eq:tilt} can be a non-constant function of
$Q_i$ even when $\mathbb{E}[C_i\mid Q_i]=0$. We regard this as the central conceptual challenge of the extension.

\begin{proposition}\label{prop:selection}
Suppose that \eqref{eq:frailty} holds and $Q_i\perp Y_i\mid(X_i,C_i)$. Then \eqref{eq:tilt} holds. Moreover, if $X_i=X_i^{*}+\xi_i$ with $\alpha(X_i^{*})=\exp(X_i^{*\intercal}\theta_0)$, $\xi_i\perp(X_i^{*},Y_i,Q_i)$ and $\mathbb{E}[\xi_i]=0$, so that $C_i=-\xi_i^{\intercal}\theta_0$, then $\mathbb{E}[U_i\mid Q_i,S_i=1]=0$ for every $w_n$.
\end{proposition}

In what follows, we take the resulting orthogonality condition as a high-level assumption, Assumption \ref{assum:iv}(ii). 
Proposition \ref{prop:selection} gives a primitive sufficient condition.

Randomly partition the effective tail sample $\{i:Y_i>w_n\}$ into an auxiliary fold $\mathcal{D}_{0}$ and $K$ further folds $\mathcal{D}_{1},\dots,\mathcal{D}_{K}$, all of common size $n_k=n_0/(K+1)$, and write $n_{0}^{\ast}=\sum_{k=1}^{K}n_{k}=Kn_k$ for the number of observations used in the second stage. We propose the following two-stage estimator. 

On the first stage, on $\mathcal{D}_{0}$, and for each $\ell=1,\dots,p$, compute
\begin{equation}\label{eq:first_stage}
\hat{a}_{\ell}=\arg\min_{a\in\mathbb{R}^{p_q}}\left\{\frac{1}{|\mathcal{D}_{0}|}\sum_{i\in\mathcal{D}_{0}}\left(X_{i\ell}-Q_i^{\intercal}a\right)^2+\lambda_{3,n_0}\left\Vert a\right\Vert_1\right\},
\end{equation}
collect $\hat{A}=[\hat{a}_1,\dots,\hat{a}_p]\in\mathbb{R}^{p_q\times p}$, and form the generated regressors
$\hat{d}_i=\hat{A}^{\intercal}Q_i\in\mathbb{R}^{p}$ for $i\notin\mathcal{D}_0$.
 
On the second stage, for each $k\in\{1,\dots,K\}$, compute the regularized two-stage least-squares estimator on the complementary folds,
\begin{equation}\label{eq:2sls}
\hat{\theta}_k=\arg\min_{\theta\in\mathbb{R}^{p}}\left\{\frac{1}{n_0^{\ast}-n_k}\sum_{i\in\mathcal{D}^{c}_{k}\setminus\mathcal{D}_0}\left(Z_i-\hat{d}_i^{\intercal}\theta\right)^2+\lambda_{4,n_0}\left\Vert\theta\right\Vert_1\right\}.
\end{equation}

Let $\hat{\Sigma}_{d,k}=n_k^{-1}\sum_{i\in\mathcal{D}_k}\hat{d}_i\hat{d}_i^{\intercal}$.
Exactly as in \eqref{alg:cross_fitting}--\eqref{constraint2:cross}, and using \emph{only the instruments} $\{Q_i:i\in\mathcal{D}_k\}$ of the fold, construct
\begin{align}\label{alg:cross_fitting_iv}
\hat{u}_{j,k} =&\arg\min_{u\in\mathbb{R}^p} u^{\intercal}\hat{\Sigma}_{d,k}\,u\\
s.t. & \ \left\Vert \hat{\Sigma}_{d,k}\,u-e_{j}\right\Vert _{\infty}\leq\gamma_{1n_{0}}\quad\mbox{and}\label{constraint1:iv}\\
     & \ \max_{i\in\mathcal{D}_{k}}\left\vert \hat{d}_{i}^{\intercal}u\right\vert \leq\gamma_{2n_{0}},\label{constraint2:iv}
\end{align}
and define, for each $j\in\{1,\dots,p\}$, the two-stage debiased IV estimator
\begin{equation}\label{eq:theta_tilde_IV}
\tilde{\theta}_{j,k}:=\hat{\theta}_{j,k}+\frac{\hat{u}_{j,k}^{\intercal}}{n_k}\sum_{i\in\mathcal{D}_k}\left(Z_i-X_i^{\intercal}\hat{\theta}_k\right)\hat{d}_i,
\qquad
\tilde{\theta}^{\text{\texttt{iv}}}_{j}:=\frac{1}{K}\sum_{k=1}^{K}\tilde{\theta}_{j,k}.
\end{equation}
Note that the residual in \eqref{eq:theta_tilde_IV} is the \emph{structural} residual $Z_i-X_i^{\intercal}\hat{\theta}_k$, formed with the endogenous regressor, while the weight is the generated regressor $\hat{d}_i$. This is what makes the debiasing term a sample analogue of the moment condition $\mathbb{E}[d_i(Z_i-X_i^{\intercal}\theta_0)\mid Y_i>w_n]=0$.
 
Let $\sigma^2_U=\mathrm{Var}(U_i|Q_i,Y_i>w_n)$ and let $\hat{\sigma}^2_U=(n^{\ast}_0)^{-1}\sum_{k=1}^{K}\sum_{i\in\mathcal{D}_k}(Z_i-X_i^{\intercal}\hat{\theta}_k)^2$.
The asymptotic variance estimator is
\begin{equation}\label{eq:var_iv}
\hat{V}_{1j,\text{\texttt{iv}}}:=\hat{\sigma}^2_{U}\cdot\frac{1}{K^2}\sum_{k=1}^{K}\frac{n^{\ast}_0}{n_k}\left[\hat{u}_{j,k}^{\intercal}\hat{\Sigma}_{d,k}\hat{u}_{j,k}\right].
\end{equation}
As in Section \ref{sec:debias}, the variance retains the ``no-sandwich'' form, because the composite error $U_i$ is conditionally homoskedastic given the instrument. 
Under the classical independent measurement-error model of Proposition~\ref{prop:selection}, $\sigma^2_U=\pi^2/6+\theta_0^{\intercal}\Sigma_{\xi}\theta_0>\pi^2/6$ must be estimated, where $\Sigma_\xi;=\mbox{Var}(\xi_i)$.
 
\subsection{Asymptotic Theory for the Debiased-IV Estimator}\label{sec:2sls_theory}
Let $d_i=A^{\intercal}Q_i$ denote the population generated regressor, $\Sigma_d=\mathbb{E}[d_id_i^{\intercal}\mid Y_i>w_n]$, $\Theta=\Sigma_d^{-1}$, and $m_{\Theta}=\max_{j}\Vert\Theta e_j\Vert_1$.
 
\begin{assumption}[Endogeneity and Instruments]\label{assum:iv}
${}$
\begin{parts}
\item[(i)]The augmented observations $\{(Y_i,X_i,C_i,Q_i)\}_{i=1}^n$ are i.i.d., and the fold randomization is independent of the observations. Equation \eqref{eq:frailty} holds and Assumption \ref{assum:support} holds. The instrument $Q_i\in\mathbb{R}^{p_q}$ has compact support with $\|Q_i\|_{\infty}\le M$, and $p_q\ge p$.
\item[(ii)]
$\mathbb{E}[U_i\mid Q_i,Y_i>w_n]=0$ and
$\mathrm{Var}(U_i\mid Q_i,Y_i>w_n)=\sigma_U^{2}$, where
$0<\underline{\sigma}^{2}\le\sigma_U^{2}\le\bar{\sigma}^{2}<\infty$.
The error $U_i$ is conditionally sub-exponential such that there is a
constant $K_U>0$,
\[
\mathbb{E}\left[\exp\left(\frac{|U_i|}{K_U}\right)\;\middle|\; Q_i,Y_i>w_n\right]\le2 \qquad\text{a.s}.
\]
\item[(iii)]$X_i=A^{\intercal}Q_i+V_i$ with $\mathbb{E}[V_i\mid Q_i,Y_i>w_n]=0$. The columns $a_{\ell}$ of $A$ satisfy $\max_{\ell}\|a_{\ell}\|_0\le s_A$ and $\max_{\ell}\|a_{\ell}\|_1\le C_1$.
\item[(iv)]Let $\Sigma_q:=\mathbb{E}[Q_iQ_i^{\intercal}\mid Y_i>w_n]$, $\Sigma_d:=\mathbb{E}[d_id_i^{\intercal}\mid Y_i>w_n]$ with $d_i=A^{\intercal}Q_i$, and $\Theta:=\Sigma_d^{-1}$. 
The matrices satisfy $\lambda_{\min}(\Sigma_q)\ge C_2^{-1}$ and $\lambda_{\min}(\Sigma_d)\ge C_2^{-1}$, and $\max_j\|\Theta e_j\|_1\le C_3$. 

\item[(v)]$\|\theta_0\|_0\le s_0$ and $\|\theta_0\|_1\le C_1$.

\item[(vi)] With $L_{n_0}:=\sqrt{\log(p\vee p_q)/n_0}$ and $\rho_{n_0}:=\sqrt{s_A}\,L_{n_0}$, the penalty levels satisfy $\lambda_{3,n_0}=c_3L_{n_0}$ and $\lambda_{4,n_0}=c_4\rho_{n_0}$ for sufficiently large fixed constants $c_3,c_4>0$, and the constraint levels are $\gamma_{1n_0}=c'\rho_{n_0}$ and $\gamma_{2n_0}=c''\sqrt{\log n_0}$ for a sufficiently large constant $c'$ and a constant $c''>0$.
\end{parts}
The constants
$M,C_1,C_2,C_3,\underline{\sigma}^{2},\bar{\sigma}^{2}, K_U,\underline{\alpha},\bar{\alpha}$
do not depend on $n$, $p$, $p_q$ or $w_n$.
\end{assumption}
 
Part (i) states the sampling scheme and extends the bounded support of Assumption \ref{assum:support} to the instruments. The order condition $p_q\ge p$ is necessary for $\Sigma_d$ to be nonsingular. Part (ii) is the instrument validity condition in the tail sample. It requires the composite error $U_i=C_i+\varepsilon_i$ to have mean zero and constant variance given the instruments and $Y_i>w_n$. As discussed above, the mean condition is a high-level condition, and Proposition \ref{prop:selection} gives a primitive sufficient condition for it. Since the Gumbel component $\varepsilon_i$ is sub-exponential, the tail condition in part (ii) is a restriction on the confounder $C_i$. Part (iii) specifies a linear and sparse first stage in the tail subpopulation. The $\ell_1$ bound on the columns of $A$ keeps the generated regressors $d_i=A^{\intercal}Q_i$ bounded.

Part (iv) imposes population lower eigenvalue bounds, as in Theorem \ref{Thm:consistency}. The bound on $\Sigma_d$ is the relevance condition for the instruments. Both stages minimize a quadratic loss. The restricted strong convexity of each $\ell_1$-penalized step follows from these lower bounds and entrywise concentration of the corresponding Gram matrices. The first-stage prediction bound follows from the Lasso basic inequality and is transferred across folds by entrywise concentration. The argument permits discrete instruments subject to the remaining assumptions. The bound on $\|\Theta e_j\|_1$ is used in Lemma \ref{lem:iv_projection} in two ways. It establishes feasibility of the population projection direction $\Theta e_j$ for \eqref{constraint1:iv}--\eqref{constraint2:iv}, and it controls the quadratic objective at that direction. Part (v) restates the sparsity of $\theta_0$ and adds an $\ell_1$ bound, which keeps $V_i^{\intercal}\theta_0$ bounded.

Part (vi) sets the tuning rates.
The factor $\sqrt{s_A}$ in $\lambda_{4,n_0}$ and $\gamma_{1n_0}$ is due to the generated regressors. The second-stage loss in \eqref{eq:2sls} treats
\[
Z_i-\hat{d}_i^{\intercal}\theta_0=U_i+V_i^{\intercal}\theta_0+Q_i^{\intercal}(A-\hat{A})\theta_0
\]
as the error term. The first two terms contribute a score of the usual order $L_{n_0}$. The third term contributes a score of the order of the first-stage prediction error $\rho_{n_0}=\sqrt{s_A}L_{n_0}$, which the penalty $\lambda_{4,n_0}$ must dominate. For the same reason, the constraint in \eqref{constraint1:iv} must allow for $\|\hat{\Sigma}_{d,k}-\Sigma_d\|_{\infty}=O_p(\rho_{n_0})$, so that $\gamma_{1n_0}\asymp\rho_{n_0}$. Both orders are obtained from the Cauchy--Schwarz inequality in the empirical prediction norm. Bounding the same terms by H\"older's inequality would replace $\sqrt{s_A}$ by $s_A$ and would strengthen the sparsity condition of Theorem \ref{Thm:2sls} to $s_0s_A^{2}\log(p\vee p_q)\ll\sqrt{n_0}$.
 
\begin{theorem}\label{Thm:2sls}
Suppose that Assumption \ref{assum:iv} holds and that
\[
s_0\,s_A\log\left(p\vee p_q\right)\ll\sqrt{n_0}.
\]
Then, for each $j=1,\dots,p$,
\[
\sqrt{n^{\ast}_0}\,\hat{V}_{1j,\text{\texttt{iv}}}^{-1/2}\left(\tilde{\theta}^{\text{\texttt{iv}}}_{j}-\theta_{0j}\right)\overset{d}{\rightarrow}\mathcal{N}(0,1)
\qquad\text{ as } n_0\rightarrow\infty .
\]
\end{theorem}

The sparsity condition of Theorem \ref{Thm:2sls} differs from that of Theorem \ref{Thm:normality2} in two ways. First, model \eqref{eq:z_frailty} is linear, so the expansion of $\tilde{\theta}^{\text{\texttt{iv}}}_j$ has no second-order remainder. The factor $(\log(p\vee n_0))^{1/2}$ in \eqref{eq:sparsity2}, which is due to that remainder, therefore does not appear. Second, the regressors are estimated. The product of the first-stage and second-stage estimation errors contributes a remainder of order $s_0s_A\log(p\vee p_q)/n_0$, which must be of smaller order than $n_0^{-1/2}$. The binding condition is therefore $s_0s_A\log(p\vee p_q)\ll\sqrt{n_0}$. Conditions that involve the sparsity levels of both stages are familiar from the high-dimensional instrumental variables literature \citep{belloni2012sparse, chernozhukov2015post, gold2020inference}.

Program \eqref{alg:cross_fitting_iv} minimizes the quadratic form $u^{\intercal}\hat{\Sigma}_{d,k}u$ over the feasible set, as in \citet{javanmard2014confidence}. The high-dimensional two-stage least squares literature instead minimizes $\|u\|_1$ subject to a constraint of the form \eqref{constraint1:iv}, as in Program 3.2 of \citet{gold2020inference}. In that literature the precision matrix is estimated on the same observations that are used for debiasing. The resulting remainder is bounded by H\"older's inequality, which requires an $\ell_1$ rate for the estimated projection direction. In our construction the first stage is computed on the auxiliary fold, so the projection direction depends on its own fold only through the instruments. The leading term is then centered conditional on the instruments, and no such rate is needed. The proof of Theorem \ref{Thm:2sls} uses only the two constraints \eqref{constraint1:iv}--\eqref{constraint2:iv} and upper and lower bounds on the minimized quadratic form, which Lemma \ref{lem:iv_projection} provides.

\section{Transforming Coefficients into Tail Probabilities}\label{sec:transformation}
The main text focuses on estimation and inference for the coefficient $\theta_j$, which corresponds to the covariate $x_j$ in the Pareto exponent model
$\alpha(x)=\exp(x^\intercal\theta)$.
However, the coefficient $\theta_j$ itself may not be directly interpretable. A more meaningful quantity is the marginal effect of $x_j$ on the conditional tail probability, obtained by translating the estimated coefficient into its implied effect on the tail probability.

We can transform the estimated $\theta_j$ into an estimate of the marginal effect on the conditional tail probability (or the conditional long-rise probability) as follows.
Let
$$
p(y,x) = \mathbb{P}\left( Y > y \mid Y > w_n, X=x\right)
$$
for shorthand.
Then, the marginal effect of $x_j$ on this quantity evaluated at $(y,x)$ can be written as
$$
    \frac{\partial p(y,x)}{\partial x_j}
    =
    -\theta_j \exp(\theta^\intercal x) \log\left(\frac{y}{w_n}\right)\left(\frac{y}{w_n}\right)^{-\exp(\theta^\intercal x)}.
$$
Since this measure depends on $x$ only through the index $\theta^\intercal x$, we evaluate it by fixing $\theta^\intercal x=\xi$ rather than fixing $x$.
Accordingly, we define the transformed parameter of interest as
$$
m(\xi,\zeta_n)
=
-\exp(\xi)\log(\zeta_n)\zeta_n^{-\exp(\xi)}\theta_j,
$$
where $\zeta_n=y/w_n$, $\xi$ and $\zeta_n>1$ are the tail index and exceedance-ratio benchmarks.
For a binary covariate, $m(\xi,\zeta_n)$ is the derivative of the model's continuous extension.
At a baseline profile with $x_j=0$ and
$\theta^{\intercal}x=\xi$, the exact conditional
probability contrast for changing $x_j$ from zero to one, holding the other covariates fixed, is
\[
\Delta_j(\xi,\zeta_n)
=
\zeta_n^{-\exp(\xi+\theta_j)}
-
\zeta_n^{-\exp(\xi)}.
\]
The quantity $m(\xi,\zeta_n)$ is its first-order
approximation in $\theta_j$.

The standard-error formula below applies to the prespecified benchmark measure.
In this case, the standard error of $\widehat m(\xi,\zeta_n)$ can be obtained simply as
$$
\widehat{se}(\widehat m(\xi,\zeta_n)) = \exp(\xi)\log(\zeta_n)\zeta_n^{-\exp(\xi)} \widehat{se}(\widetilde\theta_j).
$$
In this way, we can present the implied marginal effects on the conditional long-rise probability by appropriately scaling the estimate of $\theta_j$ by the constant
$\exp(\xi)\log(\zeta_n)\zeta_n^{-\exp(\xi)}$.
We now discuss our empirical results in terms of this more intuitive measure.

For the empirical application in Section~\ref{sec:application_text}, the unconditional Pareto exponent is approximately 1 (see Figure~\ref{fig:lgbt_hill}). We therefore set $\xi=0$, corresponding to a Pareto exponent of $\exp(\xi)=1$, as the benchmark subpopulation.
We set $w_n=26$, corresponding to the 96th unconditional quantile in this application. The target evaluation point $y$ is set to the 99th percentile of $Y$, which is $118$. 
Accordingly, we set $\zeta_n = 4.54$. Consequently,
$$
m(\xi,\zeta_n) \approx - 0.333 \theta_j
\qquad\text{and}\qquad
\widehat{se}(\widehat m(\xi,\zeta_n)) = 0.333 \widehat{se}(\widetilde\theta_j).
$$

For the empirical application in Section~\ref{sec:application_health}, the unconditional Pareto exponent is approximately $\exp(1)$ (see Figure~\ref{fig:health_hill}). We therefore set $\xi=1$ as the benchmark subpopulation. We set $w_n=98,\!143$, corresponding to the 99th unconditional quantile in this application. The target evaluation point $y$ is set to the 99.5th percentile of $Y$, which is $140,\!988$. Accordingly, we set $\zeta_n = 1.437$. Consequently,
$$
m(\xi,\zeta_n) = - 0.368 \theta_j
\qquad\text{and}\qquad
\widehat{se}(\widehat m(\xi,\zeta_n)) = 0.368 \widehat{se}(\widetilde\theta_j).
$$

\section{Extension to Large-Scale Online Data}\label{app:online}

This section extends our proposed regularized HDTIR method to accommodate online streaming data, addressing the associated computational challenges posed by large-scale datasets. 
As online data are collected sequentially, we update our estimator at each data point using stochastic gradient descent (SGD). 
Although the X dataset we use in our real data illustration is not large, this extension can be applied to other research that exploits larger social media datasets in future empirical research.
To highlight sequential data generation, we replace the subscript $i$ with $t$ in $\{Y_{t},X_{t}\}$, which remains i.i.d. across $t$.

We modify our baseline HDTIR as follows. 
First, we assume the threshold $w_{n}=\bar{w}$ is predetermined and fixed for the current online streaming data setting. 
If this threshold is selected using an independent pilot sample, the analysis is conditional on that sample and requires the exact conditional Pareto model, with the same parameter $\theta_0$, and the
remaining assumptions to hold uniformly over the possible selected thresholds.

Otherwise, allowing $\bar{w}$ to change with data collection makes asymptotic derivation intractable.
Second, for tail observations $Y_{t}>\bar{w}$, we effectively have a random sample $\{{Y_{t},X_{t}}\}$ from the distribution $F_{Y,X|Y>\bar{w}}$. 

Let $T$ be a deterministic observation horizon known before the stream starts. Denote $T_{0}=\sum_{t=1}^{T} 1\left[ Y_{t}>\bar{w}\right] $ as the effective tail sample size. 
Focusing on the tail observations only, consider the population loss
\begin{equation}\label{eq:radar1}
\bar{\mathcal{L}}(\theta)
=\mathbb{E}\left[
\exp(X_t^{\intercal}\theta)\log(Y_t/\bar{w})-X_t^{\intercal}\theta
\;\middle|\;Y_t>\bar{w}
\right].
\end{equation}
We estimate $\theta_0$ by the Regularization Annealed epoch Dual AveRaging (RADAR) algorithm of \citet{agarwal2012stochastic}, which is a variant of SGD, with the tuning specified below. Like SGD, RADAR uses the stochastic gradient from one observation at each iteration. It proceeds in epochs. Each epoch runs dual averaging around a prox center and passes an updated prox center to the next epoch. Let $K_T$ denote the number of epochs completed within the $T_0$ tail observations. Our estimator is $\hat{\theta}^{\text{\texttt{on}}}:=y_{K_T+1}$, the prox center produced by the last completed epoch. RADAR attains the minimax-optimal rate for the squared $\ell_2$ error of sparse parameters, up to constants. We refer readers to \citet{agarwal2012stochastic} for details of the algorithm.

For debiasing inference, we compute the nodewise Lasso estimators on the retained covariates after the stream ends.
Given a fixed $\bar{w}$, denote the tail-conditional second-moment matrix as
\begin{equation*}
\Sigma _{\bar{w}}=\mathbb{E}\left[ X_{t}X_{t}^{\intercal }|Y_{t}>\bar{w}\right]
\end{equation*}%
and let $\Xi =\Sigma _{\bar{w}}^{-1}$. 
Given an estimate $\hat{\Xi}$ of $\Xi$, which will be discussed shortly, we propose the debiased estimator
\begin{equation*}
\tilde{\theta}^{\text{\texttt{on}}}=\hat{\theta}^{\text{\texttt{on}}} - \frac{1}{T_{0}}\hat{\Xi}X^{\intercal }Z\left( \hat{\theta}^{\text{\texttt{on}}}\right) ,
\end{equation*}%
where $X=[X_{1}^{\intercal },X_{2}^{\intercal },...,X_{T_{0}}^{\intercal }]^{\intercal }$ and $Z(\theta )=\{Z_{1}(\theta ),...,Z_{T_0}(\theta )\}^{\intercal }$ with
\begin{equation*}
Z_{t}(\theta )=\exp (X_{t}^{\intercal }\theta )\log (Y_{t}/\bar{w})-1.
\end{equation*}
Consequently, the $j$-th component of the debiased estimator reads as 
\begin{equation*}
\tilde{\theta}_{j}^{\text{\texttt{on}}}=\hat{\theta}_{j}^{\text{\texttt{on}}}-\frac{\hat{\Xi}_{j}^{\intercal }}{T_{0}}\sum_{t=1}^{T_{0}}\left\{ \exp \left( X_{t}^{\intercal }\hat{\theta}^{\text{\texttt{on}}}\right) \log \left( Y_{t}/\bar{w} \right) -1\right\} X_{t}.
\end{equation*}

\bigskip We construct $\hat{\Xi}$ by first running the nodewise Lasso\footnote{See \citet{meinshausen2006high}, \citet{van2014asymptotically}, \citet{zhang2014confidence} and \citet{caner2018asymptotically}, among others, for the nodewise regression approach.}
\begin{equation}\label{eq:radar2}
\hat{\gamma}^{j}=\underset{\gamma ^{j}\in \mathbb{R}^{p-1}}{\arg \min }\frac{1}{2T_{0}}\left\vert \left\vert X_{\cdot ,j}-X_{\cdot ,-j}\gamma ^{j}\right\vert \right\vert _{2}^{2}+\lambda_{j}\left\vert \left\vert \gamma ^{j}\right\vert \right\vert _{1},
\end{equation}
where $X_{\cdot ,j}$ is the $j$-th column of the matrix $X$, $X_{\cdot,-j}$ is the design submatrix without the $j$-th column, and $\lambda _{j}\asymp \sqrt{\log p/T_{0}}$. 
Now, construct
\begin{equation*}
\hat{\tau}_{j}=\frac{1}{T_{0}}\left( X_{\cdot ,j}-X_{\cdot ,-j}\hat{\gamma}^{j}\right) ^{\intercal }X_{\cdot ,j}.
\end{equation*}%
Given $\hat{\gamma}^{j}$ and $\hat{\tau}_{j}$, the matrix $\Xi $ can be estimated by 
\begin{equation}
\hat{\Xi}=\mathcal{\hat{T}}\times \mathcal{\hat{C}},  \label{eq:radarOmg}
\end{equation}%
where $\mathcal{\hat{T}}=$ \,diag$\left( 1/\hat{\tau}_{1},...,1/\hat{\tau}_{p}\right) $ and 
\begin{equation*}
\mathcal{\hat{C}}=\left( 
\begin{array}{cccc}
1 & -\hat{\gamma}_{2}^{1} & \cdots & -\hat{\gamma}_{p}^{1} \\ 
-\hat{\gamma}_{1}^{2} & 1 & \cdots & -\hat{\gamma}_{p}^{2} \\ 
\vdots & \vdots & \ddots & \vdots \\ 
-\hat{\gamma}_{1}^{p} & -\hat{\gamma}_{2}^{p} & \cdots & 1
\end{array}
\right).
\end{equation*}

The step-by-step algorithm to implement the above procedure is provided on the next page.
\begin{table}[tb]
\centering
\renewcommand{\arraystretch}{0.9}
\small
\begin{tabular}{@{}p{0.97\textwidth}@{}}
\hline\hline
\textbf{Algorithm} Stochastic optimization based estimation and confidence interval for HDTIR \\ \hline
\textbf{Inputs:} RADAR tuning specified below, confidence level $1-\alpha$, and deterministic fixed threshold $\bar{w}>0$. Nodewise penalties $\lambda_j=c_j\sqrt{\log p/T_0}$, where $0<c_*\le c_j\le c^*<\infty$, $c_*$ is sufficiently large, and $c_*,c^*$ do not depend on $p$ or $T$. \\ \hline
Initialize RADAR as specified below and initialize the retained data matrices $X$ and $Y$ as empty. \\
\textbf{for} $t=1$ to $T$ \textbf{do} \\
\hspace*{1.5em}Observe $(Y_t,X_t)$. If $Y_t\le\bar{w}$, discard this observation and continue. \\
\hspace*{1.5em}Otherwise, append $X_t^{\intercal}$ to $X$ and $Y_t$ to $Y$. \\
\hspace*{1.5em}Run one iteration of RADAR using the stochastic gradient at its current iterate $\vartheta$, \\
\hspace*{3em}$g_t(\vartheta)=\left\{\exp(X_t^{\intercal}\vartheta)\log(Y_t/\bar{w})-1\right\}X_t$. \\
\textbf{end for} \\ \hline
Let $T_0$ be the number of retained observations and relabel them $t=1,\dots,T_0$. \\
Set $\hat{\theta}^{\text{\texttt{on}}}=y_{K_T+1}$, the prox center produced by the last completed RADAR epoch. \\
\textbf{for} $j=1$ to $p$ \textbf{do} \\
\hspace*{1.5em}Compute $\hat{\gamma}^{j}$ as an exact minimizer of \eqref{eq:radar2}, using only the retained covariates $X$. \\
\hspace*{1.5em}Compute $\hat{\tau}_j=T_0^{-1}\big(X_{\cdot,j}-X_{\cdot,-j}\hat{\gamma}^{j}\big)^{\intercal}X_{\cdot,j}$. \\
\hspace*{1.5em}Set $\hat{\Xi}_{jj}=1/\hat{\tau}_j$ and $\hat{\Xi}_{j\ell}=-\hat{\gamma}^{j}_{\ell}/\hat{\tau}_j$ for $\ell\neq j$, where the coordinates of $\hat{\gamma}^{j}$ are indexed by $\ell\neq j$. \\
\textbf{end for} \\ \hline
Compute $\hat{\Sigma}=T_0^{-1}X^{\intercal}X$ and $Z(\theta)=\{Z_1(\theta),\dots,Z_{T_0}(\theta)\}^{\intercal}$, where $Z_t(\theta)=\exp(X_t^{\intercal}\theta)\log(Y_t/\bar{w})-1$. \\
Construct the debiased estimator $\tilde{\theta}^{\text{\texttt{on}}}=\hat{\theta}^{\text{\texttt{on}}}-T_0^{-1}\hat{\Xi}X^{\intercal}Z\big(\hat{\theta}^{\text{\texttt{on}}}\big)$. \\ \hline
\textbf{Outputs:} The estimator $\tilde{\theta}^{\text{\texttt{on}}}$ and the asymptotic $(1-\alpha)$ confidence interval for each $\theta_{0j}$, \\
\hspace*{3em}$\tilde{\theta}_{j}^{\text{\texttt{on}}}\pm z_{\alpha/2}\sqrt{(\hat{\Xi}\hat{\Sigma}\hat{\Xi}^{\intercal})_{jj}/T_0}$, where $z_{\alpha/2}=\Phi^{-1}(1-\alpha/2)$. \\
\hline\hline
\end{tabular}
\end{table}
To study the asymptotic properties, we modify our previous assumptions as follows.

\begin{assumption}[Online]\label{assum:online}
\begin{parts}
\item[(i)] The sequence $\{(Y_t,X_t)\}_{t\ge1}$ is i.i.d.\ and its distribution satisfies Assumption \ref{assum:shape} with the fixed threshold $\bar{w}$.
\item[(ii)] For each $j\in\{1,\dots,p\}$, $X_{t,j}$ has compact support $\mathcal{X}_j\subset\mathbb{R}$ with $\sup_{x\in\mathcal{X}_j}|x|\le M$, where $M$ does not depend on $j$, $p$ or $T$.
\item[(iii)] $\|\theta_0\|_0\le s_0$ and $\|\theta_0\|_1\le C_1$.
\item[(iv)] The matrix $\Sigma_{\bar{w}}=\mathbb{E}[X_tX_t^{\intercal}\mid Y_t>\bar{w}]$ satisfies $\lambda_{\min}(\Sigma_{\bar{w}})\ge C_2^{-1}$, and its inverse $\Xi=\Sigma_{\bar{w}}^{-1}$ satisfies $\max_j\|\Xi e_j\|_1\le C_3$ and $\max_j\|\Xi e_j\|_0\le s_{\Xi}$ with $s_{\Xi}\sqrt{\log p/T_0}\to0$ in probability.
\end{parts}
The constants $C_1>0$, $C_2>\max\{1,M^2\}$ and $C_3>0$ do not depend on $p$ or $T$.
\end{assumption}

Parts (i) to (iii) restate Assumptions \ref{assum:shape} and \ref{assum:support} and the sparsity of $\theta_0$ for the tail sample at the fixed threshold $\bar{w}$, and add $\|\theta_0\|_1\le C_1$. Part (iv) restates the lower eigenvalue condition of Theorem \ref{Thm:consistency}. It also bounds the $\ell_1$ norm of the rows of the precision matrix $\Xi$ and imposes their sparsity, which the nodewise Lasso in \eqref{eq:radar2} requires \citep{van2014asymptotically}. This condition is not needed for the cross-fitted estimator of Section \ref{sec:debias}, where the projection direction is obtained from the quadratic program \eqref{alg:cross_fitting}--\eqref{constraint2:cross} rather than from an estimate of $\Xi$.

\paragraph{Tuning of the RADAR algorithm.}
Let $\bar{R}:=2eC_1$ and define
\begin{equation}\label{eq:on_constants}
G:=M\left(e^{M\bar{R}}+1\right),\qquad
\gamma:=\frac{e^{-M\bar{R}}}{C_2},\qquad
\sigma_0:=6Me^{M\bar{R}},\qquad
\sigma:=2e\,\sigma_0 .
\end{equation}
The RADAR algorithm of \citet{agarwal2012stochastic} is run on \eqref{eq:radar1} with the stochastic gradient $g_t(\theta)=\{\exp(X_t^{\intercal}\theta)\log(Y_t/\bar{w})-1\}X_t$, initial prox center $y_1=0$, initial radius $R_1=C_1$, the epoch lengths and step sizes of their Theorem 1 with the constants $(G,\gamma,\sigma)$ in \eqref{eq:on_constants}, and $\omega^{2}=A\log(p\vee T_0)$ for a sufficiently large constant $A$. The proof of Lemma \ref{lem:radar_rates} in the supplement shows that $G$ and $\gamma$ are the Lipschitz and strong convexity constants of the population loss on the region visited by the algorithm, and that $\sigma_0$ and $\sigma$ bound the gradient noise.

\begin{theorem}\label{Thm:online}
Suppose that Assumption \ref{assum:online} holds and that $\tilde{\theta}^{\text{\texttt{on}}}$ is computed with the tuning described above. Suppose also that $p\to\infty$, $T_0\to\infty$ in probability, and, for a sparsity bound $s_0\ge1$,
\[
\frac{s_0\log(p\vee T_0)}{\sqrt{T_0}}\overset{p}{\longrightarrow}0.
\]
Then, for each fixed coordinate $j$,
\[
\frac{\sqrt{T_0}\left(\tilde{\theta}_{j}^{\text{\texttt{on}}}-\theta_{0j}\right) }{ \sqrt{\left(\hat{\Xi}\hat{\Sigma}\hat{\Xi}^{\intercal}\right)_{jj} }}
\overset{d}{\longrightarrow}\mathcal{N}(0,1).
\]
\end{theorem}
\newpage
\newpage
\setcounter{page}{1}
\setcounter{section}{0}
\renewcommand{\thesection}{S}
\renewcommand{\thelemma}{S.\arabic{lemma}}
\renewcommand{\thecorollary}{S.\arabic{corollary}}
\renewcommand{\theequation}{S.\arabic{equation}}
\renewcommand{\thecondition}{S.\arabic{condition}}
\renewcommand{\theproposition}{S.\arabic{proposition}}
\renewcommand{\theHsection}{S}
\renewcommand{\theHequation}{S.\arabic{equation}}
\renewcommand{\theHproposition}{S.\arabic{proposition}}
\renewcommand{\theHlemma}{S.\arabic{lemma}}
\setcounter{proposition}{0}
\setcounter{lemma}{0}\setcounter{corollary}{0}\setcounter{equation}{0}

\section*{Supplementary Appendix}
This supplementary material provides technical details omitted from the main manuscript entitled
``High-Dimensional Tail Index Regression.''
Sections~\ref{app:consistency}--\ref{app:proof_online} contain the proofs of Theorems \ref{Thm:consistency}--\ref{Thm:online} and Proposition \ref{prop:selection}.
Section~\ref{app:lemma} collects auxiliary results.

\section{Proof of Main Theorems}\label{app:theorem}
\subsection{Proof of Theorem \textup{\ref{Thm:consistency}}}\label{app:consistency}

\begin{proof}[Proof of Theorem \textup{\ref{Thm:consistency}}]
Let $S=\{j\colon\theta_{0j}\neq0\}$ with $|S|\le s_0$ and $\Delta:=\hat{\theta}-\theta_0$. Define
\[
W_i:=\exp(X_i^{\intercal}\theta_0)\log(Y_i/w_n)
\qquad\text{and}\qquad
\widehat{\Sigma}_{n_0}:=\frac{1}{n_0}\sum_{i=1}^{n_0}X_iX_i^{\intercal}.
\] We work on the tail subsample of size $n_0$, relabeled $i=1,\dots,n_0$.

The proof applies the oracle inequality of Proposition \ref{prop:lasso_glm} with $\psi=\|\cdot\|_2$, $\psi_0=M\|\cdot\|_1$, $\varsigma=3$ and $\eta=1/2$. That proposition bounds $\|\Delta\|_2$ on an event $\Omega$ defined by two requirements. The penalty $\lambda_{n_0}$ must dominate the score $z^{\ast}=\|\dot{\ell}_{n_0}(\theta_0)\|_{\infty}$, and it must be small relative to the curvature $F$ of the tail log-likelihood on the restricted cone. Step 1 bounds the score by Bernstein's inequality and the curvature by Proposition \ref{prop:RSC}. Step 2 verifies that $\Omega$ holds, which gives the $\ell_2$ bound, and obtains the $\ell_1$ bound from the cone condition. Step 3 obtains the empirical and population prediction bounds from the first-order optimality condition, the local curvature of the likelihood, and entrywise concentration of the weighted and unweighted empirical second-moment matrices.

\medskip\noindent\textit{Step 1 (score and curvature).}
The score at $\theta_0$ is $\dot{\ell}_{n_0}(\theta_0)=n_0^{-1}\sum_{i=1}^{n_0}(W_i-1)X_i$. Its $j$-th coordinate is an average of i.i.d.\ variables that are centered by Lemma \ref{lem:pareto_moments}(ii) and sub-exponential by Lemma \ref{lem:score_tail}, with constants free of $w_n$. Bernstein's inequality \citep[Proposition 5.16]{vershynin2010introduction} and a union bound over $j\le p$ give constants $C_0,c>0$ such that
\begin{equation}\label{eq:score_bound}
z^{\ast}:=\big\|\dot{\ell}_{n_0}(\theta_0)\big\|_{\infty}\le C_0\sqrt{\frac{\log p}{n_0}}
\end{equation}
with probability at least $1-p^{-c}$. We take $\lambda_{n_0}=2C_0\sqrt{(\log p)/n_0}$, so that $z^{\ast}\le\lambda_{n_0}/2$ on this event. For the curvature, Proposition \ref{prop:RSC} gives, with probability approaching one,
\begin{equation}\label{eq:F_bound}
\begin{aligned}
F(\varsigma,S;\|\cdot\|_2,M\|\cdot\|_1)
&\ge \frac{\kappa e^{-2}}{\sqrt{s_0}},\\
F(\varsigma,S;M\|\cdot\|_1,M\|\cdot\|_1)
&\ge \frac{\kappa e^{-2}}{M(1+\varsigma)s_0},
\end{aligned}
\end{equation}
where $\kappa=\frac{\underline{\alpha}}{4e\,\overline{\alpha}C_2}$, and the second bound follows from \eqref{eq:RSC}, since on the cone
\[
\|b_S\|_1\,M\|b\|_1
\le M(1+\varsigma)\|b_S\|_1^2
\le M(1+\varsigma)s_0\|b\|_2^2,
\]
and
$e^{-M^2\|b\|_1^2-M\|b\|_1}\ge e^{-2}$
whenever $M\|b\|_1\le1$.
 
\medskip\noindent\textit{Step 2 (the $\ell_2$ and $\ell_1$ bounds).} On the intersection of the events in \eqref{eq:score_bound}
and \eqref{eq:F_bound},
\[
\frac{\lambda_{n_0}+z^{\ast}}
     {(\lambda_{n_0}-z^{\ast})_{+}}
\le3=\varsigma,
\]
and
\[\frac{\lambda_{n_0}+z^{\ast}}
     {F(\varsigma,S;M\|\cdot\|_1,M\|\cdot\|_1)}
\le
\frac{3C_0M(1+\varsigma)e^{2}}{\kappa}
\sqrt{\frac{s_0^{2}\log p}{n_0}}
\le\eta e^{-\eta^{2}-\eta}
\]
for $n_0$ large, since $s_0^{2}\log p=o(n_0)$.
The event $\Omega$ of Proposition \ref{prop:lasso_glm} therefore holds, and because $\ell_{n_0}$ is convex the proposition yields
\[
\|\Delta\|_2\le\frac{(\lambda_{n_0}+z^{\ast})\,e^{\eta^{2}+\eta}}{F(\varsigma,S;\|\cdot\|_2,M\|\cdot\|_1)}\le\frac{3e^{11/4}}{2\kappa}\,\lambda_{n_0}\sqrt{s_0}\ \lesssim\ \sqrt{\frac{s_0\log p}{n_0}} .
\]
On $\Omega$ the error also lies in the cone, $\|\Delta_{S^{c}}\|_1\le\varsigma\|\Delta_S\|_1$ \citep[Lemma 7 in the supplement]{cai2023statistical}, so
\[
\|\Delta\|_1\le(1+\varsigma)\|\Delta_S\|_1\le(1+\varsigma)\sqrt{s_0}\,\|\Delta\|_2\ \lesssim\ \sqrt{\frac{s_0^{2}\log p}{n_0}} .
\]
 
\medskip\noindent\textit{Step 3 (the prediction bound).}
Let
\[
\widetilde{\Sigma}_{n_0}
:=\frac{1}{n_0}\sum_{i=1}^{n_0}W_iX_iX_i^{\intercal}.
\]
Conditional on the retained covariates, the variables $W_i-1$ are independent, centered, and sub-exponential with a uniformly bounded sub-exponential norm. Since $|X_{i,r}X_{i,s}|\le M^2$, Bernstein's inequality and a union bound over the $p^2$ entries give
\[
\left\|\widetilde{\Sigma}_{n_0}-\widehat{\Sigma}_{n_0}\right\|_{\infty}
\le C\sqrt{\frac{\log p}{n_0}}
\]
with probability at least $1-Cp^{-c}$. We intersect this event with the events used in Steps 1 and 2.

The first-order optimality condition for \eqref{eq:lasso} gives
\[
\dot{\ell}_{n_0}(\hat{\theta})+\lambda_{n_0}\hat{z}=0
\qquad\text{for some}\qquad
\hat{z}\in\partial\|\hat{\theta}\|_1,
\quad\|\hat{z}\|_{\infty}\le1.
\]
Consequently,
\[
\Delta^{\intercal}
\big\{\dot{\ell}_{n_0}(\hat{\theta})-\dot{\ell}_{n_0}(\theta_0)\big\}
\le(\lambda_{n_0}+z^{\ast})\|\Delta\|_1.
\]
By Step 2 and $s_0^2\log p=o(n_0)$, we have $M\|\Delta\|_1\le1$ on the event under consideration for all sufficiently large $n_0$. Using $e^a-1=a\int_0^1e^{ta}\,dt$ and $W_i\ge0$, we obtain
\begin{align*}
&\Delta^{\intercal}
\big\{\dot{\ell}_{n_0}(\theta_0+\Delta)-\dot{\ell}_{n_0}(\theta_0)\big\}\\
&\quad=\frac{1}{n_0}\sum_{i=1}^{n_0}
W_i(X_i^{\intercal}\Delta)^2
\int_0^1\exp(tX_i^{\intercal}\Delta)\,dt
\ \ge\ e^{-1}\Delta^{\intercal}\widetilde{\Sigma}_{n_0}\Delta.
\end{align*}
It follows that
\begin{align*}
\Delta^{\intercal}\widehat{\Sigma}_{n_0}\Delta
&\le\Delta^{\intercal}\widetilde{\Sigma}_{n_0}\Delta
+\left\|\widetilde{\Sigma}_{n_0}-\widehat{\Sigma}_{n_0}\right\|_{\infty}
\|\Delta\|_1^2\\
&\le e(\lambda_{n_0}+z^{\ast})\|\Delta\|_1
+C\sqrt{\frac{\log p}{n_0}}\|\Delta\|_1^2\\
&\lesssim\frac{s_0\log p}{n_0}
+s_0^2\left(\frac{\log p}{n_0}\right)^{3/2}
\ \lesssim\ \frac{s_0\log p}{n_0},
\end{align*}
where the last inequality uses $s_0\sqrt{(\log p)/n_0}\to0$. This proves the empirical prediction bound.

We also record the corresponding population prediction bound. Hoeffding's inequality and a union bound give
\[
\left\|\widehat{\Sigma}_{n_0}-\Sigma_{w_n}\right\|_{\infty}
\le C\sqrt{\frac{\log p}{n_0}}
\]
with probability at least $1-Cp^{-c}$. Thus, on the intersection of the preceding events,
\begin{equation}\label{eq:population_prediction_bound}
\Delta^{\intercal}\Sigma_{w_n}\Delta
\le\Delta^{\intercal}\widehat{\Sigma}_{n_0}\Delta
+\left\|\widehat{\Sigma}_{n_0}-\Sigma_{w_n}\right\|_{\infty}\|\Delta\|_1^2
\ \lesssim\ \frac{s_0\log p}{n_0}.
\end{equation}
All events used have probability approaching one. This completes the proof.

\end{proof}

\subsection{Proof of Theorem \textup{\ref{Thm:normality2}}}

\begin{proof}[Proof of Theorem \textup{\ref{Thm:normality2}}]
Fix $j$. Write $\delta_k:=\hat{\theta}_k-\theta_0$, $n_k=n_0/K$, $W_i:=\exp(X_i^{\intercal}\theta_0)\log(Y_i/w_n)$,
\[
\widehat{\Sigma}_{k}:=\frac{1}{n_k}\sum_{i\in\mathcal{D}_k}X_iX_i^{\intercal},\qquad
\widetilde{\Sigma}_{k}:=\frac{1}{n_k}\sum_{i\in\mathcal{D}_k}W_iX_iX_i^{\intercal},\qquad
v_{j,k}:=\hat{u}_{j,k}^{\intercal}\widehat{\Sigma}_{k}\hat{u}_{j,k}.
\]
Since $n_0/n_k=K$, the variance estimator \eqref{eq:var} equals $\hat{V}_{1j}=K^{-1}\sum_{k=1}^{K}v_{j,k}$.

The proof proceeds in four steps. Step 1 expands $\tilde{\theta}_{j,k}-\theta_{0j}$ into a leading term, which is linear in the centered scores $W_i-1$, and two remainders. Step 2 shows that both remainders are $o_p(n_0^{-1/2})$ under \eqref{eq:sparsity2}. Step 3 shows that the leading term, self-normalized by $\hat{V}_{1j}$, is asymptotically standard normal by a conditional Berry--Esseen bound. Step 4 completes the proof by combining previous steps.

Let
\[
\mathcal{G}:=\sigma\Big(\{X_i\}_{i=1}^{n_0},\ (\mathcal{D}_k)_{k=1}^{K}\Big)
\]
be the $\sigma$-field generated by the covariates of the effective sample and by the fold assignment. Three facts are used. First, each $\hat{u}_{j,k}$ is a function of $\{X_i\}_{i\in\mathcal{D}_k}$ only, because the program \eqref{alg:cross_fitting}--\eqref{constraint2:cross} does not involve the responses. Hence $\hat{u}_{j,k}$, $v_{j,k}$ and $\hat{V}_{1j}$ are $\mathcal{G}$-measurable. Second, by Assumption \ref{assum:shape}, conditional on $X_i$ and $Y_i>w_n$ the variable $W_i$ is $\mathrm{Exp}(1)$, and this law does not depend on $X_i$. Together with independence across $i$, the variables $\{W_i\}_{i=1}^{n_0}$ are i.i.d.\ $\mathrm{Exp}(1)$ and independent of $\mathcal{G}$. In particular $\mathbb{E}[W_i-1\mid\mathcal{G}]=0$, $\mathrm{Var}(W_i\mid\mathcal{G})=1$ and $\mathbb{E}[|W_i-1|^{3}\mid\mathcal{G}]=:\kappa_3<\infty$. Third, given the partition, $\delta_k$ is a function of $\{(X_i,Y_i)\}_{i\in\mathcal{D}_k^{c}}$ and is therefore independent of $\{(X_i,Y_i)\}_{i\in\mathcal{D}_k}$.

Let $\mathcal{E}_n$ be the event on which Lemma \ref{lem:projection}(i), (ii) and (iv) hold for all $k$. This event is $\mathcal{G}$-measurable, $\mathbb{P}(\mathcal{E}_n)\to1$, and on $\mathcal{E}_n$
\begin{equation}\label{eq:Vhat_bounds}
\frac{(1-\gamma_{1n_0})^{2}}{M^{2}}\ \le\ \hat{V}_{1j}\ \le\ C_2+1 .
\end{equation}
By Theorem \ref{Thm:consistency} applied to $\mathcal{D}_k^{c}$, whose size is $(K-1)n_0/K\asymp n_0$,
\begin{equation}\label{eq:delta_rates}
\max_{k}\|\delta_k\|_1=O_p\Big(s_0\sqrt{\frac{\log p}{n_0}}\Big),\qquad
\max_{k}\|\delta_k\|_2^{2}=O_p\Big(\frac{s_0\log p}{n_0}\Big).
\end{equation}

\medskip\noindent\textit{Step 1 (decomposition).}
For $i\in\mathcal{D}_k$, Taylor's theorem with Lagrange remainder applied to $a\mapsto e^{a}$ at $a=X_i^{\intercal}\delta_k$ gives
\[
\exp(X_i^{\intercal}\hat{\theta}_k)=\exp(X_i^{\intercal}\theta_0)\Big\{1+X_i^{\intercal}\delta_k+\tfrac12\,\varrho_i\,(X_i^{\intercal}\delta_k)^{2}\Big\},\qquad
\varrho_i:=\exp\!\big(t_iX_i^{\intercal}\delta_k\big),\ t_i\in(0,1).
\]
Multiplying by $\log(Y_i/w_n)$, averaging over $\mathcal{D}_k$, substituting into \eqref{eq:theta_tilde_MLE} and using $\hat{\theta}_{j,k}-\theta_{0j}=e_j^{\intercal}\delta_k$ gives the exact decomposition
\begin{equation}\label{eq:decomp2}
\tilde{\theta}_{j,k}-\theta_{0j}=-\,S_{j,k}-I_{2,k}-I_{3,k},
\end{equation}
where
\begin{align*}
S_{j,k}&:=\frac{1}{n_k}\sum_{i\in\mathcal{D}_k}(W_i-1)\,\hat{u}_{j,k}^{\intercal}X_i,\\
I_{2,k}&:=\big(\hat{u}_{j,k}^{\intercal}\widetilde{\Sigma}_{k}-e_j^{\intercal}\big)\delta_k,\\
I_{3,k}&:=\frac{1}{2n_k}\sum_{i\in\mathcal{D}_k}\big(\hat{u}_{j,k}^{\intercal}X_i\big)\,W_i\,\varrho_i\,(X_i^{\intercal}\delta_k)^{2}.
\end{align*}
Define the self-normalized leading term
\[
T_n:=\sqrt{n_0}\,\hat{V}_{1j}^{-1/2}\,\frac{1}{K}\sum_{k=1}^{K}S_{j,k}.
\]
Averaging \eqref{eq:decomp2} over $k$ gives
\begin{equation}\label{eq:decomp2_avg}
\sqrt{n_0}\,\hat{V}_{1j}^{-1/2}\left(\tilde{\theta}_{j}-\theta_{0j}\right)
=-\,T_n-\hat{V}_{1j}^{-1/2}\,\frac{\sqrt{n_0}}{K}\sum_{k=1}^{K}\left(I_{2,k}+I_{3,k}\right).
\end{equation}
By \eqref{eq:Vhat_bounds}, $\hat{V}_{1j}^{-1/2}=O_p(1)$. It therefore suffices to show that $\sqrt{n_0}\,I_{2,k}=o_p(1)$ and $\sqrt{n_0}\,I_{3,k}=o_p(1)$ for each $k$, and that $T_n\overset{d}{\rightarrow}\mathcal{N}(0,1)$.

\medskip\noindent\textit{Step 2 (remainders).}
Fix $k$. For $I_{2,k}$, H\"older's inequality, Lemma \ref{lem:projection}(iii) and \eqref{eq:delta_rates} give
\[
\sqrt{n_0}\,|I_{2,k}|\le\sqrt{n_0}\,\big\|\hat{u}_{j,k}^{\intercal}\widetilde{\Sigma}_{k}-e_j^{\intercal}\big\|_{\infty}\,\|\delta_k\|_1
=O_p\Big(\sqrt{n_0}\cdot\sqrt{\frac{\log p}{n_0}}\cdot s_0\sqrt{\frac{\log p}{n_0}}\Big)
=O_p\Big(\frac{s_0\log p}{\sqrt{n_0}}\Big)=o_p(1)
\]
by \eqref{eq:sparsity2}.

For $I_{3,k}$ we first bound the out-of-sample prediction error. Applying \eqref{eq:population_prediction_bound} to the initial estimator trained on $\mathcal{D}_k^c$, whose sample size is $(K-1)n_0/K\asymp n_0$, gives
\[
\max_{k\le K}\delta_k^{\intercal}\Sigma_{w_n}\delta_k
=O_p\Big(\frac{s_0\log p}{n_0}\Big).
\]
Let $\mathcal{H}_k^{0}:=\sigma\big(\{(X_i,Y_i)\}_{i\in\mathcal{D}_k^{c}},(\mathcal{D}_k)_{k=1}^{K}\big)$ and $\mathcal{H}_k:=\sigma\big(\mathcal{H}_k^{0},\{X_i\}_{i\in\mathcal{D}_k}\big)$. By the third fact and Assumption \ref{assum:shape},
\begin{align*}
\mathbb{E}\Big[\frac{1}{n_k}\sum_{i\in\mathcal{D}_k}(X_i^{\intercal}\delta_k)^{2}\ \Big|\ \mathcal{H}_k^{0}\Big]
&=\delta_k^{\intercal}\Sigma_{w_n}\delta_k,\\
\mathbb{E}\Big[\frac{1}{n_k}\sum_{i\in\mathcal{D}_k}W_i(X_i^{\intercal}\delta_k)^{2}\ \Big|\ \mathcal{H}_k\Big]
&=\frac{1}{n_k}\sum_{i\in\mathcal{D}_k}(X_i^{\intercal}\delta_k)^{2}.
\end{align*}
The second identity uses $\mathbb{E}[W_i\mid\mathcal{H}_k]=1$ for $i\in\mathcal{D}_k$. Both variables inside the conditional expectations are nonnegative. Two applications of the conditional Markov inequality\footnote{If $Z\ge0$ and $\mathbb{E}[Z\mid\mathcal{F}]\le R$ with $R=O_p(r_n)$, then $\mathbb{P}(Z>tr_n)\le\mathbb{P}(R>\eta tr_n)+\eta$ for all $t,\eta>0$, so $Z=O_p(r_n)$.} therefore give
\begin{equation}\label{eq:oos_pred}
\frac{1}{n_k}\sum_{i\in\mathcal{D}_k}W_i(X_i^{\intercal}\delta_k)^{2}
=O_p\Big(\frac{s_0\log p}{n_0}\Big).
\end{equation}

Next, by Assumption \ref{assum:support} and \eqref{eq:delta_rates}, $\max_{i}|X_i^{\intercal}\delta_k|\le M\|\delta_k\|_1=o_p(1)$, so $\max_{i\in\mathcal{D}_k}\varrho_i\le2$ with probability approaching one. On that event, since $W_i\ge0$, Lemma \ref{lem:projection}(iv) and \eqref{eq:oos_pred} give
\[
\sqrt{n_0}\,|I_{3,k}|
\le\sqrt{n_0}\,\gamma_{2n_0}\cdot\frac{1}{n_k}\sum_{i\in\mathcal{D}_k}W_i(X_i^{\intercal}\delta_k)^{2}
=O_p\Big(\frac{s_0\log p\sqrt{\log n_0}}{\sqrt{n_0}}\Big)=o_p(1)
\]
by Assumption \ref{assum:rate_tuning}(iii) and \eqref{eq:sparsity2}. This is the only place where the bound $\gamma_{2n_0}\asymp\sqrt{\log n_0}$ enters a rate, and it is the source of the factor $(\log(p\vee n_0))^{3/2}$ in \eqref{eq:sparsity2}.

\medskip\noindent\textit{Step 3 (leading term).}
Write $k(i)$ for the fold containing $i$. Since $Kn_k=n_0$,
\[
T_n=\sum_{i=1}^{n_0}\zeta_i,\qquad
\zeta_i:=\frac{(W_i-1)\,\hat{u}_{j,k(i)}^{\intercal}X_i}{\sqrt{n_0\hat{V}_{1j}}} .
\]
By the first two facts, conditional on $\mathcal{G}$ the summands $\zeta_i$ are independent and centered, with
\[
\sum_{i=1}^{n_0}\mathrm{Var}(\zeta_i\mid\mathcal{G})=\frac{1}{n_0\hat{V}_{1j}}\sum_{k=1}^{K}\sum_{i\in\mathcal{D}_k}\big(\hat{u}_{j,k}^{\intercal}X_i\big)^{2}
=\frac{1}{n_0\hat{V}_{1j}}\sum_{k=1}^{K}n_kv_{j,k}=1 .
\]
The conditional variance equals one because $\mathrm{Var}(W_i\mid\mathcal{G})=1$. This is the ``no-sandwich'' property of the Pareto score. On $\mathcal{E}_n$, by Lemma \ref{lem:projection}(iv) and \eqref{eq:Vhat_bounds},
\begin{align*}
\sum_{i=1}^{n_0}\mathbb{E}\big[|\zeta_i|^{3}\mid\mathcal{G}\big]
&\le\frac{\kappa_3\,\gamma_{2n_0}}{(n_0\hat{V}_{1j})^{3/2}}\sum_{i=1}^{n_0}\big(\hat{u}_{j,k(i)}^{\intercal}X_i\big)^{2}
=\frac{\kappa_3\,\gamma_{2n_0}}{(n_0\hat{V}_{1j})^{1/2}}
\le\frac{\kappa_3\,M\,c''\sqrt{\log n_0}}{(1-\gamma_{1n_0})\sqrt{n_0}}=:r_n\to0 .
\end{align*}
The Berry--Esseen theorem for independent non-identically distributed summands \citep[Theorem 5.4]{petrov1995limit} therefore gives, on $\mathcal{E}_n$,
\begin{equation}\label{eq:cond_BE}
\sup_{x\in\mathbb{R}}\Big|\,\mathbb{P}\big(T_n\le x\mid\mathcal{G}\big)-\Phi(x)\Big|\ \le\ C_{\mathrm{BE}}\,r_n
\end{equation}
for a universal constant $C_{\mathrm{BE}}\geq 1$. Taking expectations, for every $x$,
\[
\big|\mathbb{P}(T_n\le x)-\Phi(x)\big|
\le\mathbb{E}\Big[\big|\mathbb{P}(T_n\le x\mid\mathcal{G})-\Phi(x)\big|\,\mathbf{1}_{\mathcal{E}_n}\Big]+\mathbb{P}(\mathcal{E}_n^{c})
\le C_{\mathrm{BE}}\,r_n+\mathbb{P}(\mathcal{E}_n^{c})\to0 ,
\]
where the limit follows because $r_n \rightarrow 0$ in probability and the integrand is bounded.
Thus, $T_n\overset{d}{\rightarrow}\mathcal{N}(0,1)$. Because $T_n$ is normalized by the $\mathcal{G}$-measurable quantity $\hat{V}_{1j}$, only the bounds \eqref{eq:Vhat_bounds} are needed. No convergence of $\hat{V}_{1j}$ to a non-random limit is required.

\medskip\noindent\textit{Step 4 (conclusion).}
In \eqref{eq:decomp2_avg}, $T_n\overset{d}{\rightarrow}\mathcal{N}(0,1)$ by Step 3, and the remainder is $O_p(1)\cdot o_p(1)=o_p(1)$ by Step 2 and \eqref{eq:Vhat_bounds}. Slutsky's theorem and the symmetry of the normal law give
\[
\sqrt{n_0}\,\hat{V}_{1j}^{-1/2}\left(\tilde{\theta}_{j}-\theta_{0j}\right)\overset{d}{\rightarrow}\mathcal{N}(0,1).
\]
\end{proof}
\subsection{Proof of Proposition \textup{\ref{prop:selection}}}
\begin{proof}[Proof of Proposition \textup{\ref{prop:selection}}]
We drop the subscript $i$ for notational simplicity. Write $S=\mathbf{1}[Y>w_n]$, $\alpha(x,c)=\exp(x^{\intercal}\theta_0+c)$ and $\pi_n(x,c)=\mathbb{P}(Y>w_n\mid X=x,C=c)$.
Our proof proceeds with 2 steps. The proof has two steps. Step 1 proves \eqref{eq:tilt}. Step 2 proves the second claim, which concerns classical measurement error in the covariates.

\medskip\noindent\textit{Step 1 (the tilt identity \eqref{eq:tilt}).}
Two consequences of the imposed assumptions are used. First, by the exclusion restriction $Q\perp Y\mid(X,C)$, the conditional law of $Y$ given $(X,C,Q)$ equals its conditional law given $(X,C)$. In particular $\mathbb{P}(S=1\mid X,C,Q)=\pi_n(X,C)$, and by \eqref{eq:frailty} the conditional law of $Y$ given $(X,C,Q,S=1)$ is Pareto with scale $w_n$ and exponent $\alpha(X,C)$. Second, let $E:=\alpha(X,C)\log(Y/w_n)$. Conditional on $(X,C,Q,S=1)$ the variable $E$ is standard exponential, since $\mathbb{P}(E>e\mid X,C,Q,S=1)=\mathbb{P}(Y/w_n>\exp(e/\alpha(X,C))\mid X,C,Q,S=1)=e^{-e}$ for $e>0$. The error in \eqref{eq:z_frailty} is $\varepsilon=-(\log E+\eta)$, and $\mathbb{E}[\log E]=-\eta$ for a standard exponential variable. Hence
\begin{equation}\label{eq:eps_zero}
\mathbb{E}[\varepsilon\mid X,C,Q,S=1]=0\qquad\text{and}\qquad\mathbb{E}[\varepsilon\mid Q,S=1]=0 .
\end{equation}
For the confounder, Bayes' rule gives, for any bounded measurable $h$,
\[
\mathbb{E}[h(X,C)\mid Q,S=1]=\frac{\mathbb{E}[h(X,C)\,\mathbb{P}(S=1\mid X,C,Q)\mid Q]}{\mathbb{P}(S=1\mid Q)}
=\frac{\mathbb{E}[h(X,C)\,\pi_n(X,C)\mid Q]}{\mathbb{E}[\pi_n(X,C)\mid Q]},
\]
where the second equality uses $\mathbb{P}(S=1\mid X,C,Q)=\pi_n(X,C)$. Taking $h(x,c)=c$, with a truncation argument if $C$ is unbounded, and adding \eqref{eq:eps_zero} yields \eqref{eq:tilt}, since $U=C+\varepsilon$.

\medskip\noindent\textit{Step 2 (the measurement-error case).}
Let $X=X^{*}+\xi$ with $\xi\perp(X^{*},Y,Q)$ and $\mathbb{E}[\xi]=0$, and let $C=-\xi^{\intercal}\theta_0$. Then $\exp(X^{*\intercal}\theta_0)=\exp(X^{\intercal}\theta_0+C)=\alpha(X,C)$ is a function of $(X,C)$. Since $\xi\perp(X^{*},Y)$, the conditional law of $Y$ given $(X^{*},\xi)$ is its conditional law given $X^{*}$, which is Pareto with exponent $\exp(X^{*\intercal}\theta_0)$ on $\{Y>w_n\}$. Conditioning further on $(X,C)$, which is a function of $(X^{*},\xi)$, shows that \eqref{eq:frailty} holds with this $C$. By Step 1, $\mathbb{E}[\varepsilon|Q,S=1]=0$. For the confounder, $S$ is a function of $Y$ and $\xi\perp(Y,Q)$, so $\xi\perp(Q,S)$ and
\[
\mathbb{E}[C\mid Q,S=1]=-\theta_0^{\intercal}\mathbb{E}[\xi\mid Q,S=1]=-\theta_0^{\intercal}\mathbb{E}[\xi]=0 .
\]
Hence $\mathbb{E}[U|Q,S=1]=0$. No property of $w_n$ was used, so the identity holds for every threshold.

\end{proof}

\subsection{Proof of Theorem \textup{\ref{Thm:2sls}}}

\begin{proof}[Proof of Theorem \textup{\ref{Thm:2sls}}]
Fix $j$. Write $\delta_k:=\hat{\theta}_k-\theta_0$, $v_{j,k}:=\hat{u}_{j,k}^{\intercal}\hat{\Sigma}_{d,k}\hat{u}_{j,k}$ and $\bar{V}_j:=K^{-1}\sum_{k=1}^{K}v_{j,k}$. Since $n_0^{\ast}/n_k=K$, the variance estimator \eqref{eq:var_iv} equals $\hat{V}_{1j,\text{\texttt{iv}}}=\hat{\sigma}_U^{2}\bar{V}_j$. For $b\in\mathbb{R}^{p_q}$ let $\|b\|_{n,\mathcal{D}_k}^{2}:=n_k^{-1}\sum_{i\in\mathcal{D}_k}(Q_i^{\intercal}b)^{2}$.

The proof proceeds in four steps. Step 1 decomposes $\tilde{\theta}_{j,k}-\theta_{0j}$ into a leading term, which is linear in $U_i$, and a remainder that is linear in $\delta_k$. Step 2 shows that the remainder is $o_p(n_0^{-1/2})$ under the product sparsity condition. Step 3 shows that the leading term, normalized by $\sigma_U\bar{V}_j^{1/2}$, is asymptotically standard normal by a conditional Berry--Esseen bound. Step 4 shows that $\hat{\sigma}_U^{2}$ is consistent and combines the steps.

Let
\begin{equation}\label{eq:Gfield}
\mathcal{G}:=\sigma\Big(\{(X_i,Y_i,Q_i)\}_{i\in\mathcal{D}_0},\ \big\{Q_i,\mathbf{1}[Y_i>w_n]\big\}_{i\notin\mathcal{D}_0},\ (\mathcal{D}_k)_{k=0}^{K}\Big)
\end{equation}
be the $\sigma$-field generated by the data of the auxiliary fold, by the instruments and selection indicators of the remaining observations, and by the fold assignment. Three facts are used. First, $\hat{A}$ is computed on $\mathcal{D}_0$, so $\hat{d}_i=\hat{A}^{\intercal}Q_i$ is $\mathcal{G}$-measurable for $i\notin\mathcal{D}_0$. The program \eqref{alg:cross_fitting_iv}--\eqref{constraint2:iv} depends on the data only through $\{\hat{d}_i\}_{i\in\mathcal{D}_k}$, so $\hat{u}_{j,k}$, $v_{j,k}$ and $\bar{V}_j$ are $\mathcal{G}$-measurable as well. 
Second, by Assumption \ref{assum:iv}(ii) and (iii), for $i\notin\mathcal{D}_0$ we have
$\mathbb{E}[U_i\mid\mathcal{G}]=0$, $\mathrm{Var}(U_i\mid\mathcal{G})=\sigma_U^{2}$, $\mathbb{E}[V_i\mid\mathcal{G}]=0$ and $\mathbb{E}[\exp(|U_i|/K_U)\mid\mathcal{G}]\le2$. Since $x^{3}\le3!\,e^{x}$ for $x\ge0$,
\[
\mathbb{E}[|U_i|^{3}\mid\mathcal{G}]
\le 3!K_U^{3}\,
\mathbb{E}\left[
\exp\left(\frac{|U_i|}{K_U}\right)
\;\middle|\;\mathcal{G}
\right]
\le12K_U^{3}=:\kappa_3 .
\]
where $\kappa_3$ is determined by the Orlicz bound, and $\mathbb{E}[V_i\mid\mathcal{G}]=0$. Third, given $\mathcal{G}$ the pairs $\{(U_i,V_i)\}_{i\notin\mathcal{D}_0}$ are independent across $i$.

Let $\mathcal{E}_n$ be the event on which Lemma \ref{lem:2sls_rates}(ii), (iii) and Lemma \ref{lem:iv_projection} hold for all $k$. This event is $\mathcal{G}$-measurable, $\mathbb{P}(\mathcal{E}_n)\to1$, and on $\mathcal{E}_n$
\begin{equation}\label{eq:Vbar_bounds}
\frac{(1-\gamma_{1n_0})^{2}}{4M^{2}C_1^{2}}\ \le\ \bar{V}_j\ \le\ C_2+1 .
\end{equation}
Let $\mathcal{E}_n^{\ast}\subseteq\mathcal{E}_n$ be the event on which in addition Lemma \ref{lem:2sls_rates}(iv) holds for all $k$. Then $\mathbb{P}(\mathcal{E}_n^{\ast})\to1$ as well. The condition $s_0s_A\log(p\vee p_q)\ll\sqrt{n_0}$ gives $\rho_{n_0}\to0$, $s_0\rho_{n_0}\to0$ and
\begin{equation}\label{eq:binding}
\sqrt{n_0}\,s_0\rho_{n_0}^{2}=\frac{s_0s_A\log(p\vee p_q)}{\sqrt{n_0}}\to0 .
\end{equation}

\medskip\noindent\textit{Step 1 (decomposition).}
By \eqref{eq:z_frailty}, $Z_i-X_i^{\intercal}\hat{\theta}_k=U_i-X_i^{\intercal}\delta_k$. Substituting into \eqref{eq:theta_tilde_IV} and using $\hat{\theta}_{j,k}-\theta_{0j}=e_j^{\intercal}\delta_k$,
\begin{equation}\label{eq:2sls_decomp}
\tilde{\theta}_{j,k}-\theta_{0j}=T_{1,k}-T_{2,k},\qquad
T_{1,k}:=\frac{\hat{u}_{j,k}^{\intercal}}{n_k}\sum_{i\in\mathcal{D}_k}\hat{d}_iU_i,\qquad
T_{2,k}:=\Big(\frac{\hat{u}_{j,k}^{\intercal}}{n_k}\sum_{i\in\mathcal{D}_k}\hat{d}_iX_i^{\intercal}-e_j^{\intercal}\Big)\delta_k .
\end{equation}
Define the normalized leading term
\[
T_n:=\frac{\sqrt{n_0^{\ast}}}{\sigma_U\bar{V}_j^{1/2}}\cdot\frac{1}{K}\sum_{k=1}^{K}T_{1,k}.
\]
Averaging \eqref{eq:2sls_decomp} over $k$ gives
\begin{equation}\label{eq:2sls_main}
\sqrt{n_0^{\ast}}\,\hat{V}_{1j,\text{\texttt{iv}}}^{-1/2}\left(\tilde{\theta}^{\text{\texttt{iv}}}_{j}-\theta_{0j}\right)
=\Big(\frac{\sigma_U^{2}}{\hat{\sigma}_U^{2}}\Big)^{1/2}\left\{T_n-\frac{\sqrt{n_0^{\ast}}}{\sigma_U\bar{V}_j^{1/2}}\cdot\frac{1}{K}\sum_{k=1}^{K}T_{2,k}\right\}.
\end{equation}

By \eqref{eq:Vbar_bounds} and
$\sigma_U^{2}\ge\underline{\sigma}^{2}>0$,
the factor $(\sigma_U\bar{V}_j^{1/2})^{-1}$ is $O_p(1)$. It therefore suffices to show that $\sqrt{n_0^{\ast}}\,T_{2,k}=o_p(1)$ for each $k$, that $T_n\overset{d}{\rightarrow}\mathcal{N}(0,1)$, and that $\hat{\sigma}_U^{2}\overset{p}{\rightarrow}\sigma_U^{2}$.

\medskip\noindent\textit{Step 2 (remainder).}
Fix $k$ and work on $\mathcal{E}_n^{\ast}$. By Assumption \ref{assum:iv}(iii), $X_i=\hat{d}_i+V_i+(A-\hat{A})^{\intercal}Q_i$, so
\begin{align*}
\frac{\hat{u}_{j,k}^{\intercal}}{n_k}\sum_{i\in\mathcal{D}_k}\hat{d}_iX_i^{\intercal}-e_j^{\intercal}
&=\big(\hat{u}_{j,k}^{\intercal}\hat{\Sigma}_{d,k}-e_j^{\intercal}\big)
+\frac{\hat{u}_{j,k}^{\intercal}}{n_k}\sum_{i\in\mathcal{D}_k}\hat{d}_iV_i^{\intercal}
+\frac{\hat{u}_{j,k}^{\intercal}}{n_k}\sum_{i\in\mathcal{D}_k}\hat{d}_iQ_i^{\intercal}(A-\hat{A})\\
&=:r^{\intercal}+R_1^{\intercal}+R_2^{\intercal},
\end{align*}
and $|T_{2,k}|\le\|r\|_{\infty}\|\delta_k\|_1+|R_1^{\intercal}\delta_k|+|R_2^{\intercal}\delta_k|$. Throughout this step, $\|\delta_k\|_1\le Cs_0\rho_{n_0}$ by Lemma \ref{lem:2sls_rates}(iv).

The first term is bounded by the constraint \eqref{constraint1:iv}, which gives $\|r\|_{\infty}\le\gamma_{1n_0}=c'\rho_{n_0}$. Hence $\sqrt{n_0^{\ast}}\,\|r\|_{\infty}\|\delta_k\|_1\le C\sqrt{n_0}\,s_0\rho_{n_0}^{2}\to0$ by \eqref{eq:binding}.

For the second term, fix $\ell\le p$ and set $\zeta_i:=(\hat{u}_{j,k}^{\intercal}\hat{d}_i)V_{i\ell}$ for $i\in\mathcal{D}_k$. Given $\mathcal{G}$ the weights $\hat{u}_{j,k}^{\intercal}\hat{d}_i$ are fixed and the $\zeta_i$ are independent and centered. By \eqref{constraint2:iv}, Lemma \ref{lem:iv_projection}(ii) and $|V_{i\ell}|\le M(1+C_1)$,
\[
\max_{i\in\mathcal{D}_k}|\zeta_i|\le M(1+C_1)\gamma_{2n_0},\qquad
\frac{1}{n_k}\sum_{i\in\mathcal{D}_k}\mathrm{Var}(\zeta_i\mid\mathcal{G})\le M^{2}(1+C_1)^{2}\,v_{j,k}\le M^{2}(1+C_1)^{2}(C_2+1).
\]
Bernstein's inequality for bounded independent summands and a union bound over $\ell\le p$ give, with conditional probability at least $1-p^{-c}$,
\[
\|R_1\|_{\infty}\le C\Big(\sqrt{\frac{\log p}{n_k}}+\frac{\gamma_{2n_0}\log p}{n_k}\Big)\le C\,L_{n_0},
\]
where the second inequality uses $\gamma_{2n_0}\log p/n_0=o(L_{n_0})$. The bound then holds unconditionally with the same probability. Hence
\[
\sqrt{n_0^{\ast}}\,|R_1^{\intercal}\delta_k|\le\sqrt{n_0^{\ast}}\,\|R_1\|_{\infty}\|\delta_k\|_1
\le C\sqrt{n_0}\,L_{n_0}\,s_0\rho_{n_0}=C\,\frac{s_0\sqrt{s_A}\log(p\vee p_q)}{\sqrt{n_0}}\to0 .
\]

For the third term, let $b_k:=(A-\hat{A})\delta_k=\sum_{\ell=1}^{p}\delta_{k\ell}(a_{\ell}-\hat{a}_{\ell})\in\mathbb{R}^{p_q}$, so that $R_2^{\intercal}\delta_k=n_k^{-1}\sum_{i\in\mathcal{D}_k}(\hat{u}_{j,k}^{\intercal}\hat{d}_i)(Q_i^{\intercal}b_k)$. The Cauchy--Schwarz inequality, Lemma \ref{lem:iv_projection}(ii), the triangle inequality for $\|\cdot\|_{n,\mathcal{D}_k}$ and Lemma \ref{lem:2sls_rates}(ii) give
\begin{align*}
|R_2^{\intercal}\delta_k|
&\le\Big(\frac{1}{n_k}\sum_{i\in\mathcal{D}_k}(\hat{u}_{j,k}^{\intercal}\hat{d}_i)^{2}\Big)^{1/2}\Big(\frac{1}{n_k}\sum_{i\in\mathcal{D}_k}(Q_i^{\intercal}b_k)^{2}\Big)^{1/2}
=v_{j,k}^{1/2}\,\|b_k\|_{n,\mathcal{D}_k}\\
&\le(C_2+1)^{1/2}\,\|\delta_k\|_1\max_{\ell}\|\hat{a}_{\ell}-a_{\ell}\|_{n,\mathcal{D}_k}
\le C\,s_0\rho_{n_0}^{2}.
\end{align*}
Hence $\sqrt{n_0^{\ast}}\,|R_2^{\intercal}\delta_k|\le C\sqrt{n_0}\,s_0\rho_{n_0}^{2}\to0$ by \eqref{eq:binding}.

Combining the three bounds, $\sqrt{n_0^{\ast}}\,T_{2,k}=o_p(1)$ for each $k$.

\medskip\noindent\textit{Step 3 (leading term).}
Write $k(i)$ for the fold containing $i\notin\mathcal{D}_0$. Since $n_0^{\ast}=Kn_k$,
\[
T_n=\sum_{i\notin\mathcal{D}_0}\xi_i,\qquad
\xi_i:=\frac{(\hat{u}_{j,k(i)}^{\intercal}\hat{d}_i)\,U_i}{\sigma_U\sqrt{n_0^{\ast}\bar{V}_j}} .
\]
By the three facts, conditional on $\mathcal{G}$ the summands $\xi_i$ are independent and centered, with
\[
\sum_{i\notin\mathcal{D}_0}\mathrm{Var}(\xi_i\mid\mathcal{G})
=\frac{1}{n_0^{\ast}\bar{V}_j}\sum_{k=1}^{K}\sum_{i\in\mathcal{D}_k}(\hat{u}_{j,k}^{\intercal}\hat{d}_i)^{2}
=\frac{1}{n_0^{\ast}\bar{V}_j}\sum_{k=1}^{K}n_kv_{j,k}=1 .
\]
The conditional variance equals one because $U_i$ is conditionally homoskedastic given the instruments. On $\mathcal{E}_n$, by \eqref{constraint2:iv}, \eqref{eq:Vbar_bounds} and $\mathbb{E}[|U_i|^{3}\mid\mathcal{G}]\le\kappa_3$,
\begin{align*}
\sum_{i\notin\mathcal{D}_0}\mathbb{E}\big[|\xi_i|^{3}\mid\mathcal{G}\big]
&\le\frac{\kappa_3\,\gamma_{2n_0}}{\sigma_U^{3}(n_0^{\ast}\bar{V}_j)^{3/2}}\sum_{k=1}^{K}\sum_{i\in\mathcal{D}_k}(\hat{u}_{j,k}^{\intercal}\hat{d}_i)^{2}
=\frac{\kappa_3\,\gamma_{2n_0}}{\sigma_U^{3}(n_0^{\ast}\bar{V}_j)^{1/2}}\\
&\le\frac{2MC_1\kappa_3\,c''\sqrt{\log n_0}}{\sigma_U^{3}(1-\gamma_{1n_0})\sqrt{n_0^{\ast}}}=:r_n\to0 .
\end{align*}
The Berry--Esseen theorem for independent non-identically distributed summands \citep[Theorem 5.4]{petrov1995limit} therefore gives, on $\mathcal{E}_n$,
\begin{equation}\label{eq:cond_BE_iv}
\sup_{x\in\mathbb{R}}\Big|\mathbb{P}\big(T_n\le x\mid\mathcal{G}\big)-\Phi(x)\Big|\le C_{\mathrm{BE}}\,r_n
\end{equation}
for a universal constant $C_{\mathrm{BE}}$. Taking expectations, for every $x$,
\[
\big|\mathbb{P}(T_n\le x)-\Phi(x)\big|
\le\mathbb{E}\Big[\big|\mathbb{P}(T_n\le x\mid\mathcal{G})-\Phi(x)\big|\,\mathbf{1}_{\mathcal{E}_n}\Big]+\mathbb{P}(\mathcal{E}_n^{c})
\le C_{\mathrm{BE}}\,r_n+\mathbb{P}(\mathcal{E}_n^{c})\to0 ,
\]
so $T_n\overset{d}{\rightarrow}\mathcal{N}(0,1)$. Because $T_n$ is normalized by the $\mathcal{G}$-measurable quantity $\bar{V}_j$, only the bounds \eqref{eq:Vbar_bounds} are needed.

\medskip\noindent\textit{Step 4 (variance estimator and conclusion).}
Using $Z_i-X_i^{\intercal}\hat{\theta}_k=U_i-X_i^{\intercal}\delta_k$,
\begin{align*}
\hat{\sigma}_U^{2}
&=\frac{1}{n_0^{\ast}}\sum_{k=1}^{K}\sum_{i\in\mathcal{D}_k}U_i^{2}
-\frac{2}{n_0^{\ast}}\sum_{k=1}^{K}\sum_{i\in\mathcal{D}_k}U_iX_i^{\intercal}\delta_k
+\frac{1}{n_0^{\ast}}\sum_{k=1}^{K}\sum_{i\in\mathcal{D}_k}(X_i^{\intercal}\delta_k)^{2}
=:A_1-2A_2+A_3 .
\end{align*}
By Assumption \ref{assum:iv}(ii) the $U_i^{2}$ are i.i.d.\ with mean $\sigma_U^{2}$ and finite variance, so $A_1\overset{p}{\rightarrow}\sigma_U^{2}$. By Assumption \ref{assum:support} and Lemma \ref{lem:2sls_rates}(iv), $|X_i^{\intercal}\delta_k|\le M\|\delta_k\|_1\le CMs_0\rho_{n_0}$, so $A_3\le C^{2}M^{2}s_0^{2}\rho_{n_0}^{2}=o_p(1)$, and $|A_2|\le(A_1A_3)^{1/2}=o_p(1)$ by the Cauchy--Schwarz inequality. Hence $\hat{\sigma}_U^{2}\overset{p}{\rightarrow}\sigma_U^{2}$ and $\sigma_U^{2}/\hat{\sigma}_U^{2}\overset{p}{\rightarrow}1$.

In \eqref{eq:2sls_main} the prefactor converges to one in probability, $T_n\overset{d}{\rightarrow}\mathcal{N}(0,1)$ by Step 3, and the remainder is $O_p(1)\cdot o_p(1)=o_p(1)$ by Step 2 and \eqref{eq:Vbar_bounds}. Slutsky's theorem gives
\[
\sqrt{n_0^{\ast}}\,\hat{V}_{1j,\text{\texttt{iv}}}^{-1/2}\left(\tilde{\theta}^{\text{\texttt{iv}}}_{j}-\theta_{0j}\right)\overset{d}{\rightarrow}\mathcal{N}(0,1).
\]
\end{proof}
 
\subsection{Proof of Theorem \textup{\ref{Thm:online}}}\label{app:proof_online}

\begin{proof}[Proof of Theorem \textup{\ref{Thm:online}}]
Fix $j$. Write $\delta:=\hat{\theta}^{\text{\texttt{on}}}-\theta_0$, $W_t:=\exp(X_t^{\intercal}\theta_0)\log(Y_t/\bar{w})$, $\Sigma:=\Sigma_{\bar{w}}$, $\Xi:=\Sigma^{-1}$,
\[
\hat{\Sigma}:=\frac{1}{T_0}\sum_{t=1}^{T_0}X_tX_t^{\intercal},\qquad
\widetilde{\Sigma}:=\frac{1}{T_0}\sum_{t=1}^{T_0}W_tX_tX_t^{\intercal},\qquad
v_j:=\hat{\Xi}_j^{\intercal}\hat{\Sigma}\hat{\Xi}_j=\big(\hat{\Xi}\hat{\Sigma}\hat{\Xi}^{\intercal}\big)_{jj},
\]
where $\hat{\Xi}_j$ is the $j$-th row of $\hat{\Xi}$ written as a column vector.

The tail sample size $T_0$ is random, so we first reduce the problem to a deterministic sample size. Let $\mathcal{S}:=\sigma\big(\mathbf{1}[Y_t>\bar{w}],\,t\le T\big)$, so that $T_0$ is $\mathcal{S}$-measurable. The observations are i.i.d.\ and each indicator depends only on its own $Y_t$. Hence, conditional on $\mathcal{S}$, the retained observations are i.i.d.\ draws from the law of $(X_t,Y_t)$ given $Y_t>\bar{w}$. The rate conditions of the theorem and of Assumption \ref{assum:online}(iv) hold in probability. Every subsequence therefore has a further subsequence along which $T_0\to\infty$, $s_0\log(p\vee T_0)/\sqrt{T_0}\to0$ and $s_{\Xi}\sqrt{(\log p)/T_0}\to0$ almost surely. Along such a subsequence and conditional on $\mathcal{S}$, $T_0$ is a deterministic sequence with these properties.

The remaining argument follows the proof of Theorem \ref{Thm:normality2}. Steps 1 to 3 treat $T_0$ as such a deterministic sequence, and all probabilities there are conditional on $\mathcal{S}$. Step 1 decomposes $\tilde{\theta}_j^{\text{\texttt{on}}}-\theta_{0j}$ into a leading term, linear in $W_t-1$, and two remainders. Step 2 shows that both remainders are $o_p(T_0^{-1/2})$. Step 3 shows that the leading term, normalized by $v_j$, is asymptotically standard normal. Step 4 removes the conditioning and passes from subsequences to the full sequence.

Let $\mathcal{G}:=\sigma(\mathcal{S},X_1,\dots,X_{T_0})$. Three facts are used. First, $\hat{\Xi}$ is computed from the retained covariates alone, so $\hat{\Xi}_j$ and $v_j$ are $\mathcal{G}$-measurable. Second, by Assumption \ref{assum:online}(i) and Lemma \ref{lem:pareto_moments}, the variables $W_1,\dots,W_{T_0}$ are i.i.d.\ $\mathrm{Exp}(1)$ and independent of $\mathcal{G}$. Hence $\mathbb{E}[W_t-1\mid\mathcal{G}]=0$, $\mathrm{Var}(W_t\mid\mathcal{G})=1$, $\mathbb{E}[|W_t-1|^{3}\mid\mathcal{G}]=:\kappa_3<\infty$ and $\mathbb{E}[\widetilde{\Sigma}\mid\mathcal{G}]=\hat{\Sigma}$. Third, by Lemma \ref{lem:radar_rates}, whose conditions hold because $s_0\ge1$, with probability approaching one
\begin{equation}\label{eq:on_rates}
\begin{aligned}
\|\delta\|_1&\le Cs_0\sqrt{\frac{\log(p\vee T_0)}{T_0}},\qquad
\|\delta\|_2^{2}\le C\,\frac{s_0\log(p\vee T_0)}{T_0},\\
\delta^{\intercal}\Sigma\delta&\le C\,\frac{s_0\log(p\vee T_0)}{T_0}.
\end{aligned}
\end{equation}

Let $\mathcal{E}_T$ be the event on which Lemma \ref{lem:nodewise} holds and
\begin{equation}\label{eq:on_gram}
\|\hat{\Sigma}-\Sigma\|_{\infty}\le C\sqrt{\frac{\log p}{T_0}} .
\end{equation}
The entries of $X_tX_t^{\intercal}$ are bounded by $M^{2}$, so \eqref{eq:on_gram} holds with probability at least $1-p^{-c}$ by Hoeffding's inequality and a union bound over the $p^{2}$ entries. Hence $\mathcal{E}_T$ is $\mathcal{G}$-measurable and $\mathbb{P}(\mathcal{E}_T)\to1$. On $\mathcal{E}_T$ the normalizer $v_j$ is bounded away from zero and infinity. 
Indeed, $\Xi_j^{\intercal}\Sigma\Xi_j=\Xi_{jj}\le\lambda_{\max}(\Xi)\le C_2$ by Assumption \ref{assum:online}(iv). The Cauchy--Schwarz inequality gives $1\le\Sigma_{jj}\Xi_{jj}$, while $\Sigma_{jj}\le M^2\le C_2$ by Assumption \ref{assum:online}(ii) and the choice of $C_2$. Thus $\Xi_{jj}\in[C_2^{-1},C_2]$. By Lemma \ref{lem:nodewise}(ii) and \eqref{eq:on_gram},

\begin{align*}
|v_j-\Xi_{jj}|
&\le\big|(\hat{\Xi}_j-\Xi_j)^{\intercal}\hat{\Sigma}(\hat{\Xi}_j+\Xi_j)\big|+\big|\Xi_j^{\intercal}(\hat{\Sigma}-\Sigma)\Xi_j\big|\\
&\le M^{2}\,\|\hat{\Xi}_j-\Xi_j\|_1\,\|\hat{\Xi}_j+\Xi_j\|_1+\|\hat{\Sigma}-\Sigma\|_{\infty}\|\Xi_j\|_1^{2}\\
&\le C\Big(3M^{2}C_3\,s_{\Xi}\sqrt{\frac{\log p}{T_0}}+C_3^{2}\sqrt{\frac{\log p}{T_0}}\Big)\to0 .
\end{align*}
Hence, on $\mathcal{E}_T$ and for $T_0$ large,
\begin{equation}\label{eq:on_v_bounds}
\tfrac12C_2^{-1}\le v_j\le C_2+1
\qquad\text{and}\qquad
\max_{t\le T_0}|\hat{\Xi}_j^{\intercal}X_t|\le M\|\hat{\Xi}_j\|_1\le2MC_3 .
\end{equation}

\medskip\noindent\textit{Step 1 (decomposition).}
For each $t$, Taylor's theorem with Lagrange remainder gives
\[
\exp(X_t^{\intercal}\hat{\theta}^{\text{\texttt{on}}})=\exp(X_t^{\intercal}\theta_0)\Big\{1+X_t^{\intercal}\delta+\tfrac12\varrho_t(X_t^{\intercal}\delta)^{2}\Big\},\qquad
\varrho_t:=\exp(\tau_tX_t^{\intercal}\delta),\ \tau_t\in(0,1).
\]
Substituting into the definition of $\tilde{\theta}_j^{\text{\texttt{on}}}$ gives the exact decomposition
\begin{equation}\label{eq:on_decomp}
\tilde{\theta}_j^{\text{\texttt{on}}}-\theta_{0j}=-A_1-A_2-A_3,
\end{equation}
where
\begin{align*}
A_1&:=\frac{1}{T_0}\sum_{t=1}^{T_0}(W_t-1)\,\hat{\Xi}_j^{\intercal}X_t,\\
A_2&:=\big(\hat{\Xi}_j^{\intercal}\widetilde{\Sigma}-e_j^{\intercal}\big)\delta,\\
A_3&:=\frac{1}{2T_0}\sum_{t=1}^{T_0}\big(\hat{\Xi}_j^{\intercal}X_t\big)\,W_t\,\varrho_t\,(X_t^{\intercal}\delta)^{2}.
\end{align*}
Define $T_n:=\sqrt{T_0}\,v_j^{-1/2}A_1$. Then
\begin{equation}\label{eq:on_main}
\frac{\sqrt{T_0}\big(\tilde{\theta}_j^{\text{\texttt{on}}}-\theta_{0j}\big)}{\sqrt{v_j}}=-\,T_n-\frac{\sqrt{T_0}\,(A_2+A_3)}{\sqrt{v_j}} .
\end{equation}

\medskip\noindent\textit{Step 2 (remainders).}
Each entry of $\widetilde{\Sigma}-\hat{\Sigma}$ is an average of the variables $(W_t-1)X_{t,r}X_{t,s}$. Given $\mathcal{G}$ these are independent, centered and sub-exponential. Bernstein's inequality \citep[Proposition 5.16]{vershynin2010introduction} and a union bound over the $p^{2}$ entries give, with probability at least $1-p^{-c}$,
\begin{equation}\label{eq:on_gram_w}
\|\widetilde{\Sigma}-\hat{\Sigma}\|_{\infty}\le C\sqrt{\frac{\log p}{T_0}} .
\end{equation}
Work on $\mathcal{E}_T$ and on the events of \eqref{eq:on_rates} and \eqref{eq:on_gram_w}. For $A_2$, H\"older's inequality, \eqref{eq:on_gram_w} and Lemma \ref{lem:nodewise} give
\[
\|\hat{\Xi}_j^{\intercal}\widetilde{\Sigma}-e_j^{\intercal}\|_{\infty}
\le\|\hat{\Xi}_j\|_1\|\widetilde{\Sigma}-\hat{\Sigma}\|_{\infty}+\|\hat{\Sigma}\hat{\Xi}_j-e_j\|_{\infty}
\le C\sqrt{\frac{\log p}{T_0}} .
\]
Hence, by \eqref{eq:on_rates},
\[
\sqrt{T_0}\,|A_2|\le\sqrt{T_0}\,\|\hat{\Xi}_j^{\intercal}\widetilde{\Sigma}-e_j^{\intercal}\|_{\infty}\|\delta\|_1
\le C\,\frac{s_0\log(p\vee T_0)}{\sqrt{T_0}}\to0 .
\]
For $A_3$, Assumption \ref{assum:online}(ii) and \eqref{eq:on_rates} give $\max_t|X_t^{\intercal}\delta|\le M\|\delta\|_1\to0$, so $\max_t\varrho_t\le2$ for $T_0$ large. Since $W_t\ge0$, \eqref{eq:on_v_bounds}, \eqref{eq:on_gram} and \eqref{eq:on_gram_w} give
\begin{align*}
|A_3|
&\le\max_{t}|\hat{\Xi}_j^{\intercal}X_t|\cdot\delta^{\intercal}\widetilde{\Sigma}\delta\\
&\le2MC_3\Big(\delta^{\intercal}\Sigma\delta+\big(\|\widetilde{\Sigma}-\hat{\Sigma}\|_{\infty}+\|\hat{\Sigma}-\Sigma\|_{\infty}\big)\|\delta\|_1^{2}\Big)\\
&\le2MC_3\Big(\delta^{\intercal}\Sigma\delta+2C\sqrt{\frac{\log p}{T_0}}\,\|\delta\|_1^{2}\Big)
\le C\,\frac{s_0\log(p\vee T_0)}{T_0},
\end{align*}
where the last step uses \eqref{eq:on_rates} and $s_0\sqrt{\log(p\vee T_0)/T_0}\to0$. Hence $\sqrt{T_0}\,|A_3|\le Cs_0\log(p\vee T_0)/\sqrt{T_0}\to0$. Since $v_j^{-1/2}\le\sqrt{2C_2}$ on $\mathcal{E}_T$, the last term in \eqref{eq:on_main} is $o_p(1)$.

\medskip\noindent\textit{Step 3 (leading term).}
Write $T_n=\sum_{t=1}^{T_0}\zeta_t$ with $\zeta_t:=(W_t-1)\hat{\Xi}_j^{\intercal}X_t/\sqrt{T_0v_j}$. By the first two facts, conditional on $\mathcal{G}$ the $\zeta_t$ are independent and centered with
\[
\sum_{t=1}^{T_0}\mathrm{Var}(\zeta_t\mid\mathcal{G})=\frac{1}{T_0v_j}\sum_{t=1}^{T_0}\big(\hat{\Xi}_j^{\intercal}X_t\big)^{2}=1 .
\]
On $\mathcal{E}_T$, by \eqref{eq:on_v_bounds},
\[
\sum_{t=1}^{T_0}\mathbb{E}\big[|\zeta_t|^{3}\mid\mathcal{G}\big]
\le\frac{\kappa_3\max_t|\hat{\Xi}_j^{\intercal}X_t|}{(T_0v_j)^{3/2}}\sum_{t=1}^{T_0}\big(\hat{\Xi}_j^{\intercal}X_t\big)^{2}
=\frac{\kappa_3\max_t|\hat{\Xi}_j^{\intercal}X_t|}{\sqrt{T_0v_j}}
\le\frac{2MC_3\kappa_3\sqrt{2C_2}}{\sqrt{T_0}}=:r_T\to0 .
\]
The Berry--Esseen theorem for independent non-identically distributed summands \citep[Theorem 5.4]{petrov1995limit} gives, on $\mathcal{E}_T$,
\[
\sup_{x\in\mathbb{R}}\big|\mathbb{P}(T_n\le x\mid\mathcal{G})-\Phi(x)\big|\le C_{\mathrm{BE}}\,r_T
\]
for a universal constant $C_{\mathrm{BE}}$. Taking expectations given $\mathcal{S}$, for every $x$,
\[
\big|\mathbb{P}(T_n\le x\mid\mathcal{S})-\Phi(x)\big|\le C_{\mathrm{BE}}\,r_T+\mathbb{P}(\mathcal{E}_T^{c}\mid\mathcal{S})\to0 .
\]

\medskip\noindent\textit{Step 4 (conclusion).}
By \eqref{eq:on_main}, Steps 2 and 3, Slutsky's theorem and the symmetry of the normal law, the conditional distribution function of the statistic given $\mathcal{S}$ converges to $\Phi(x)$ for every $x$. This holds almost surely along the subsequence chosen at the start of the proof. The bounded convergence theorem then gives the same limit for the unconditional distribution function along that subsequence. Since the initial subsequence was arbitrary,
\[
\frac{\sqrt{T_0}\big(\tilde{\theta}_j^{\text{\texttt{on}}}-\theta_{0j}\big)}{\sqrt{(\hat{\Xi}\hat{\Sigma}\hat{\Xi}^{\intercal})_{jj}}}\overset{d}{\rightarrow}\mathcal{N}(0,1).
\]
\end{proof}

\subsection{Auxiliary Lemmas and Their Proofs}\label{app:lemma}
Lemmas \ref{lem:score_tail} and \ref{lem:pareto_moments} collect properties of the conditional Pareto tail. 
Propositions \ref{prop:lasso_glm} and \ref{prop:RSC} are used in the proof of Theorem \ref{Thm:consistency}.
Lemma \ref{lem:projection} is used in the proof of Theorem \ref{Thm:normality2}, Lemmas \ref{lem:lasso_oracle} to \ref{lem:iv_projection} in the proof of Theorem \ref{Thm:2sls}, and Lemmas \ref{lem:nodewise} to \ref{lem:radar_rates} in the proof of Theorem \ref{Thm:online}.

\begin{lemma}[Sub-exponential tail of the score]\label{lem:score_tail}
Suppose that Assumptions \ref{assum:shape} and \ref{assum:support} hold, and define
\begin{equation*}
Z_{i,j}=\left[ \alpha \left( X_{i}\right) \log \left(Y_{i}/w_{n}\right) -1\right] X_{i,j}.
\end{equation*}
Then, for all $u>0$,
\begin{equation*}
\mathbb{P}\left( \left\vert Z_{i,j}\right\vert >u\mid Y_{i}>w_{n}\right) \leq c_{1}\exp \left( -c_{2}u\right)
\end{equation*}
with $c_1=e+e^{-1}$ and $c_2=1/M$. The constants do not depend on $j$, $p$ or $w_{n}$.
\end{lemma}

\begin{proof}[Proof of Lemma \ref{lem:score_tail}]
Write $W_i:=\alpha(X_i)\log(Y_i/w_n)$. By Assumption \ref{assum:shape}, conditional on $X_i$ and $Y_i>w_n$ the variable $W_i$ is standard exponential, so for $v>0$
\[
\mathbb{P}(W_i>v\mid X_i,Y_i>w_n)=e^{-v}\qquad\text{and}\qquad\mathbb{P}(W_i<0\mid X_i,Y_i>w_n)=0 .
\]
On $\{X_{i,j}=0\}$ we have $Z_{i,j}=0$, so it suffices to consider $\{X_{i,j}\neq0\}$. On this event, $|Z_{i,j}|>u$ if and only if $W_i>1+u/|X_{i,j}|$ or $W_i<1-u/|X_{i,j}|$. Hence
\[
\mathbb{P}\left(|Z_{i,j}|>u\mid Y_i>w_n\right)=P_1(u)+P_2(u),
\]
where
\[
P_1(u):=\mathbb{E}\Big[\mathbb{P}\Big(W_i>1+\tfrac{u}{|X_{i,j}|}\ \Big|\ X_i,Y_i>w_n\Big)\mathbf{1}\{X_{i,j}\neq0\}\ \Big|\ Y_i>w_n\Big]
\]
and $P_2(u)$ is defined in the same way with the event $\{W_i<1-u/|X_{i,j}|\}$.

For $P_1(u)$, the exponential tail and $|X_{i,j}|\le M$ from Assumption \ref{assum:support} give
\[
P_1(u)=\mathbb{E}\Big[e^{-(1+u/|X_{i,j}|)}\mathbf{1}\{X_{i,j}\neq0\}\ \Big|\ Y_i>w_n\Big]\le e^{-1}e^{-u/M}.
\]
For $P_2(u)$, since $W_i\ge0$, the event $\{W_i<1-u/|X_{i,j}|\}$ is empty unless $|X_{i,j}|>u$. Because $|X_{i,j}|\le M$, this gives $P_2(u)\le\mathbf{1}\{u<M\}\le e^{\,1-u/M}$.

Adding the two bounds,
\[
\mathbb{P}\left(|Z_{i,j}|>u\mid Y_i>w_n\right)\le\left(e^{-1}+e\right)e^{-u/M},
\]
which is the claim. Only the exponential tail of $W_i$ and the bound $|X_{i,j}|\le M$ were used, so the constants do not depend on $j$, $p$ or $w_n$.
\end{proof}

\begin{lemma}[Moment identities under the conditional Pareto tail]\label{lem:pareto_moments}
Suppose that Assumption \ref{assum:shape} holds and write $W_i:=\exp(X_i^{\intercal}\theta_0)\log(Y_i/w_n)$. Then, conditional on $X_i$ and $Y_i>w_n$, the variable $W_i$ is standard exponential. In particular, the following identities hold.
\begin{parts}
\item[(i)] $\mathbb{E}\left[ \log \left( Y_{i}/w_{n}\right) \mid X_{i},Y_{i}>w_{n}\right] =\exp \left( -X_{i}^{\intercal }\theta _{0}\right)$.
\item[(ii)] $\mathbb{E}\left[ (W_i-1)X_{i}\mid X_i,Y_{i}>w_{n}\right] = 0$, and hence $\mathbb{E}\left[ (W_i-1)X_{i}\mid Y_{i}>w_{n}\right] = 0$.
\item[(iii)] $\mathbb{E}\left[ (W_i-1)^{2}X_{i}X_{i}^{\intercal }\mid X_i,Y_{i}>w_{n}\right] = X_{i}X_{i}^{\intercal }$, and hence $\mathbb{E}\left[ (W_i-1)^{2}X_{i}X_{i}^{\intercal }\mid Y_{i}>w_{n}\right] = \Sigma_{w_n}$.
\item[(iv)] $\mathbb{E}\left[ W_iX_{i}X_{i}^{\intercal}\mid X_i,Y_{i}>w_{n}\right] = X_{i}X_{i}^{\intercal }$, and hence $\mathbb{E}\left[ W_iX_{i}X_{i}^{\intercal}\mid Y_{i}>w_{n}\right] = \Sigma _{w_{n}}$.
\end{parts}
\end{lemma}

\begin{proof}[Proof of Lemma \ref{lem:pareto_moments}]
Fix $x$ in the support of $X_i$ and write $\alpha=\alpha(x)=\exp(x^{\intercal}\theta_0)$. By Assumption \ref{assum:shape}, for $v>0$,
\[
\mathbb{P}\left(W_i>v\mid X_i=x,Y_i>w_n\right)
=\mathbb{P}\left(Y_i>w_n e^{v/\alpha}\mid X_i=x,Y_i>w_n\right)
=\left(e^{v/\alpha}\right)^{-\alpha}=e^{-v},
\]
so $W_i$ is standard exponential conditional on $X_i=x$ and $Y_i>w_n$. Hence $\mathbb{E}[W_i\mid X_i,Y_i>w_n]=1$ and $\mathrm{Var}(W_i\mid X_i,Y_i>w_n)=1$.

Part (i) follows from $\log(Y_i/w_n)=W_i/\alpha(X_i)$ and $\mathbb{E}[W_i\mid X_i,Y_i>w_n]=1$. For parts (ii) to (iv), $X_i$ is fixed under the conditioning, so
\[
\mathbb{E}\left[ (W_i-1)X_{i}\mid X_i,Y_{i}>w_{n}\right]=X_i\,\mathbb{E}[W_i-1\mid X_i,Y_i>w_n]=0,
\]
\[
\mathbb{E}\left[ (W_i-1)^{2}X_{i}X_{i}^{\intercal }\mid X_i,Y_{i}>w_{n}\right]=X_iX_i^{\intercal}\,\mathrm{Var}(W_i\mid X_i,Y_i>w_n)=X_iX_i^{\intercal},
\]
\[
\mathbb{E}\left[ W_iX_{i}X_{i}^{\intercal}\mid X_i,Y_{i}>w_{n}\right]=X_iX_i^{\intercal}\,\mathbb{E}[W_i\mid X_i,Y_i>w_n]=X_iX_i^{\intercal}.
\]
The unconditional statements follow by taking expectations given $Y_i>w_n$, since $\Sigma_{w_n}=\mathbb{E}[X_iX_i^{\intercal}\mid Y_i>w_n]$.
\end{proof}


The following two propositions are used only in the proof of Theorem \ref{Thm:consistency}. The first is the oracle inequality of \citet{HuangZhang2012} in the form given by \citet{cai2023statistical}. The second verifies its curvature requirement for the tail log-likelihood.

 
\begin{proposition}[Oracle inequality for $\ell_1$-penalized generalized linear models]\label{prop:lasso_glm}
Let $\ell_{n}$ be a convex normalized negative log-likelihood with true parameter $\theta_0$, let $\hat{\theta}=\arg\min_{\theta}\{\ell_{n}(\theta)+\lambda\|\theta\|_{1}\}$, and let $S=\{j\colon\theta_{0j}\neq0\}$. For $\varsigma>0$ define the cone $\mathcal{C}(\varsigma,S)=\{b\in\mathbb{R}^{p}\colon\|b_{S^{c}}\|_{1}\leq\varsigma\|b_{S}\|_{1}\neq0\}$ and, for semi-norms $\psi$ and $\psi_0$,
\[
F(\varsigma,S;\psi,\psi_{0})=\inf_{b\in\mathcal{C}(\varsigma,S),\ \psi_{0}(b)\leq1}\frac{b^{\intercal}\big(\dot{\ell}_{n}(\theta_{0}+b)-\dot{\ell}_{n}(\theta_{0})\big)\,e^{-\psi_{0}^{2}(b)-\psi_{0}(b)}}{\|b_{S}\|_{1}\,\psi(b)} .
\]
Let $z^{*}=\|\dot{\ell}_{n}(\theta_{0})\|_{\infty}$ and, for some $\eta\leq1/2$,
\[
\Omega=\left\{\frac{\lambda+z^{*}}{(\lambda-z^{*})_{+}}\leq\varsigma,\ \ \frac{\lambda+z^{*}}{F(\varsigma,S;\psi_0,\psi_{0})}\leq\eta e^{-\eta^{2}-\eta}\right\}.
\]
Then, on the event $\Omega$, $\psi(\hat{\theta}-\theta_{0})\leq(\lambda+z^{*})\,e^{\eta^{2}+\eta}/F(\varsigma,S;\psi,\psi_{0})$.
\end{proposition}

 
\begin{proposition}[Restricted strong convexity of the tail log-likelihood]\label{prop:RSC}
Fix $\varsigma>0$.
Suppose that Assumptions \ref{assum:shape} and \ref{assum:support} and the lower eigenvalue condition of Theorem \ref{Thm:consistency} hold, and that $s_0^{2}\log p=o(n_0)$. Let $\kappa:=\underline{\alpha}/(4e\,\overline{\alpha}C_2)$. Then, with probability approaching one,
\begin{equation}\label{eq:RSC}
b^{\intercal}\big(\dot{\ell}_{n_0}(\theta_0+b)-\dot{\ell}_{n_0}(\theta_0)\big)\ \ge\ \kappa\,\|b\|_2^{2}
\qquad\text{for all }b\in\mathcal{C}(\varsigma,S)\text{ with }M\|b\|_1\le1 ,
\end{equation}
and consequently $F(\varsigma,S;\|\cdot\|_2,M\|\cdot\|_1)\ge\kappa e^{-2}/\sqrt{s_0}$.
\end{proposition}

\begin{proof}[Proof of Proposition \textup{\ref{prop:RSC}}]
Write $W_i:=\exp(X_i^{\intercal}\theta_0)\log(Y_i/w_n)$, $\widehat{\Sigma}_{n_0}:=n_0^{-1}\sum_{i=1}^{n_0}X_iX_i^{\intercal}$ and $\mathcal{G}:=\sigma(X_1,\dots,X_{n_0})$. By Lemma \ref{lem:pareto_moments}, conditional on $\mathcal{G}$ the variables $W_i$ are i.i.d.\ standard exponential, and $\log(Y_i/w_n)=W_i/\alpha(X_i)$. 
The proof has four steps. Step 1 bounds the curvature from below by a weighted Gram form $T_{n_0}(b)$. Step 2 bounds the conditional mean of $T_{n_0}(b)$ given $\mathcal{G}$ from below using the population eigenvalue condition. Step 3 bounds the fluctuation of $T_{n_0}(b)$ around its conditional mean. Step 4 concludes. The events in Steps 2 and 3 do not depend on $b$, which gives the uniformity in \eqref{eq:RSC}. 

\medskip\noindent\textit{Step 1 (curvature as a weighted Gram form).}
Fix $b$ with $M\|b\|_1\le1$. Since
\[
\dot{\ell}_{n_0}(\theta)=\frac{1}{n_0}\sum_{i=1}^{n_0}\Big\{\exp(X_i^{\intercal}\theta)\frac{W_i}{\alpha(X_i)}-1\Big\}X_i
\]
and $e^{a+c}-e^{a}=c\int_0^1e^{a+tc}\,dt$,
\begin{align*}
b^{\intercal}\big(\dot{\ell}_{n_0}(\theta_0+b)-\dot{\ell}_{n_0}(\theta_0)\big)
&=\frac{1}{n_0}\sum_{i=1}^{n_0}\frac{W_i}{\alpha(X_i)}\Big(\int_0^1\exp\big(X_i^{\intercal}(\theta_0+tb)\big)\,dt\Big)(b^{\intercal}X_i)^{2}
\ \ge\ e^{-1}\underline{\alpha}\,T_{n_0}(b),
\end{align*}
where
\[
T_{n_0}(b):=\frac{1}{n_0}\sum_{i=1}^{n_0}\frac{W_i\,(b^{\intercal}X_i)^{2}}{\alpha(X_i)} .
\]
The inequality holds because $\exp(X_i^{\intercal}(\theta_0+tb))=\alpha(X_i)\exp(tX_i^{\intercal}b)\ge\underline{\alpha}\,e^{-M\|b\|_1}\ge e^{-1}\underline{\alpha}$ for $t\in[0,1]$ by Assumptions \ref{assum:shape}(ii) and \ref{assum:support}, and because $W_i\ge0$.

\medskip\noindent\textit{Step 2 (conditional mean).}
Since $\mathbb{E}[W_i\mid\mathcal{G}]=1$ and $\alpha(X_i)\le\overline{\alpha}$,
\[
\mathbb{E}\big[T_{n_0}(b)\mid\mathcal{G}\big]=\frac{1}{n_0}\sum_{i=1}^{n_0}\frac{(b^{\intercal}X_i)^{2}}{\alpha(X_i)}\ \ge\ \frac{1}{\overline{\alpha}}\,b^{\intercal}\widehat{\Sigma}_{n_0}b .
\]
The entries of $X_iX_i^{\intercal}$ are bounded by $M^{2}$ under Assumption \ref{assum:support}, so Hoeffding's inequality and a union bound over the $p^{2}$ entries give $\|\widehat{\Sigma}_{n_0}-\Sigma_{w_n}\|_{\infty}\le C\sqrt{(\log p)/n_0}$ with probability at least $1-p^{-c}$. On this event, for every $b\in\mathcal{C}(\varsigma,S)$, the cone inequality $\|b\|_1\le(1+\varsigma)\|b_S\|_1\le(1+\varsigma)\sqrt{s_0}\,\|b\|_2$ and the eigenvalue condition give
\[
b^{\intercal}\widehat{\Sigma}_{n_0}b\ \ge\ b^{\intercal}\Sigma_{w_n}b-\|\widehat{\Sigma}_{n_0}-\Sigma_{w_n}\|_{\infty}\|b\|_1^{2}
\ \ge\ \Big(C_2^{-1}-C(1+\varsigma)^{2}s_0\sqrt{\frac{\log p}{n_0}}\Big)\|b\|_2^{2}\ \ge\ \frac{1}{2C_2}\,\|b\|_2^{2}
\]
for $n_0$ large, since $s_0\sqrt{(\log p)/n_0}\to0$ under $s_0^{2}\log p=o(n_0)$. Hence $\mathbb{E}[T_{n_0}(b)\mid\mathcal{G}]\ge\|b\|_2^{2}/(2\overline{\alpha}C_2)$.

\medskip\noindent\textit{Step 3 (fluctuation).}
Let
\[
M_{n_0}:=\frac{1}{n_0}\sum_{i=1}^{n_0}\frac{W_i-1}{\alpha(X_i)}\,X_iX_i^{\intercal},
\]
so that $T_{n_0}(b)-\mathbb{E}[T_{n_0}(b)\mid\mathcal{G}]=b^{\intercal}M_{n_0}b$. Given $\mathcal{G}$, each entry of $M_{n_0}$ is an average of independent centered variables $(W_i-1)\,\alpha(X_i)^{-1}X_{i,r}X_{i,s}$. Each is sub-exponential with $\psi_1$ norm at most $2M^{2}/\underline{\alpha}$, since $\|W_i-1\|_{\psi_1}\le2$, $|X_{i,r}X_{i,s}|\le M^{2}$ and $\alpha(X_i)\ge\underline{\alpha}$. Bernstein's inequality \citep[Proposition 5.16]{vershynin2010introduction} and a union bound over the $p^{2}$ entries give $\|M_{n_0}\|_{\infty}\le C\sqrt{(\log p)/n_0}$ with conditional probability at least $1-p^{-c}$, and hence unconditionally. On this event, for every $b\in\mathcal{C}(\varsigma,S)$,
\[
\big|T_{n_0}(b)-\mathbb{E}[T_{n_0}(b)\mid\mathcal{G}]\big|\le\|M_{n_0}\|_{\infty}\|b\|_1^{2}
\le C(1+\varsigma)^{2}\,s_0\sqrt{\frac{\log p}{n_0}}\;\|b\|_2^{2}
\le\frac{1}{4\overline{\alpha}C_2}\,\|b\|_2^{2}
\]
for $n_0$ large, again because $s_0\sqrt{(\log p)/n_0}\to0$.

\medskip\noindent\textit{Step 4 (conclusion).}
On the intersection of the events of Steps 2 and 3, which has probability at least $1-2p^{-c}$, Steps 1 to 3 give, for every $b\in\mathcal{C}(\varsigma,S)$ with $M\|b\|_1\le1$,
\[
b^{\intercal}\big(\dot{\ell}_{n_0}(\theta_0+b)-\dot{\ell}_{n_0}(\theta_0)\big)\ge e^{-1}\underline{\alpha}\Big(\frac{1}{2\overline{\alpha}C_2}-\frac{1}{4\overline{\alpha}C_2}\Big)\|b\|_2^{2}=\kappa\,\|b\|_2^{2},
\]
which is \eqref{eq:RSC}. For the bound on $F$, on the cone $\|b_S\|_1\le\sqrt{s_0}\|b\|_2$, and $e^{-M^{2}\|b\|_1^{2}-M\|b\|_1}\ge e^{-2}$ for $M\|b\|_1\le1$. By the definition of $F$ in Proposition \ref{prop:lasso_glm} with $\psi=\|\cdot\|_2$ and $\psi_0=M\|\cdot\|_1$,
\[
F(\varsigma,S;\|\cdot\|_2,M\|\cdot\|_1)\ \ge\ \inf_{b\in\mathcal{C}(\varsigma,S),\,M\|b\|_1\le1}\frac{\kappa\|b\|_2^{2}\,e^{-2}}{\sqrt{s_0}\,\|b\|_2^{2}}=\frac{\kappa e^{-2}}{\sqrt{s_0}} .
\]
\end{proof}

\begin{lemma}[Projection direction]\label{lem:projection}
Suppose that Assumptions \ref{assum:shape}--\ref{assum:proj} hold. In addition, suppose that $\lambda_{\min}(\Sigma_{w_n})\ge C_2^{-1}$, $(\log n_0)(\log p)=o(n_0)$, and $c'$ in Assumption \ref{assum:rate_tuning}(ii) is sufficiently large. Fix $j\in\{1,\dots,p\}$ and write $W_i:=\exp(X_i^{\intercal}\theta_0)\log(Y_i/w_n)$,
\[
\widehat{\Sigma}_{k}:=\frac{1}{n_k}\sum_{i\in\mathcal{D}_k}X_iX_i^{\intercal},\qquad
\widetilde{\Sigma}_{k}:=\frac{1}{n_k}\sum_{i\in\mathcal{D}_k}W_iX_iX_i^{\intercal},\qquad
v_{j,k}:=\hat{u}_{j,k}^{\intercal}\widehat{\Sigma}_{k}\hat{u}_{j,k}.
\]
Then there exist constants $c,C>0$ such that, with probability at least $1-p^{-c}-n_0^{-c}$, the following hold simultaneously for all $k=1,\dots,K$.
\begin{parts}
\item[(i)] The feasible set of \eqref{alg:cross_fitting}--\eqref{constraint2:cross} is nonempty, so that $\hat{u}_{j,k}$ exists.
\item[(ii)] $\displaystyle \frac{(1-\gamma_{1n_0})^{2}}{M^{2}}\ \le\ v_{j,k}\ \le\ C_2+1$.
\item[(iii)] $\displaystyle \left\Vert \hat{u}_{j,k}^{\intercal}\widetilde{\Sigma}_{k}-e_j^{\intercal}\right\Vert_{\infty}\ \le\ C\sqrt{\frac{\log p}{n_0}}$.
\item[(iv)] $\displaystyle \max_{i\in\mathcal{D}_k}\left\vert X_i^{\intercal}\hat{u}_{j,k}\right\vert\le\gamma_{2n_0}$.
\end{parts}
\end{lemma}

\begin{proof}[Proof of Lemma \ref{lem:projection}]
Fix $k$ and let $\mathcal{G}_k:=\sigma\left(\{X_i\colon i\in\mathcal{D}_k\}\right)$. Since $K$ is fixed, a union bound over $k$ at the end only changes the constants, and $n_k=n_0/K\asymp n_0$. By Lemma \ref{lem:pareto_moments}, the variables $\{W_i\}_{i\in\mathcal{D}_k}$ are i.i.d.\ standard exponential and independent of $\mathcal{G}_k$. In particular $\mathbb{E}[W_i-1\mid\mathcal{G}_k]=0$ and $\|W_i-1\|_{\psi_1}\le2$. The proof has four parts. Step 1 shows that the population direction $u_j^{*}=\Sigma_{w_n}^{-1}e_j$ is feasible, which proves (i). Step 2 bounds $v_{j,k}$ from above by comparison with $u_j^{*}$ and from below by the first constraint, which proves (ii). Step 3 proves (iii) by a conditional Bernstein inequality in which the weights are fixed given $\mathcal{G}_k$. Step 4 concludes.

\medskip\noindent\textit{Step 1 (feasibility of $u_j^{*}$).}
Since $\Sigma_{w_n}u_j^{*}=e_j$, for each coordinate $r\in\{1,\dots,p\}$,
\[
\left(\widehat{\Sigma}_{k}u_j^{*}-e_j\right)_{r}
=\frac{1}{n_k}\sum_{i\in\mathcal{D}_k}\Big\{X_{i,r}\left(X_i^{\intercal}u_j^{*}\right)-\mathbb{E}\left[X_{i,r}\left(X_i^{\intercal}u_j^{*}\right)\mid Y_i>w_n\right]\Big\},
\]
an average of i.i.d.\ centered variables. By Assumptions \ref{assum:support} and \ref{assum:proj}, $X_{i,r}(X_i^{\intercal}u_j^{*})$ is the product of a variable bounded by $M$ and a sub-Gaussian variable, and is therefore sub-Gaussian with $\psi_2$ norm at most a universal multiple of $MK_0$. Hoeffding's inequality for sub-Gaussian summands \citep[Proposition 5.10]{vershynin2010introduction} and a union bound over $r\le p$ give a constant $C_{*}$, depending only on $(M,K_0)$, such that
\[
\mathbb{P}\left(\left\Vert\widehat{\Sigma}_{k}u_j^{*}-e_j\right\Vert_{\infty}>C_{*}\sqrt{\frac{\log p}{n_k}}\right)\le p^{-c}.
\]
Since $\gamma_{1n_0}=c'\sqrt{(\log p)/n_0}$ and $n_k=n_0/K$, constraint \eqref{constraint1:cross} holds for $u_j^{*}$ on the complement of this event provided $c'\ge C_{*}\sqrt{K}$. For constraint \eqref{constraint2:cross}, Assumption \ref{assum:proj} and a union bound over $i\in\mathcal{D}_k$ give
\[
\mathbb{P}\left(\max_{i\in\mathcal{D}_k}\left\vert X_i^{\intercal}u_j^{*}\right\vert>K_0\sqrt{2\log(2n_0)}\right)\le 2n_k\exp\{-2\log(2n_0)\}=\frac{2n_k}{4n_0^{2}}\le\frac{1}{2n_0},
\]
and $K_0\sqrt{2\log(2n_0)}\le2K_0\sqrt{\log n_0}\le c''\sqrt{\log n_0}=\gamma_{2n_0}$ for $n_0\ge2$ because $c''\ge2K_0$. Thus $u_j^{*}$ is feasible with probability at least $1-p^{-c}-n_0^{-1}$, which proves (i).

\medskip\noindent\textit{Step 2 (two-sided bound on $v_{j,k}$).}
For the upper bound, on the event of Step 1 the vector $u_j^{*}$ is feasible and $\hat{u}_{j,k}$ minimizes the objective, so
\[
v_{j,k}=\hat{u}_{j,k}^{\intercal}\widehat{\Sigma}_{k}\hat{u}_{j,k}\ \le\ u_j^{*\intercal}\widehat{\Sigma}_{k}u_j^{*}
=\frac{1}{n_k}\sum_{i\in\mathcal{D}_k}\left(X_i^{\intercal}u_j^{*}\right)^{2}.
\]
The summands are i.i.d.\ and sub-exponential as squares of sub-Gaussian variables, with $\psi_1$ norm at most a universal multiple of $K_0^{2}$, and their mean is $u_j^{*\intercal}\Sigma_{w_n}u_j^{*}=e_j^{\intercal}\Sigma_{w_n}^{-1}e_j\le\lambda_{\min}(\Sigma_{w_n})^{-1}\le C_2$. Bernstein's inequality \citep[Proposition 5.16]{vershynin2010introduction} gives
\[
\mathbb{P}\left(\frac{1}{n_k}\sum_{i\in\mathcal{D}_k}\left(X_i^{\intercal}u_j^{*}\right)^{2}>C_2+1\right)\le 2\exp(-cn_k)\le n_0^{-c},
\]
so $v_{j,k}\le C_2+1$ with probability at least $1-n_0^{-c}$.

For the lower bound, constraint \eqref{constraint1:cross} gives $e_j^{\intercal}\widehat{\Sigma}_{k}\hat{u}_{j,k}\ge 1-\gamma_{1n_0}$. Since $\widehat{\Sigma}_k$ is positive semi-definite, the Cauchy--Schwarz inequality in the semi-inner product $\langle a,b\rangle:=a^{\intercal}\widehat{\Sigma}_{k}b$ and Assumption \ref{assum:support} yield
\[
1-\gamma_{1n_0}\ \le\ e_j^{\intercal}\widehat{\Sigma}_{k}\hat{u}_{j,k}
\ \le\ \left(e_j^{\intercal}\widehat{\Sigma}_{k}e_j\right)^{1/2}\left(\hat{u}_{j,k}^{\intercal}\widehat{\Sigma}_{k}\hat{u}_{j,k}\right)^{1/2}
=\left(\frac{1}{n_k}\sum_{i\in\mathcal{D}_k}X_{i,j}^{2}\right)^{1/2}v_{j,k}^{1/2}\ \le\ M\,v_{j,k}^{1/2},
\]
so $v_{j,k}\ge(1-\gamma_{1n_0})^{2}/M^{2}$ whenever $\hat{u}_{j,k}$ exists. This is the bound of \citet[][Section 3.1]{javanmard2014confidence}.

\medskip\noindent\textit{Step 3 (the constraint for the weighted matrix $\widetilde{\Sigma}_k$).}
Decompose
\begin{equation}\label{eq:lem2_decomp}
\hat{u}_{j,k}^{\intercal}\widetilde{\Sigma}_{k}-e_j^{\intercal}
=\hat{u}_{j,k}^{\intercal}\left(\widetilde{\Sigma}_{k}-\widehat{\Sigma}_{k}\right)
+\left(\hat{u}_{j,k}^{\intercal}\widehat{\Sigma}_{k}-e_j^{\intercal}\right)
=:R_{k}^{\intercal}+\left(\hat{u}_{j,k}^{\intercal}\widehat{\Sigma}_{k}-e_j^{\intercal}\right).
\end{equation}
The second term satisfies $\|\hat{u}_{j,k}^{\intercal}\widehat{\Sigma}_{k}-e_j^{\intercal}\|_{\infty}\le\gamma_{1n_0}$ by \eqref{constraint1:cross}. For the first term, the $r$-th coordinate is
\[
R_{k,r}=\frac{1}{n_k}\sum_{i\in\mathcal{D}_k}(W_i-1)\,a_{i,r},\qquad a_{i,r}:=\left(X_i^{\intercal}\hat{u}_{j,k}\right)X_{i,r}.
\]
Conditional on $\mathcal{G}_k$, the weights $a_{i,r}$ are fixed because $\hat{u}_{j,k}$ is $\mathcal{G}_k$-measurable, and $\{W_i-1\}$ are i.i.d.\ centered sub-exponential variables independent of $\mathcal{G}_k$. By \eqref{constraint2:cross}, Assumption \ref{assum:support} and Step 2, the weights satisfy
\[
\max_{i}|a_{i,r}|\le M\gamma_{2n_0},\qquad
\frac{1}{n_k}\sum_{i\in\mathcal{D}_k}a_{i,r}^{2}\le M^{2}\,v_{j,k}\le M^{2}(C_2+1).
\]
Bernstein's inequality for weighted sums of independent sub-exponential variables \citep[Proposition 5.16]{vershynin2010introduction} gives, for every $t>0$,
\[
\mathbb{P}\left(|R_{k,r}|>t\;\middle|\;\mathcal{G}_k\right)
\le 2\exp\left[-c\,\min\left\{\frac{n_k t^{2}}{M^{2}(C_2+1)},\ \frac{n_k t}{M\gamma_{2n_0}}\right\}\right].
\]
Take $t=C\sqrt{(\log p)/n_k}$ with $C$ large. The first term in the minimum equals $C^{2}\log p/\{M^{2}(C_2+1)\}$. The second equals $C\sqrt{n_k\log p}/(M\gamma_{2n_0})$, which exceeds $\log p$ for $n_0$ large because $\gamma_{2n_0}\sqrt{\log p}\asymp\sqrt{(\log n_0)(\log p)}=o(\sqrt{n_0})$ by hypothesis. A union bound over $r\le p$ then gives $\Vert R_k\Vert_{\infty}\le C\sqrt{(\log p)/n_k}$ with conditional probability at least $1-p^{-c}$, and therefore also unconditionally. Combining with \eqref{eq:lem2_decomp} and $\gamma_{1n_0}\asymp\sqrt{(\log p)/n_0}$ proves (iii).

\medskip\noindent\textit{Step 4 (conclusion).}
Part (iv) holds by construction. Intersecting the events of Steps 1 to 3 and taking a union bound over $k=1,\dots,K$ completes the proof.
\end{proof}


\begin{lemma}[Lasso oracle inequality]\label{lem:lasso_oracle}
Let $\{(y_i,x_i)\}_{i=1}^{n}$ satisfy $y_i=x_i^{\intercal}\beta_0+r_i$ with $\|\beta_0\|_0\le s$ and support $S$. Let $\hat{\beta}=\arg\min_{\beta}\{n^{-1}\sum_{i=1}^{n}(y_i-x_i^{\intercal}\beta)^{2}+\lambda\|\beta\|_1\}$ and $\hat{\Sigma}_x=n^{-1}\sum_{i=1}^{n} x_ix_i^{\intercal}$. Suppose that $\lambda\ge 4\|n^{-1}\sum_{i=1}^{n}x_ir_i\|_{\infty}$ and that $b^{\intercal}\hat{\Sigma}_xb\ge\kappa\|b\|_2^{2}$ for all $b$ with $\|b_{S^c}\|_1\le 3\|b_S\|_1$. Then $\hat{\beta}-\beta_0$ satisfies $\|(\hat{\beta}-\beta_0)_{S^{c}}\|_1\le3\|(\hat{\beta}-\beta_0)_{S}\|_1$ and
\[
\|\hat{\beta}-\beta_0\|_1\le\frac{6\,s\lambda}{\kappa},\qquad
\|\hat{\beta}-\beta_0\|_2\le\frac{3\sqrt{s}\,\lambda}{2\kappa},\qquad
(\hat{\beta}-\beta_0)^{\intercal}\hat{\Sigma}_x(\hat{\beta}-\beta_0)\le\frac{9\,s\lambda^{2}}{4\kappa}.
\]
\end{lemma}

\begin{proof}[Proof of Lemma \ref{lem:lasso_oracle}]
Write $\Delta:=\hat{\beta}-\beta_0$ and $\xi:=n^{-1}\sum_{i=1}^{n}x_ir_i$, so that $\|\xi\|_{\infty}\le\lambda/4$ by hypothesis. Since $\hat{\beta}$ minimizes the penalized objective and $y_i-x_i^{\intercal}\hat{\beta}=r_i-x_i^{\intercal}\Delta$,
\[
\frac{1}{n}\sum_{i=1}^{n}r_i^{2}-2\xi^{\intercal}\Delta+\Delta^{\intercal}\hat{\Sigma}_x\Delta+\lambda\|\hat{\beta}\|_1\le\frac{1}{n}\sum_{i=1}^{n}r_i^{2}+\lambda\|\beta_0\|_1 .
\]
Since $\beta_0$ is supported on $S$, $\|\beta_0\|_1-\|\hat{\beta}\|_1\le\|\Delta_S\|_1-\|\Delta_{S^{c}}\|_1$. Together with $|\xi^{\intercal}\Delta|\le\|\xi\|_{\infty}\|\Delta\|_1\le(\lambda/4)(\|\Delta_S\|_1+\|\Delta_{S^{c}}\|_1)$, this gives
\begin{equation}\label{eq:basic_ineq}
\Delta^{\intercal}\hat{\Sigma}_x\Delta\le\frac{3\lambda}{2}\|\Delta_S\|_1-\frac{\lambda}{2}\|\Delta_{S^{c}}\|_1 .
\end{equation}
The left side is nonnegative, so $\|\Delta_{S^{c}}\|_1\le3\|\Delta_S\|_1$, which is the cone condition. Hence the restricted eigenvalue hypothesis applies to $\Delta$. Dropping the negative term in \eqref{eq:basic_ineq} and using $\|\Delta_S\|_1\le\sqrt{s}\|\Delta\|_2\le\sqrt{s}\,(\Delta^{\intercal}\hat{\Sigma}_x\Delta/\kappa)^{1/2}$,
\[
\Delta^{\intercal}\hat{\Sigma}_x\Delta\le\frac{3\lambda\sqrt{s}}{2}\Big(\frac{\Delta^{\intercal}\hat{\Sigma}_x\Delta}{\kappa}\Big)^{1/2},
\qquad\text{so}\qquad
\Delta^{\intercal}\hat{\Sigma}_x\Delta\le\frac{9s\lambda^{2}}{4\kappa}.
\]
The restricted eigenvalue hypothesis then gives $\|\Delta\|_2\le(\Delta^{\intercal}\hat{\Sigma}_x\Delta/\kappa)^{1/2}\le3\sqrt{s}\lambda/(2\kappa)$, and the cone condition gives $\|\Delta\|_1\le4\|\Delta_S\|_1\le4\sqrt{s}\|\Delta\|_2\le6s\lambda/\kappa$. This argument is that of \citet[Theorem 7.2]{bickel2009simultaneous}, see also \citet[Theorem 6.1]{buhlmann2011statistics}.
\end{proof}


\begin{lemma}[Rates for the two Lasso stages]\label{lem:2sls_rates}
Suppose that Assumption \ref{assum:iv} holds and that $s_0s_A\log(p\vee p_q)\ll\sqrt{n_0}$. Write $\mathcal{D}_{-k}:=\mathcal{D}_k^{c}\setminus\mathcal{D}_0$ and, for a set $\mathcal{I}$ of indices and $b\in\mathbb{R}^{p_q}$,
\[
\|b\|_{n,\mathcal{I}}^{2}:=\frac{1}{|\mathcal{I}|}\sum_{i\in\mathcal{I}}(Q_i^{\intercal}b)^{2},\qquad
\hat{\Sigma}_{q,\mathcal{I}}:=\frac{1}{|\mathcal{I}|}\sum_{i\in\mathcal{I}}Q_iQ_i^{\intercal},\qquad
\hat{\Sigma}_{d,\mathcal{I}}:=\frac{1}{|\mathcal{I}|}\sum_{i\in\mathcal{I}}\hat{d}_i\hat{d}_i^{\intercal}.
\]
Then there exist constants $c,C>0$ such that, with probability at least $1-C(p\vee p_q)^{-c}$, the following bounds hold simultaneously for all $\ell\in\{1,\dots,p\}$, all $k\in\{1,\dots,K\}$ and all $\mathcal{I}\in\{\mathcal{D}_k,\mathcal{D}_{-k}\}$.
\begin{parts}
\item[(i)] $\|\hat{a}_{\ell}-a_{\ell}\|_1\le C\,s_AL_{n_0}$ and $\|\hat{a}_{\ell}-a_{\ell}\|_2\le C\,\rho_{n_0}$.
\item[(ii)] $\|\hat{a}_{\ell}-a_{\ell}\|_{n,\mathcal{I}}\le C\,\rho_{n_0}$, $\|\hat{a}_{\ell}\|_{n,\mathcal{I}}\le2MC_1$ and $\max_{i\in\mathcal{I}}|\hat{d}_{i\ell}|\le2MC_1$.
\item[(iii)] $\|\hat{\Sigma}_{d,\mathcal{I}}-\Sigma_d\|_{\infty}\le C\,\rho_{n_0}$.
\item[(iv)] $\|\hat{\theta}_k-\theta_0\|_1\le C\,s_0\rho_{n_0}$ and $\|\hat{\theta}_k-\theta_0\|_2\le C\sqrt{s_0}\,\rho_{n_0}$.
\end{parts}
\end{lemma}

\begin{proof}[Proof of Lemma \ref{lem:2sls_rates}]
Write $b_{\ell}:=\hat{a}_{\ell}-a_{\ell}$. The proof has four steps. Step 1 applies Lemma \ref{lem:lasso_oracle} to the first-stage regressions on $\mathcal{D}_0$.
Step 2 transfers the first-stage training prediction bound to the empirical prediction norm on the other folds. Step 3 bounds the entrywise error of the Gram matrix of the generated regressors. Step 4 applies Lemma \ref{lem:lasso_oracle} to the second-stage regression on $\mathcal{D}_{-k}$, where the first-stage error enters the score.

The condition $s_0s_A\log(p\vee p_q)\ll\sqrt{n_0}$ implies
\begin{equation}\label{eq:rates_small}
s_AL_{n_0}\to0,\qquad s_0\rho_{n_0}\to0,\qquad \rho_{n_0}\to0,
\end{equation}
because $s_A^{2}L_{n_0}^{2}=s_A^{2}\log(p\vee p_q)/n_0\le\{s_As_0\log(p\vee p_q)\}^{2}/n_0$ and $s_0^{2}\rho_{n_0}^{2}=s_0^{2}s_A\log(p\vee p_q)/n_0\le\{s_0s_A\log(p\vee p_q)\}^{2}/n_0$. All folds have size $n_k=n_0/(K+1)\asymp n_0$ and $|\mathcal{D}_{-k}|=(K-1)n_k\asymp n_0$, so these sample sizes are not distinguished in the rates.

By Assumption \ref{assum:iv}(i) the entries of $Q_iQ_i^{\intercal}$ are bounded by $M^{2}$. Hoeffding's inequality and a union bound over the $p_q^{2}$ entries and the $2K+1$ index sets give, with probability at least $1-(p\vee p_q)^{-c}$,
\begin{equation}\label{eq:gram_q}
\max_{\mathcal{I}\in\{\mathcal{D}_0,\dots,\mathcal{D}_K,\mathcal{D}_{-1},\dots,\mathcal{D}_{-K}\}}\ \|\hat{\Sigma}_{q,\mathcal{I}}-\Sigma_q\|_{\infty}\le C\,L_{n_0}.
\end{equation}
Similarly, $|d_{i\ell}|=|Q_i^{\intercal}a_{\ell}|\le\|Q_i\|_{\infty}\|a_{\ell}\|_1\le MC_1$ by Assumption \ref{assum:iv}(i) and (iii), so the entries of $d_id_i^{\intercal}$ are bounded by $M^{2}C_1^{2}$ and, with the same probability,
\begin{equation}\label{eq:gram_d}
\max_{\mathcal{I}}\ \Big\|\,\frac{1}{|\mathcal{I}|}\sum_{i\in\mathcal{I}}d_id_i^{\intercal}-\Sigma_d\Big\|_{\infty}\le C\,L_{n_0}.
\end{equation}
We work on the intersection of these two events throughout.

\medskip\noindent\textit{Step 1 (first stage, part (i)).}
Fix $\ell$. By Assumption \ref{assum:iv}(iii), $X_{i\ell}=Q_i^{\intercal}a_{\ell}+V_{i\ell}$ with $\mathbb{E}[V_{i\ell}\mid Q_i,Y_i>w_n]=0$ and $|V_{i\ell}|\le|X_{i\ell}|+|Q_i^{\intercal}a_{\ell}|\le M(1+C_1)$. Hence the entries of $Q_iV_{i\ell}$ are centered and bounded by $M^{2}(1+C_1)$. Hoeffding's inequality and a union bound over the $p\,p_q\le(p\vee p_q)^{2}$ pairs $(\ell,r)$ give, with probability at least $1-(p\vee p_q)^{-c}$,
\[
\max_{\ell\le p}\ \Big\|\frac{1}{|\mathcal{D}_0|}\sum_{i\in\mathcal{D}_0}Q_iV_{i\ell}\Big\|_{\infty}\le C_{0}\,L_{n_0}.
\]
Since $c_3$ is sufficiently large, $\lambda_{3,n_0}\ge4C_0L_{n_0}$. For the restricted eigenvalue condition, let $b$ satisfy $\|b_{S_{\ell}^{c}}\|_1\le3\|b_{S_{\ell}}\|_1$, where $S_{\ell}$ is the support of $a_{\ell}$ and $|S_{\ell}|\le s_A$. Then $\|b\|_1\le4\sqrt{s_A}\|b\|_2$, and by \eqref{eq:gram_q} and Assumption \ref{assum:iv}(iv),
\[
b^{\intercal}\hat{\Sigma}_{q,\mathcal{D}_0}b\ \ge\ b^{\intercal}\Sigma_qb-\|\hat{\Sigma}_{q,\mathcal{D}_0}-\Sigma_q\|_{\infty}\|b\|_1^{2}
\ \ge\ \big(C_2^{-1}-16C\,s_AL_{n_0}\big)\|b\|_2^{2}\ \ge\ \tfrac12C_2^{-1}\|b\|_2^{2}
\]
for $n_0$ large, by \eqref{eq:rates_small}. Lemma \ref{lem:lasso_oracle} with $\kappa=\tfrac12C_2^{-1}$, $s=s_A$ and $\lambda=\lambda_{3,n_0}$ yields $\|b_{\ell}\|_1\le C\,s_AL_{n_0}$, $\|b_{\ell}\|_2\le C\sqrt{s_A}L_{n_0}=C\rho_{n_0}$, and the training prediction bound
\[
b_{\ell}^{\intercal}\hat{\Sigma}_{q,\mathcal{D}_0}b_{\ell}
\le C s_A\lambda_{3,n_0}^{2}
\le C\rho_{n_0}^{2}.
\]
The restricted eigenvalue bound does not depend on $\ell$ and the score bound is uniform over $\ell$, so these bounds hold for all $\ell\le p$ simultaneously. In particular, (i) follows.

\medskip\noindent\textit{Step 2 (out-of-sample prediction norm, part (ii)).}
Fix $\mathcal{I}\in\{\mathcal{D}_k,\mathcal{D}_{-k}\}$. By \eqref{eq:gram_q},
\[
\|\hat{\Sigma}_{q,\mathcal{I}}-\hat{\Sigma}_{q,\mathcal{D}_0}\|_{\infty}
\le2C L_{n_0}.
\]
Using the training prediction bound from Step 1 and \eqref{eq:rates_small}, we obtain
\begin{align*}
\|b_{\ell}\|_{n,\mathcal{I}}^{2}
&=b_{\ell}^{\intercal}\hat{\Sigma}_{q,\mathcal{I}}b_{\ell}\\
&\le b_{\ell}^{\intercal}\hat{\Sigma}_{q,\mathcal{D}_0}b_{\ell}
+\|\hat{\Sigma}_{q,\mathcal{I}}-\hat{\Sigma}_{q,\mathcal{D}_0}\|_{\infty}
\|b_{\ell}\|_1^{2}\\
&\le C\rho_{n_0}^{2}+C L_{n_0}s_A^{2}L_{n_0}^{2}
\le C\rho_{n_0}^{2}(1+s_A L_{n_0})
\le C\rho_{n_0}^{2}.
\end{align*}

This bound is a deterministic consequence of \eqref{eq:gram_q} and Step 1, so it holds for all $\ell$ and all $\mathcal{I}$ simultaneously. Moreover, $\|a_{\ell}\|_{n,\mathcal{I}}\le\max_i|Q_i^{\intercal}a_{\ell}|\le MC_1$, so $\|\hat{a}_{\ell}\|_{n,\mathcal{I}}\le\|a_{\ell}\|_{n,\mathcal{I}}+\|b_{\ell}\|_{n,\mathcal{I}}\le MC_1+C\rho_{n_0}\le2MC_1$ for $n_0$ large. Finally, $\max_{i}|\hat{d}_{i\ell}|\le\max_i|d_{i\ell}|+M\|b_{\ell}\|_1\le MC_1+CMs_AL_{n_0}\le2MC_1$ for $n_0$ large.

\medskip\noindent\textit{Step 3 (Gram matrix of the generated regressors, part (iii)).}
Fix $\mathcal{I}$ and an entry $(\ell,m)$. Since $\hat{d}_{i\ell}=d_{i\ell}+Q_i^{\intercal}b_{\ell}$,
\[
\hat{d}_{i\ell}\hat{d}_{im}-d_{i\ell}d_{im}=d_{i\ell}\,Q_i^{\intercal}b_m+d_{im}\,Q_i^{\intercal}b_{\ell}+(Q_i^{\intercal}b_{\ell})(Q_i^{\intercal}b_m).
\]
Averaging over $\mathcal{I}$ and applying the Cauchy--Schwarz inequality to each term,
\begin{align*}
\Big|\frac{1}{|\mathcal{I}|}\sum_{i\in\mathcal{I}}\big(\hat{d}_{i\ell}\hat{d}_{im}-d_{i\ell}d_{im}\big)\Big|
&\le\|a_{\ell}\|_{n,\mathcal{I}}\|b_m\|_{n,\mathcal{I}}+\|a_m\|_{n,\mathcal{I}}\|b_{\ell}\|_{n,\mathcal{I}}+\|b_{\ell}\|_{n,\mathcal{I}}\|b_m\|_{n,\mathcal{I}}\\
&\le2MC_1\cdot C\rho_{n_0}+C\rho_{n_0}^{2}\le C\rho_{n_0},
\end{align*}
by Step 2. Combining with \eqref{eq:gram_d} and $L_{n_0}\le\rho_{n_0}$ gives $\|\hat{\Sigma}_{d,\mathcal{I}}-\Sigma_d\|_{\infty}\le C\rho_{n_0}$, uniformly over $(\ell,m)$ and $\mathcal{I}$.

\medskip\noindent\textit{Step 4 (second stage, part (iv)).}
Fix $k$. The second-stage Lasso \eqref{eq:2sls} regresses $Z_i$ on $\hat{d}_i$ over $\mathcal{D}_{-k}$. By \eqref{eq:z_frailty} and Assumption \ref{assum:iv}(iii),
\[
Z_i=\hat{d}_i^{\intercal}\theta_0+r_i,\qquad r_i:=U_i+V_i^{\intercal}\theta_0-\sum_{m=1}^{p}\theta_{0m}\,Q_i^{\intercal}b_m .
\]
We bound the score $\|\,|\mathcal{D}_{-k}|^{-1}\sum_{i\in\mathcal{D}_{-k}}\hat{d}_ir_i\|_{\infty}$ term by term. 
Let $\mathcal{G}$ be the $\sigma$-field defined in \eqref{eq:Gfield}. The weights $\hat{d}_{i\ell}$ are $\mathcal{G}$-measurable and bounded by $2MC_1$ by Step 2. 

\emph{(a) The term in $U_i$.} By Assumption \ref{assum:iv}(ii), given $\mathcal{G}$ the variables $U_i$, $i\in\mathcal{D}_{-k}$, are independent, centered and sub-exponential with a uniformly bounded Orlicz norm. Bernstein's inequality for weighted sums with bounded weights \citep[Proposition 5.16]{vershynin2010introduction} and a union bound over $\ell\le p$ give $\|\,|\mathcal{D}_{-k}|^{-1}\sum_i\hat{d}_{i}U_i\|_{\infty}\le C\,L_{n_0}$ with conditional probability at least $1-(p\vee p_q)^{-c}$, hence also unconditionally.

\emph{(b) The term in $V_i^{\intercal}\theta_0$.} Given $\mathcal{G}$, the variables $V_i^{\intercal}\theta_0$ are independent and centered, with $|V_i^{\intercal}\theta_0|\le\|V_i\|_{\infty}\|\theta_0\|_1\le M(1+C_1)C_1$ by Assumption \ref{assum:iv}(iii) and (v). Hoeffding's inequality and a union bound over $\ell\le p$ give $\|\,|\mathcal{D}_{-k}|^{-1}\sum_i\hat{d}_{i}V_i^{\intercal}\theta_0\|_{\infty}\le C\,L_{n_0}$ with probability at least $1-(p\vee p_q)^{-c}$.

\emph{(c) The first-stage term.} By the Cauchy--Schwarz inequality and Step 2 with $\mathcal{I}=\mathcal{D}_{-k}$, for each $\ell\le p$,
\[
\Big|\frac{1}{|\mathcal{D}_{-k}|}\sum_{i\in\mathcal{D}_{-k}}\hat{d}_{i\ell}\sum_{m}\theta_{0m}Q_i^{\intercal}b_m\Big|
\le\sum_{m}|\theta_{0m}|\,\|\hat{a}_{\ell}\|_{n,\mathcal{D}_{-k}}\|b_m\|_{n,\mathcal{D}_{-k}}
\le\|\theta_0\|_1\cdot2MC_1\cdot C\rho_{n_0}\le C\rho_{n_0}.
\]

Combining (a) to (c) and using $L_{n_0}\le\rho_{n_0}$, the score is at most $C^{\ast}\rho_{n_0}$ for a constant $C^{\ast}$, and $\lambda_{4,n_0}\ge4C^{\ast}\rho_{n_0}$ since $c_4$ is sufficiently large. For the restricted eigenvalue condition, let $b$ satisfy $\|b_{S_0^{c}}\|_1\le3\|b_{S_0}\|_1$, where $S_0$ is the support of $\theta_0$ and $|S_0|\le s_0$, so that $\|b\|_1\le4\sqrt{s_0}\|b\|_2$. By Step 3 with $\mathcal{I}=\mathcal{D}_{-k}$ and Assumption \ref{assum:iv}(iv),
\[
b^{\intercal}\hat{\Sigma}_{d,\mathcal{D}_{-k}}b\ \ge\ b^{\intercal}\Sigma_db-\|\hat{\Sigma}_{d,\mathcal{D}_{-k}}-\Sigma_d\|_{\infty}\|b\|_1^{2}
\ \ge\ \big(C_2^{-1}-16C\,s_0\rho_{n_0}\big)\|b\|_2^{2}\ \ge\ \tfrac12C_2^{-1}\|b\|_2^{2}
\]
for $n_0$ large, by \eqref{eq:rates_small}. Lemma \ref{lem:lasso_oracle} with $\kappa=\tfrac12C_2^{-1}$, $s=s_0$ and $\lambda=\lambda_{4,n_0}$ gives $\|\hat{\theta}_k-\theta_0\|_1\le C\,s_0\rho_{n_0}$ and $\|\hat{\theta}_k-\theta_0\|_2\le C\sqrt{s_0}\,\rho_{n_0}$. A union bound over $k\le K$ completes the proof.
\end{proof}


\begin{lemma}[Projection direction for the instrumental-variable program]\label{lem:iv_projection}
Suppose that the conditions of Lemma \ref{lem:2sls_rates} hold and that $c'\ge C\,C_3$, where $C$ is the constant of Lemma \ref{lem:2sls_rates}(iii). Let $u_j^{*}:=\Theta e_j$ with $\Theta=\Sigma_d^{-1}$, and let $v_{j,k}:=\hat{u}_{j,k}^{\intercal}\hat{\Sigma}_{d,k}\hat{u}_{j,k}$. Then, on the event on which parts (ii) and (iii) of Lemma \ref{lem:2sls_rates} hold and for $n_0$ large, the following hold for all $j\le p$ and $k\le K$.
\begin{parts}
\item[(i)] $u_j^{*}$ is feasible for \eqref{alg:cross_fitting_iv}--\eqref{constraint2:iv}, so that $\hat{u}_{j,k}$ exists.
\item[(ii)] $\displaystyle \frac{(1-\gamma_{1n_0})^{2}}{4M^{2}C_1^{2}}\ \le\ v_{j,k}\ \le\ C_2+1$.
\item[(iii)] $\max_{i\in\mathcal{D}_k}|\hat{d}_i^{\intercal}\hat{u}_{j,k}|\le\gamma_{2n_0}$.
\end{parts}
\end{lemma}

\begin{proof}[Proof of Lemma \ref{lem:iv_projection}]
Throughout we are on the event on which parts (ii) and (iii) of Lemma \ref{lem:2sls_rates} hold, and $\hat{\Sigma}_{d,k}$ equals $\hat{\Sigma}_{d,\mathcal{D}_k}$ in the notation of that lemma. By Assumption \ref{assum:iv}(iv), $\|u_j^{*}\|_1\le m_{\Theta}\le C_3$ and $u_j^{*\intercal}\Sigma_du_j^{*}=\Theta_{jj}\le\lambda_{\max}(\Theta)\le C_2$. We now proceed with four steps. Step 1 shows that $u_j^{*}$ is feasible, and Step 2 bounds $v_{j,k}$ from above by comparison with $u_j^{*}$. Step 3 bounds $v_{j,k}$ from below using the first constraint alone, and Step 4 concludes.

\medskip\noindent\textit{Step 1 (feasibility, part (i)).}
Since $\Sigma_du_j^{*}=e_j$, H\"older's inequality and Lemma \ref{lem:2sls_rates}(iii) give
\[
\|\hat{\Sigma}_{d,k}u_j^{*}-e_j\|_{\infty}=\|(\hat{\Sigma}_{d,k}-\Sigma_d)u_j^{*}\|_{\infty}
\le\|\hat{\Sigma}_{d,k}-\Sigma_d\|_{\infty}\|u_j^{*}\|_1\le C\rho_{n_0}C_3\le c'\rho_{n_0}=\gamma_{1n_0},
\]
so constraint \eqref{constraint1:iv} holds for $u_j^{*}$. For constraint \eqref{constraint2:iv}, Lemma \ref{lem:2sls_rates}(ii) gives $\max_{i\in\mathcal{D}_k}\|\hat{d}_i\|_{\infty}\le2MC_1$, hence
\[
\max_{i\in\mathcal{D}_k}|\hat{d}_i^{\intercal}u_j^{*}|\le\max_{i\in\mathcal{D}_k}\|\hat{d}_i\|_{\infty}\|u_j^{*}\|_1\le2MC_1C_3\le\gamma_{2n_0}
\]
for $n_0$ large, because $\gamma_{2n_0}=c''\sqrt{\log n_0}\to\infty$. 
Thus $u_j^{*}$ is feasible.

\medskip\noindent\textit{Step 2 (upper bound, part (ii)).}
Since $\hat{u}_{j,k}$ minimizes $u^{\intercal}\hat{\Sigma}_{d,k}u$ over the feasible set and $u_j^{*}$ is feasible,
\[
v_{j,k}\le u_j^{*\intercal}\hat{\Sigma}_{d,k}u_j^{*}
=\Theta_{jj}+u_j^{*\intercal}(\hat{\Sigma}_{d,k}-\Sigma_d)u_j^{*}
\le C_2+\|\hat{\Sigma}_{d,k}-\Sigma_d\|_{\infty}\|u_j^{*}\|_1^{2}
\le C_2+C\rho_{n_0}C_3^{2}\le C_2+1
\]
for $n_0$ large, by Lemma \ref{lem:2sls_rates}(iii) and $\rho_{n_0}\to0$.

\medskip\noindent\textit{Step 3 (lower bound, part (ii)).}
By constraint \eqref{constraint1:iv}, the $j$-th coordinate of $\hat{\Sigma}_{d,k}\hat{u}_{j,k}$ satisfies $e_j^{\intercal}\hat{\Sigma}_{d,k}\hat{u}_{j,k}\ge1-\gamma_{1n_0}$. The matrix $\hat{\Sigma}_{d,k}$ is positive semi-definite, so the Cauchy--Schwarz inequality in the semi-inner product $\langle a,b\rangle:=a^{\intercal}\hat{\Sigma}_{d,k}b$ and Lemma \ref{lem:2sls_rates}(ii) give
\[
1-\gamma_{1n_0}\le e_j^{\intercal}\hat{\Sigma}_{d,k}\hat{u}_{j,k}
\le\big(e_j^{\intercal}\hat{\Sigma}_{d,k}e_j\big)^{1/2}\big(\hat{u}_{j,k}^{\intercal}\hat{\Sigma}_{d,k}\hat{u}_{j,k}\big)^{1/2}
=\Big(\frac{1}{n_k}\sum_{i\in\mathcal{D}_k}\hat{d}_{ij}^{2}\Big)^{1/2}v_{j,k}^{1/2}
\le2MC_1\,v_{j,k}^{1/2}.
\]
Hence $v_{j,k}\ge(1-\gamma_{1n_0})^{2}/(4M^{2}C_1^{2})$. This is the bound of \citet[][Lemma 12]{javanmard2014confidence}, applied to $\hat{\Sigma}_{d,k}$.

\medskip\noindent\textit{Step 4 (conclusion).}
Part (iii) holds by construction. All bounds are deterministic on this event and hold for every $j$ and $k$.
\end{proof}


\begin{lemma}[Nodewise Lasso]\label{lem:nodewise}
Suppose that Assumption \ref{assum:online} holds and that $\lambda_j=c_j\sqrt{(\log p)/T_0}$ with $c_j$ sufficiently large. Let $\Xi_j$ and $\hat{\Xi}_j$ denote the $j$-th rows of $\Xi$ and of $\hat{\Xi}$ in \eqref{eq:radarOmg}, written as column vectors, and let $\hat{\Sigma}:=T_0^{-1}\sum_{t=1}^{T_0}X_tX_t^{\intercal}$. Then there is a constant $C$ such that, with probability approaching one, simultaneously for all $j\le p$,
\begin{parts}
\item[(i)] $\|\hat{\Sigma}\hat{\Xi}_j-e_j\|_{\infty}\le C\sqrt{(\log p)/T_0}$,
\item[(ii)] $\|\hat{\Xi}_j-\Xi_j\|_1\le C\,s_{\Xi}\sqrt{(\log p)/T_0}$, and hence $\|\hat{\Xi}_j\|_1\le 2C_3$.
\end{parts}
\end{lemma}

\begin{proof}[Proof of Lemma \ref{lem:nodewise}]
Write $r_T:=s_{\Xi}\sqrt{(\log p)/T_0}$, so that $r_T\to0$ by Assumption \ref{assum:online}(iv), and $\Sigma:=\Sigma_{\bar{w}}$. The proof has four steps. Step 1 records the population counterparts of the nodewise quantities. Step 2 applies Lemma \ref{lem:lasso_oracle} to the nodewise regression. Step 3 shows that $\hat{\tau}_j$ is consistent. Step 4 proves (i) and (ii).

\medskip\noindent\textit{Step 1 (population nodewise regression).}
For each $j$, let $\gamma^{j}:=\Sigma_{-j,-j}^{-1}\Sigma_{-j,j}\in\mathbb{R}^{p-1}$ be the coefficient of the population regression of $X_{t,j}$ on $X_{t,-j}$ in the tail subpopulation, and let $\eta_{t,j}:=X_{t,j}-X_{t,-j}^{\intercal}\gamma^{j}$ be its residual, so that $\mathbb{E}[\eta_{t,j}X_{t,-j}\mid Y_t>\bar{w}]=0$. Let $\tau_j^{2}:=\mathbb{E}[\eta_{t,j}^{2}\mid Y_t>\bar{w}]=\Sigma_{jj}-\Sigma_{j,-j}\Sigma_{-j,-j}^{-1}\Sigma_{-j,j}$. By the formula for the inverse of a partitioned matrix, $\Xi_{jj}=1/\tau_j^{2}$ and $\Xi_{j,-j}=-\gamma^{j\intercal}/\tau_j^{2}$. Hence, with $C_j\in\mathbb{R}^{p}$ the vector with entry $1$ in position $j$ and entries $-\gamma^{j}$ in the remaining positions,
\[
\Xi_j=\frac{C_j}{\tau_j^{2}},\qquad
\|\gamma^{j}\|_0\le s_{\Xi},\qquad
\|\gamma^{j}\|_1=\tau_j^{2}\,\|\Xi_{j,-j}\|_1\le C_2C_3,
\]
where $\|\gamma^{j}\|_0\le s_{\Xi}$ follows from $\max_j\|\Xi e_j\|_0\le s_{\Xi}$. Moreover, $\Xi_{jj}\le\lambda_{\max}(\Xi)\le C_2$, so $\tau_j^{2}=1/\Xi_{jj}\ge C_2^{-1}$, while the population least-squares projection gives $\tau_j^{2}\le\Sigma_{jj}\le M^2\le C_2$. Thus $\tau_j^{2}\in[C_2^{-1},C_2]$ without any upper eigenvalue bound on $\Sigma$.
In particular $|\eta_{t,j}|\le M(1+C_2C_3)$ by Assumption \ref{assum:online}(ii). Similarly, let $\hat{C}_j$ be the vector with entry $1$ in position $j$ and entries $-\hat{\gamma}^{j}$ elsewhere, so that $\hat{\Xi}_j=\hat{C}_j/\hat{\tau}_j$ by \eqref{eq:radarOmg}.

\medskip\noindent\textit{Step 2 (rates for $\hat{\gamma}^{j}$).}
The objective in \eqref{eq:radar2} equals one half of $T_0^{-1}\|X_{\cdot,j}-X_{\cdot,-j}\gamma\|_2^{2}+2\lambda_j\|\gamma\|_1$, so $\hat{\gamma}^{j}$ is the Lasso of Lemma \ref{lem:lasso_oracle} with $y_t=X_{t,j}$, $x_t=X_{t,-j}$, $\beta_0=\gamma^{j}$, $r_t=\eta_{t,j}$, $s=s_{\Xi}$ and penalty $2\lambda_j$. Two conditions are to be verified.

For the score, the variables $X_{t,r}\eta_{t,j}$ are centered and bounded by $M^{2}(1+C_2C_3)$. Hoeffding's inequality and a union bound over the $p^{2}$ pairs $(j,r)$ give, with probability at least $1-p^{-c}$,
\begin{equation}\label{eq:nw_score}
\max_{j\le p}\Big\|\frac{1}{T_0}\sum_{t=1}^{T_0}X_{t,-j}\eta_{t,j}\Big\|_{\infty}\le C_0\sqrt{\frac{\log p}{T_0}} .
\end{equation}
Taking $c_j\ge2C_0$ gives $2\lambda_j\ge4C_0\sqrt{(\log p)/T_0}$, as Lemma \ref{lem:lasso_oracle} requires.

For the restricted eigenvalue condition, the entries of $X_tX_t^{\intercal}$ are bounded by $M^{2}$, so Hoeffding's inequality and a union bound over the $p^{2}$ entries give $\|\hat{\Sigma}-\Sigma\|_{\infty}\le C\sqrt{(\log p)/T_0}$ with probability at least $1-p^{-c}$. On this event, for any $b\in\mathbb{R}^{p-1}$ with $\|b_{S_j^{c}}\|_1\le3\|b_{S_j}\|_1$, where $S_j$ is the support of $\gamma^{j}$, the cone inequality $\|b\|_1\le4\sqrt{s_{\Xi}}\|b\|_2$ and $\lambda_{\min}(\Sigma_{-j,-j})\ge\lambda_{\min}(\Sigma)\ge C_2^{-1}$ give
\[
b^{\intercal}\hat{\Sigma}_{-j,-j}b\ \ge\ b^{\intercal}\Sigma_{-j,-j}b-\|\hat{\Sigma}-\Sigma\|_{\infty}\|b\|_1^{2}
\ \ge\ \big(C_2^{-1}-16C\,r_T\big)\|b\|_2^{2}\ \ge\ \tfrac12C_2^{-1}\|b\|_2^{2}
\]
for $T_0$ large. Lemma \ref{lem:lasso_oracle} with $\kappa=\tfrac12C_2^{-1}$ therefore gives, on the intersection of the two events and uniformly in $j$,
\begin{equation}\label{eq:nw_rates}
\|\hat{\gamma}^{j}-\gamma^{j}\|_1\le C\,r_T,\qquad
(\hat{\gamma}^{j}-\gamma^{j})^{\intercal}\hat{\Sigma}_{-j,-j}(\hat{\gamma}^{j}-\gamma^{j})\le C\,s_{\Xi}\lambda_j^{2}\le C\,r_T^{2}.
\end{equation}

\medskip\noindent\textit{Step 3 (consistency of $\hat{\tau}_j$).}
Write $\hat{\eta}_{t,j}:=X_{t,j}-X_{t,-j}^{\intercal}\hat{\gamma}^{j}$. The Karush--Kuhn--Tucker conditions of \eqref{eq:radar2} state that $T_0^{-1}\sum_{t}X_{t,-j}\hat{\eta}_{t,j}=\lambda_j\hat{\kappa}_j$ for a vector $\hat{\kappa}_j$ with $\|\hat{\kappa}_j\|_{\infty}\le1$ and $\hat{\kappa}_{j,l}=\mathrm{sign}(\hat{\gamma}^{j}_{l})$ whenever $\hat{\gamma}^{j}_{l}\neq0$. Hence
\begin{equation}\label{eq:tau_identity}
\hat{\tau}_j=\frac{1}{T_0}\sum_{t=1}^{T_0}\hat{\eta}_{t,j}X_{t,j}
=\frac{1}{T_0}\sum_{t=1}^{T_0}\hat{\eta}_{t,j}^{2}+\hat{\gamma}^{j\intercal}\,\frac{1}{T_0}\sum_{t=1}^{T_0}X_{t,-j}\hat{\eta}_{t,j}
=\frac{1}{T_0}\sum_{t=1}^{T_0}\hat{\eta}_{t,j}^{2}+\lambda_j\|\hat{\gamma}^{j}\|_1 .
\end{equation}
Since $\hat{\eta}_{t,j}=\eta_{t,j}-X_{t,-j}^{\intercal}(\hat{\gamma}^{j}-\gamma^{j})$,
\[
\frac{1}{T_0}\sum_{t=1}^{T_0}\hat{\eta}_{t,j}^{2}
=\frac{1}{T_0}\sum_{t=1}^{T_0}\eta_{t,j}^{2}
-2(\hat{\gamma}^{j}-\gamma^{j})^{\intercal}\frac{1}{T_0}\sum_{t=1}^{T_0}X_{t,-j}\eta_{t,j}
+(\hat{\gamma}^{j}-\gamma^{j})^{\intercal}\hat{\Sigma}_{-j,-j}(\hat{\gamma}^{j}-\gamma^{j}).
\]
The variables $\eta_{t,j}^{2}$ are bounded with mean $\tau_j^{2}$, so Hoeffding's inequality and a union bound over $j$ give $|T_0^{-1}\sum_t\eta_{t,j}^{2}-\tau_j^{2}|\le C\sqrt{(\log p)/T_0}$ for all $j$ with probability at least $1-p^{-c}$. The second term is bounded by $2\|\hat{\gamma}^{j}-\gamma^{j}\|_1$ times the left side of \eqref{eq:nw_score}, hence by $C\,r_T\sqrt{(\log p)/T_0}$. The third term is bounded by $Cr_T^{2}$ by \eqref{eq:nw_rates}. Finally, $\lambda_j\|\hat{\gamma}^{j}\|_1\le\lambda_j(\|\gamma^{j}\|_1+\|\hat{\gamma}^{j}-\gamma^{j}\|_1)\le C\sqrt{(\log p)/T_0}$. Since $s_{\Xi}\ge1$, all four deviations are at most $Cr_T$, so
\begin{equation}\label{eq:tau_rate}
\max_{j\le p}|\hat{\tau}_j-\tau_j^{2}|\le C\,r_T\to0,
\qquad\text{and hence}\qquad
\hat{\tau}_j\ge\tfrac12\tau_j^{2}\ge\tfrac12C_2^{-1}
\end{equation}
for all $j$ and $T_0$ large.

\medskip\noindent\textit{Step 4 (parts (i) and (ii)).}
For (i), the Karush--Kuhn--Tucker conditions and the definition of $\hat{\tau}_j$ give
\[
\hat{\Sigma}\hat{C}_j=\frac{1}{T_0}\sum_{t=1}^{T_0}X_t\hat{\eta}_{t,j},
\qquad\text{whose $j$-th entry is }\hat{\tau}_j\text{ and whose other entries form }\lambda_j\hat{\kappa}_j .
\]
Dividing by $\hat{\tau}_j$, the $j$-th entry of $\hat{\Sigma}\hat{\Xi}_j-e_j$ is zero and each other entry is bounded by $\lambda_j/\hat{\tau}_j\le2C_2\lambda_j$ by \eqref{eq:tau_rate}. This proves (i).

For (ii), write
\[
\hat{\Xi}_j-\Xi_j=\frac{\hat{C}_j-C_j}{\hat{\tau}_j}+C_j\Big(\frac{1}{\hat{\tau}_j}-\frac{1}{\tau_j^{2}}\Big).
\]
Here $\|\hat{C}_j-C_j\|_1=\|\hat{\gamma}^{j}-\gamma^{j}\|_1\le Cr_T$ by \eqref{eq:nw_rates}, $1/\hat{\tau}_j\le2C_2$ by \eqref{eq:tau_rate}, $\|C_j\|_1=1+\|\gamma^{j}\|_1\le1+C_2C_3$ by Step 1, and
\[
\Big|\frac{1}{\hat{\tau}_j}-\frac{1}{\tau_j^{2}}\Big|=\frac{|\tau_j^{2}-\hat{\tau}_j|}{\hat{\tau}_j\tau_j^{2}}\le2C_2^{2}\,Cr_T
\]
by \eqref{eq:tau_rate} and $\tau_j^{2}\ge C_2^{-1}$. Hence $\|\hat{\Xi}_j-\Xi_j\|_1\le Cr_T$ uniformly in $j$. Finally, $\|\hat{\Xi}_j\|_1\le\|\Xi_j\|_1+Cr_T\le C_3+Cr_T\le2C_3$ for $T_0$ large, by Assumption \ref{assum:online}(iv). All events used hold with probability at least $1-Cp^{-c}$, which completes the proof.
\end{proof}

\begin{lemma}[Noise bounds for the RADAR iterates]\label{lem:radar_noise}
Let $\{\mathcal{F}_t\}$ be a filtration and let $\{e_t\}$ be $\mathbb{R}^{p}$-valued with $e_t$ $\mathcal{F}_t$-measurable, $\mathbb{E}[e_t\mid\mathcal{F}_{t-1}]=0$, $\mathbb{E}[\|e_t\|_\infty^{4}\mid\mathcal{F}_{t-1}]\le\sigma_0^{4}$ and $\mathbb{E}[\exp(\|e_t\|_\infty/\sigma_0)\mid\mathcal{F}_{t-1}]\le2$ for a constant $\sigma_0>0$, and set $\sigma:=2e\sigma_0$. Let $q:=2\log p/(2\log p-1)$ be the exponent of the RADAR algorithm, let $y\in\mathbb{R}^{p}$ be non-random and $R>0$, let $\theta^{t}$ be $\mathcal{F}_{t-1}$-measurable with $\|\theta^{t}-y\|_q\le R$, and let $\hat{\theta}$ be non-random with $\|\hat{\theta}-y\|_q\le R$. Let $\alpha^{t}=\alpha/\sqrt{t}$ for $t\ge1$ and $\alpha^{0}=\alpha$. Then the following hold for every $T\ge1$.
\begin{parts}
\item[(a)] $\displaystyle\sum_{t=1}^{T}\alpha^{t-1}\|e_t\|_\infty^{2}\le4\sigma^{2}\alpha\sqrt{T}$ with probability at least $1-(2+\log T)/T$.
\item[(b)] There are universal constants $c_0,C_0>0$ such that, for every $\omega>0$,
$\displaystyle\sum_{t=1}^{T}\langle e_t,\hat{\theta}-\theta^{t}\rangle\le\omega R\sigma\sqrt{T}$ with probability at least $1-\exp\left(-\min\left\{\dfrac{\omega^{2}}{4C_0},\ \dfrac{c_0\,\omega\sqrt{T}}{2}\right\}\right)$.
\end{parts}
\end{lemma}

\begin{proof}[Proof of Lemma \ref{lem:radar_noise}]
Part (a) is a Chebyshev bound for a sum of martingale differences. Part (b) is a Bernstein-type bound for a martingale with conditionally sub-exponential increments.

\medskip\noindent\textit{Part (a).}
Let $\zeta_t:=\alpha^{t-1}\{\|e_t\|_\infty^{2}-\mathbb{E}[\|e_t\|_\infty^{2}\mid\mathcal{F}_{t-1}]\}$. Then $\mathbb{E}[\zeta_t\mid\mathcal{F}_{t-1}]=0$, and for $s<t$ the variable $\zeta_s$ is $\mathcal{F}_{t-1}$-measurable, so $\mathbb{E}[\zeta_s\zeta_t]=\mathbb{E}[\zeta_s\,\mathbb{E}[\zeta_t\mid\mathcal{F}_{t-1}]]=0$. Hence $\mathbb{E}[(\sum_{t=1}^{T}\zeta_t)^{2}]=\sum_{t=1}^{T}\mathbb{E}[\zeta_t^{2}]$, and by the fourth-moment hypothesis $\mathbb{E}[\zeta_t^{2}]\le(\alpha^{t-1})^{2}\,\mathbb{E}[\|e_t\|_\infty^{4}]\le(\alpha^{t-1})^{2}\sigma_0^{4}$. Since $\sum_{t=1}^{T}(\alpha^{t-1})^{2}=\alpha^{2}(1+\sum_{s=1}^{T-1}s^{-1})\le\alpha^{2}(2+\log T)$, Chebyshev's inequality gives
\[
\mathbb{P}\Big(\sum_{t=1}^{T}\zeta_t>\sigma_0^{2}\alpha\sqrt{T}\Big)\le\frac{\alpha^{2}(2+\log T)\sigma_0^{4}}{\sigma_0^{4}\alpha^{2}T}=\frac{2+\log T}{T}.
\]
By Jensen's inequality and the fourth-moment hypothesis, $\mathbb{E}[\|e_t\|_\infty^{2}\mid\mathcal{F}_{t-1}]\le\sigma_0^{2}$, and $\sum_{t=1}^{T}\alpha^{t-1}=\alpha(1+\sum_{s=1}^{T-1}s^{-1/2})\le\alpha(1+2\sqrt{T})\le3\alpha\sqrt{T}$. On the complement of the event above,
\[
\sum_{t=1}^{T}\alpha^{t-1}\|e_t\|_\infty^{2}
=\sum_{t=1}^{T}\zeta_t+\sum_{t=1}^{T}\alpha^{t-1}\mathbb{E}[\|e_t\|_\infty^{2}\mid\mathcal{F}_{t-1}]
\le\sigma_0^{2}\alpha\sqrt{T}+3\sigma_0^{2}\alpha\sqrt{T}=4\sigma_0^{2}\alpha\sqrt{T}\le4\sigma^{2}\alpha\sqrt{T}.
\]

\medskip\noindent\textit{Part (b).}
Let $\varphi_t:=\langle e_t,\hat{\theta}-\theta^{t}\rangle$. Since $\hat{\theta}$ is non-random and $\theta^{t}$ is $\mathcal{F}_{t-1}$-measurable, $\mathbb{E}[\varphi_t\mid\mathcal{F}_{t-1}]=\langle\mathbb{E}[e_t\mid\mathcal{F}_{t-1}],\hat{\theta}-\theta^{t}\rangle=0$. By H\"older's inequality, $\|v\|_1\le p^{1-1/q}\|v\|_q=p^{1/(2\log p)}\|v\|_q=e^{1/2}\|v\|_q\le e\|v\|_q$ for every $v\in\mathbb{R}^{p}$. Hence $\|\hat{\theta}-\theta^{t}\|_1\le e\big(\|\hat{\theta}-y\|_q+\|\theta^{t}-y\|_q\big)\le2eR$ and $|\varphi_t|\le\|e_t\|_\infty\|\hat{\theta}-\theta^{t}\|_1\le2eR\,\|e_t\|_\infty$. With $b:=2eR\sigma_0=R\sigma$, the exponential-moment hypothesis gives $\mathbb{E}[\exp(|\varphi_t|/b)\mid\mathcal{F}_{t-1}]\le\mathbb{E}[\exp(\|e_t\|_\infty/\sigma_0)\mid\mathcal{F}_{t-1}]\le2$. Thus, conditional on $\mathcal{F}_{t-1}$, $\varphi_t$ is centered and sub-exponential with $\psi_1$ norm at most $b$. By \citet[Lemma 5.15]{vershynin2010introduction} there are universal constants $c_0,C_0>0$ such that $\mathbb{E}[\exp(\lambda\varphi_t)\mid\mathcal{F}_{t-1}]\le\exp(C_0\lambda^{2}b^{2})$ for all $|\lambda|\le c_0/b$. Iterating conditional expectations,
\begin{align*}
\mathbb{E}\exp\Big(\lambda\sum_{t=1}^{T}\varphi_t\Big)
&=\mathbb{E}\Big[\exp\Big(\lambda\sum_{t=1}^{T-1}\varphi_t\Big)\,\mathbb{E}[\exp(\lambda\varphi_T)\mid\mathcal{F}_{T-1}]\Big]\\
&\le\exp(C_0\lambda^{2}b^{2})\,\mathbb{E}\exp\Big(\lambda\sum_{t=1}^{T-1}\varphi_t\Big).
\end{align*}
Repeating this step $T-1$ more times gives $\mathbb{E}\exp(\lambda\sum_{t=1}^{T}\varphi_t)\le\exp(C_0T\lambda^{2}b^{2})$ for $0<\lambda\le c_0/b$.
Markov's inequality gives, for every $u>0$ and such $\lambda$,
\[
\mathbb{P}\Big(\sum_{t=1}^{T}\varphi_t>u\Big)\le\exp\left(-\lambda u+C_0T\lambda^{2}b^{2}\right).
\]
If $u/(2C_0Tb^{2})\le c_0/b$, take $\lambda=u/(2C_0Tb^{2})$, which gives the bound $\exp(-u^{2}/(4C_0Tb^{2}))$. Otherwise take $\lambda=c_0/b$. Then $C_0T\lambda^{2}b^{2}=C_0Tc_0^{2}<c_0u/(2b)$, which gives the bound $\exp(-c_0u/(2b))$. In both cases,
\[
\mathbb{P}\Big(\sum_{t=1}^{T}\varphi_t>u\Big)\le\exp\left(-\min\left\{\frac{u^{2}}{4C_0Tb^{2}},\ \frac{c_0u}{2b}\right\}\right).
\]
Taking $u=\omega R\sigma\sqrt{T}=\omega b\sqrt{T}$ gives the claim.
\end{proof}


\begin{lemma}[Convergence rates of the online estimator]\label{lem:radar_rates}
Suppose that Assumption \ref{assum:online} holds and that $s_0^{2}\log(p\vee T_0)=o(T_0)$ and $\{\log(p\vee T_0)\}^{2}=o(T_0)$. Let $\hat{\theta}^{\text{\texttt{on}}}$ be the output of the RADAR algorithm of \citet{agarwal2012stochastic} applied to \eqref{eq:radar1} with the tuning of Appendix \ref{app:online}. 
Then there is a constant $C$ such that, with probability approaching one,
\[
\|\hat{\theta}^{\text{\texttt{on}}}-\theta_0\|_2\le C\sqrt{\frac{s_0\log(p\vee T_0)}{T_0}}
\qquad\text{and}\qquad
\|\hat{\theta}^{\text{\texttt{on}}}-\theta_0\|_1\le C\,s_0\sqrt{\frac{\log(p\vee T_0)}{T_0}} .
\]
On the same event, the population prediction error satisfies
\[
(\hat{\theta}^{\text{\texttt{on}}}-\theta_0)^{\intercal}\Sigma_{\bar{w}}(\hat{\theta}^{\text{\texttt{on}}}-\theta_0)\le C\,\frac{s_0\log(p\vee T_0)}{T_0}.
\]
\end{lemma}

\begin{proof}[Proof of Lemma \ref{lem:radar_rates}]
Write ANW for \citet{agarwal2012stochastic}\footnote{Equation and result numbers in ANW refer to the long version, arXiv:1207.4421.}, and let $q=2\log p/(2\log p-1)$ be the exponent of the RADAR algorithm, so that $\|v\|_1\le e\|v\|_q$ for all $v\in\mathbb{R}^{p}$. Let $\mathcal{F}_t$ be the $\sigma$-field generated by the first $t$ tail observations $(X_r,Y_r)$, $r\le t$, and let $W_t:=\exp(X_t^{\intercal}\theta_0)\log(Y_t/\bar{w})$. By Assumption \ref{assum:online}(i) and Lemma \ref{lem:pareto_moments}, the retained observations are i.i.d.\ draws from the distribution of $(X_t,Y_t)$ given $Y_t>\bar{w}$, and $W_t$ is standard exponential and independent of $(\mathcal{F}_{t-1},X_t)$. The $t$-th RADAR iterate $\theta^{t}$ is $\mathcal{F}_{t-1}$-measurable, and the $t$-th stochastic gradient is $g_t(\theta^{t})$.

The proof has three steps. Step 1 verifies Assumptions 1 and 2 of ANW for the population loss on the region visited by the algorithm, with the constants $G$ and $\gamma$ of \eqref{eq:on_constants}. Step 2 shows that the gradient noise satisfies the hypotheses of Lemma \ref{lem:radar_noise}, which replaces the sub-Gaussian Assumption 3 of ANW in their argument. Step 3 reads off the rates from the epoch structure of ANW.

\medskip\noindent\textit{Step 1 (population loss and its curvature).}
By Lemma \ref{lem:pareto_moments}(i), the population loss is
\[
\bar{\mathcal{L}}(\theta):=\mathbb{E}\left[\exp(X_t^{\intercal}\theta)\log(Y_t/\bar{w})-X_t^{\intercal}\theta\mid Y_t>\bar{w}\right]
=\mathbb{E}\left[\exp\{X_t^{\intercal}(\theta-\theta_0)\}-X_t^{\intercal}\theta\mid Y_t>\bar{w}\right],
\]
with gradient $\nabla\bar{\mathcal{L}}(\theta)=\mathbb{E}[\{\exp(X_t^{\intercal}(\theta-\theta_0))-1\}X_t\mid Y_t>\bar{w}]$ and Hessian $\nabla^{2}\bar{\mathcal{L}}(\theta)=\mathbb{E}[\exp\{X_t^{\intercal}(\theta-\theta_0)\}X_tX_t^{\intercal}\mid Y_t>\bar{w}]$. Since $(X_t,Y_t)$ is independent of $\mathcal{F}_{t-1}$ and $\theta^{t}$ is $\mathcal{F}_{t-1}$-measurable, $\mathbb{E}[g_t(\theta^{t})\mid\mathcal{F}_{t-1}]=\nabla\bar{\mathcal{L}}(\theta^{t})$. The Hessian is positive definite, so $\theta_0$ is the unique minimizer of $\bar{\mathcal{L}}$.

Let $\mathcal{B}:=\{\theta\in\mathbb{R}^{p}\colon\|\theta-\theta_0\|_1\le\bar{R}\}$ with $\bar{R}=2eC_1$. Every iterate, prox center and epoch minimizer of RADAR lies in $\mathcal{B}$. Indeed, in epoch $i$ these points lie in $\{\theta\colon\|\theta-y_i\|_q\le R_i\}$, the proof of ANW Theorem 1 shows that $\|\theta_0-y_i\|_q\le R_i\le R_1=C_1$, and hence $\|\theta-\theta_0\|_1\le e\|\theta-\theta_0\|_q\le2eR_1=\bar{R}$. The initial condition $\|\theta_0-y_1\|_q=\|\theta_0\|_q\le\|\theta_0\|_1\le C_1=R_1$ holds by Assumption \ref{assum:online}(iii). For $\theta\in\mathcal{B}$, Assumption \ref{assum:online}(ii) gives $|X_t^{\intercal}(\theta-\theta_0)|\le M\bar{R}$. With $G$, $\gamma$, $\sigma_0$ and $\sigma$ as in \eqref{eq:on_constants}, it follows that $\|\nabla\bar{\mathcal{L}}(\theta)\|_\infty\le M(e^{M\bar{R}}+1)=G$ on $\mathcal{B}$, so $\bar{\mathcal{L}}$ is $G$-Lipschitz in the $\ell_1$ norm on $\mathcal{B}$, which is Assumption 1 of ANW. For $\theta,\tilde{\theta}\in\mathcal{B}$, the segment between them lies in $\mathcal{B}$, and a second-order Taylor expansion with $\lambda_{\min}(\Sigma_{\bar{w}})\ge C_2^{-1}$ gives
\[
\bar{\mathcal{L}}(\tilde{\theta})-\bar{\mathcal{L}}(\theta)-\langle\nabla\bar{\mathcal{L}}(\theta),\tilde{\theta}-\theta\rangle
\ \ge\ \frac{e^{-M\bar{R}}}{2}\,(\tilde{\theta}-\theta)^{\intercal}\Sigma_{\bar{w}}(\tilde{\theta}-\theta)
\ \ge\ \frac{\gamma}{2}\,\|\tilde{\theta}-\theta\|_2^{2},
\]
which is Assumption 2 of ANW with constant $\gamma$. The proof of ANW Theorem 1 uses these two assumptions only at points of $\mathcal{B}$.

\medskip\noindent\textit{Step 2 (gradient noise).}
Fix an epoch $i$ and condition on a starting history in which all preceding epochs succeeded. By Step 1, the constraint set of epoch $i$ is contained in $\mathcal{B}$, regardless of the observations within the epoch.
Let $e_t:=g_t(\theta^{t})-\nabla\bar{\mathcal{L}}(\theta^{t})$ and $a_t:=\exp\{X_t^{\intercal}(\theta^{t}-\theta_0)\}\in[e^{-M\bar{R}},e^{M\bar{R}}]$. Since $g_t(\theta^{t})=(a_tW_t-1)X_t$,
\[
\|e_t\|_\infty\le M(a_tW_t+1)+M\left(e^{M\bar{R}}+1\right)\le Me^{M\bar{R}}\,(W_t+3).
\]
Because $W_t\sim\mathrm{Exp}(1)$ is independent of $\mathcal{F}_{t-1}$, $\mathbb{E}[(W_t+3)^{4}]=393\le5^{4}$ and $\mathbb{E}[\exp\{(W_t+3)/6\}]=1.2\,e^{1/2}\le2$. Hence
\begin{equation}\label{eq:on_noise}
\mathbb{E}\left[e_t\mid\mathcal{F}_{t-1}\right]=0,\qquad
\mathbb{E}\left[\|e_t\|_\infty^{4}\mid\mathcal{F}_{t-1}\right]\le\sigma_0^{4},\qquad
\mathbb{E}\left[\exp\left(\|e_t\|_\infty/\sigma_0\right)\mid\mathcal{F}_{t-1}\right]\le2 ,
\end{equation}
with $\sigma_0=6Me^{M\bar{R}}$. The noise is sub-exponential, so Assumption 3 of ANW need not hold. In the proof of ANW Theorem 1 that assumption enters only through their Lemma 7, which controls the two random terms in the within-epoch bound (55) of ANW. Lemma \ref{lem:radar_noise} provides the same control under \eqref{eq:on_noise}. It is applied in epoch $i$ conditionally on the information available at the start of the epoch, with $y=y_i$, $R=R_i$ and $\hat{\theta}=\hat{\theta}_i$, the minimizer of the regularized population objective of epoch $i$, all of which are fixed given that information. With the bound $\|\hat{g}^{t}\|_\infty^{2}\le4G^{2}+4\lambda_i^{2}+2\|e_t\|_\infty^{2}$ of ANW, $\sum_{t=1}^{T_i}\alpha^{t-1}\le3\alpha\sqrt{T_i}$ and Lemma \ref{lem:radar_noise}(a),
\[
\sum_{t=1}^{T_i}\alpha^{t-1}\|\hat{g}^{t}\|_\infty^{2}\le12(G^{2}+\lambda_i^{2})\alpha\sqrt{T_i}+8\sigma^{2}\alpha\sqrt{T_i}\le22\,\alpha\sqrt{T_i}\,(G^{2}+\lambda_i^{2}+\sigma^{2}),
\]
which is inequality (58) of ANW. Lemma \ref{lem:radar_noise}(b) is their Lemma 7(b) with the same deviation level $\omega_iR_i\sigma\sqrt{T_i}$, where $\omega_i^{2}=\omega^{2}+24\log i$. By the epoch lengths (12) of ANW and $R_i\le C_1$, $T_i\ge c_1s_0^{2}\omega_i^{2}\sigma^{2}/(\gamma^{2}R_i^{2})\ge c_1\sigma^{2}\omega_i^{2}/(\gamma^{2}C_1^{2})$, so the probability bound of Lemma \ref{lem:radar_noise}(b) is at least $1-\exp(-\tilde{c}\,\omega_i^{2})$ for a constant $\tilde{c}>0$ determined by $(c_0,C_0,c_1,\sigma,\gamma,C_1)$. Conditional on any succesful preceding history, the failure probability in epoch $i$ is therefore at most $(2+\log T_i)/T_i+\exp(-\tilde{c}\,\omega_i^{2})$ in place of $3\exp(-\omega_i^{2}/12)$. With this change, Proposition 1, Lemmas 1 and 2 and the proof of Theorem 1 of ANW apply to $\bar{\mathcal{L}}$ with the constants \eqref{eq:on_constants}. Since $\|\theta_0\|_0\le s_0$, the approximation error $\varepsilon^{2}(\theta_0;S)$ of ANW is zero, their condition (41) holds at every successful epoch, and $\theta_0$ remains feasible on the intersection of the preceding success events, as used in Step 1.

\medskip\noindent\textit{Step 3 (rates).}
Let $K_T$ be the number of epochs completed within $T_0$ iterations, so that $\hat{\theta}^{\text{\texttt{on}}}=y_{K_T+1}$. Since $T_1\le Cs_0^{2}\log(p\vee T_0)\le T_0$ for $T_0$ large, $K_T\ge1$. On the intersection of the per-epoch events, the proof of ANW Theorem 1 gives $\|y_{i+1}-\theta_0\|_q\le R_{i+1}$ for every $i$, hence $\|y_{i+1}-\theta_0\|_1\le eR_{i+1}$, and by their inequality (43) with $\varepsilon^{2}(\theta_0;S)=0$, $\|y_{i+1}-\theta_0\|_2^{2}\le cR_i^{2}/s_0$, where $R_i^{2}=R_1^{2}2^{-(i-1)}$. The total number of iterations after $K$ epochs satisfies $T(K)\le g(K)$ with
\[
g(K)=c\left[\frac{s_0^{2}}{\gamma^{2}R_1^{2}}\left\{e(G^{2}+\sigma^{2})\log p+(\omega^{2}+24\log K)\sigma^{2}\right\}2^{K}+eK\log p\right].
\]
Since $T_0<T(K_T+1)\le g(K_T+1)$, $K_T\le\log_2T_0$, $\omega^{2}=A\log(p\vee T_0)$ and $\{\log(p\vee T_0)\}^{2}=o(T_0)$, the second term of $g(K_T+1)$ is at most $T_0/2$ for $T_0$ large. The first term is at most $Cs_0^{2}\log(p\vee T_0)\,2^{K_T}$, and hence $2^{-K_T}\le Cs_0^{2}\log(p\vee T_0)/T_0$. Therefore
\[
\|\hat{\theta}^{\text{\texttt{on}}}-\theta_0\|_1^{2}\le e^{2}R_1^{2}2^{-K_T}\le C\,\frac{s_0^{2}\log(p\vee T_0)}{T_0}
\quad\text{and}\quad
\|\hat{\theta}^{\text{\texttt{on}}}-\theta_0\|_2^{2}\le\frac{2cR_1^{2}2^{-K_T}}{s_0}\le C\,\frac{s_0\log(p\vee T_0)}{T_0}.
\]

To obtain the population prediction bound, consider a completed epoch $i$ and set
$B_i:=(G^2+\sigma^2)\log p+\omega_i^2\sigma^2$.
The epoch lengths and regularization levels in equations (12) and (15) of ANW satisfy
\[
T_i\ge c\,\frac{s_0^2 B_i}{\gamma^2R_i^2},\qquad
T_i\ge c\log p,\qquad
\lambda_i^2=\frac{R_i\gamma}{s_0\sqrt{T_i}}\sqrt{B_i}
\le C\,\frac{\gamma^2R_i^2}{s_0^2}.
\]
In particular, $\lambda_i\le C\gamma R_i/s_0$.
Let $f_i(\theta):=\bar{\mathcal{L}}(\theta)+\lambda_i\|\theta\|_1$ and let $\hat{\theta}_i$ be its minimizer over the constraint set of epoch $i$.
The within-epoch objective bound (59), equivalently (36a), of ANW, with the noise bounds verified in Step 2, gives on the same per-epoch event
\[
f_i(y_{i+1})-f_i(\hat{\theta}_i)
\le C R_i\left\{\sqrt{\frac{B_i}{T_i}}
+\lambda_i\sqrt{\frac{\log p}{T_i}}\right\}
\le C\,\frac{\gamma R_i^2}{s_0}.
\]
Here $y_{i+1}$ is the average of the iterates of epoch $i$, as in Algorithm 1 of ANW. The factor $e$ in their prox-function constant $A_\psi=e\log p$ is absorbed into $C$.
Since $\theta_0$ is feasible at epoch $i$, $f_i(\hat{\theta}_i)\le f_i(\theta_0)$.
Writing $\delta_i:=y_{i+1}-\theta_0$ and using $\|\delta_i\|_1\le eR_{i+1}\le eR_i$, we obtain
\[
0\le\bar{\mathcal{L}}(y_{i+1})-\bar{\mathcal{L}}(\theta_0)
\le f_i(y_{i+1})-f_i(\hat{\theta}_i)+\lambda_i\|\delta_i\|_1
\le C\,\frac{\gamma R_i^2}{s_0}.
\]
For $\theta_0+\delta_i\in\mathcal{B}$, the identity
$e^a-1-a=a^2\int_0^1(1-t)e^{ta}\,dt$
and $|X_t^{\intercal}\delta_i|\le M\bar{R}$ imply
\[
\bar{\mathcal{L}}(\theta_0+\delta_i)-\bar{\mathcal{L}}(\theta_0)
=\mathbb{E}\left[e^{X_t^{\intercal}\delta_i}-1-X_t^{\intercal}\delta_i\mid Y_t>\bar{w}\right]
\ge\frac{e^{-M\bar{R}}}{2}\,
\delta_i^{\intercal}\Sigma_{\bar{w}}\delta_i.
\]
Taking $i=K_T$, using $R_{K_T}^2=2R_1^2 2^{-K_T}$ and the bound on $2^{-K_T}$ obtained above, gives
\[
(\hat{\theta}^{\text{\texttt{on}}}-\theta_0)^{\intercal}
\Sigma_{\bar{w}}
(\hat{\theta}^{\text{\texttt{on}}}-\theta_0)
\le C\,\frac{R_{K_T}^2}{s_0}
\le C\,\frac{s_0\log(p\vee T_0)}{T_0}.
\]
This uses the same per-epoch events and the same tuning as the norm bounds, and requires no upper eigenvalue bound on $\Sigma_{\bar{w}}$.

Finally, the probability that some per-epoch event fails is at most
\[
\sum_{i=1}^{K_T}\left\{\frac{2+\log T_i}{T_i}+\exp(-\tilde{c}\,\omega_i^{2})\right\}.
\]
By the epoch lengths of ANW and $R_i^{2}=R_1^{2}2^{-(i-1)}$, $T_i\ge c_1s_0^{2}\omega^{2}\sigma^{2}2^{i-1}/(\gamma^{2}R_1^{2})\ge c\,2^{i-1}\log(p\vee T_0)$. Since $(2+\log T)/T$ is decreasing in $T\ge1$, the first sum is at most $C\log\log(p\vee T_0)/\log(p\vee T_0)\to0$. For the second sum, take $A\ge2/\tilde{c}$, so that $\exp(-\tilde{c}\,\omega_i^{2})\le\exp(-\tilde{c}\,\omega^{2})\le(p\vee T_0)^{-2}$ and the sum is at most $K_T(p\vee T_0)^{-2}\le T_0^{-1}\to0$.
\end{proof}

\end{document}